\PassOptionsToPackage{round}{natbib}
\documentclass{article}
\usepackage[letterpaper,margin=1in]{geometry}
\usepackage{times}
\usepackage[round]{natbib}
\usepackage[round]{natbib}
\usepackage{amsmath,amssymb,amsthm,bm}
\usepackage{mathtools}
\usepackage{graphicx}
\usepackage{wrapfig}
\usepackage{booktabs}
\usepackage{array}
\usepackage{xcolor}
\usepackage{microtype}
\usepackage{caption}
\usepackage{subcaption}
\usepackage{enumitem}
\usepackage{algorithm}
\usepackage{float}
\usepackage[section]{placeins}

\usepackage{algpseudocode}
  \usepackage{caption}
  \usepackage{tikz}

\usepackage{multirow}
\usepackage{booktabs}
\usepackage{hyperref}

  \usetikzlibrary{arrows.meta,backgrounds,fit,calc,positioning,%
                  decorations.pathreplacing,shapes.geometric}
\theoremstyle{plain}
\newtheorem{theorem}{Theorem}
\newtheorem{proposition}{Proposition}
\newtheorem{lemma}{Lemma}
\newtheorem{corollary}{Corollary}
\theoremstyle{definition}
\newtheorem{definition}{Definition}
\newtheorem{assumption}{Assumption}
\theoremstyle{remark}
\newtheorem{remark}{Remark}
\DeclareMathOperator*{\argmax}{arg\,max}
\DeclareMathOperator*{\argmin}{arg\,min}
\newcommand{\R}{\mathbb{R}}
\newcommand{\E}{\mathbb{E}}
\newcommand{\W}{\mathcal{W}}
\newcommand{\X}{\mathcal{X}}
\newcommand{\passk}{\mathrm{pass}@k}
\newcommand{\TV}{\mathrm{TV}}
\newcommand{\Var}{\mathrm{Var}}
\newcommand{\Cov}{\mathrm{Cov}}
\newcommand{\CE}{\mathrm{CE}}

\newcommand{\shat}{\hat{s}}

\newcommand{\Reg}{\mathrm{Reg}}

\def\one{\mathbf{1}}

\title{The Optimal Decoder is Tilted Temperature}
\newcommand{\alignedauthor}[2]{\parbox[t]{0.29\textwidth}{\centering\vspace{0pt}\strut\large\textbf{#1}\par\vspace{3pt}\normalsize Purdue University\par\vspace{3pt}\small\texttt{#2}}}
\author{\begin{tabular}{@{}c@{\hspace{0.35cm}}c@{\hspace{0.35cm}}c@{}}
\alignedauthor{Nikolaos Papagiannis}{npapagia@purdue.edu} &
\alignedauthor{Ananth Grama}{ayg@purdue.edu} &
\alignedauthor{Wojciech Szpankowski}{szpan@purdue.edu}
\end{tabular}}
\date{}

\begin{document}
\maketitle
\thispagestyle{plain}
\sloppy
\begin{abstract}
It is known that larger sample budgets benefit from higher sampling temperatures in LLMs. However, we lack a formal understanding of the objective that temperature tuning optimizes and how it accomplishes this. We seek to answer these questions: we pose decoding as pass@$k$ maximization subject to two constraints that encode critical decoding failure modes: an entropy floor preventing diversity collapse, and a cross-entropy ceiling preventing implausibility. We prove that this maximization takes the form of a single closed-form conditional sampling rule, the \emph{tilted temperature} $q^\star \propto \bigl(p e^{\eta V}\bigr)^{1/T}$. Here, the tilt comes from a verifier potential $V$ that quantifies how likely a partial generation is to eventually pass. Ordinary temperature sampling, verifier-greedy decoding, and reward-tilted policies represent special cases. We further show that, under a natural rule for scaling the diversity budget with sample size $k$, the resulting optimal temperature increases with $k$, turning empirical choice into a theorem. The optimum still needs the unknown verifier potential, which we estimate using two design approaches. In the online setting, our J-EOS algorithm jointly tracks the two constraint multipliers and learns the verifier potential from binary pass/ fail feedback alone. In the offline setting, we show that when the model is calibrated, its log-probability tracks correctness up to an affine map. Under this condition, the tilt collapses into the exponent and the optimum reduces to ordinary temperature sampling. Calibrated Temperature Decoding (CTD) then fits one scalar offline and samples with zero additional inference-time overhead relative to temperature sampling. Building on these decision rules, we design a practical decoding algorithm for code generation and evaluate it alongside CTD on benchmarks including HumanEval and MBPP, showing improvements over standard and adaptive decoding baselines.
\end{abstract}
\section{Introduction}

LLMs are often used and evaluated under a metric that bears little resemblance
to the one that shaped them. Training maximizes the likelihood of a single next token; however, deployment in code
generation, math, and reasoning pipelines is judged by $\passk$, the probability
that at least one of $k$ independently drawn sequences solves the task. Chain-of-thought traces with lower
likelihood routinely outperform their higher-likelihood alternatives on reasoning benchmarks; best-of-$k$ and
self-consistency pipelines exploit precisely this gap; and increasing $k$ yields accuracy at a rate that
scales predictably with compute. The decoding hyper-parameter that mediates this trade-off, temperature,
together with its truncated variants top-$k$ and nucleus, is empirically tuned. Practitioners have converged on the
hypothesis that small $k$ warrants cold sampling and large $k$ prefers hot. However, there is no principled account of which
temperature is right, or even of which objective is being optimized when one is chosen. 

Formulating these questions as an optimization problem and solving it is hard for two reasons. First, the optimum of an unconstrained $\passk$ problem is degenerate:
the highest-$\passk$ sampler is supported on whichever sequences score best under the success criterion, leading to
a Dirac mass on the single best sequence. This does not generalize, collapses sample diversity to zero, and
exploits an imperfect verifier. We therefore need an \emph{entropy floor} that forces the sampling distribution to
retain spread. Second, the dual failure mode is just as bad: maximizing entropy alone produces uniform
sampling, supported on enormous amounts of fictitious data, discarding the information about coherence that the
base model distribution $p$ encodes at considerable training cost. We therefore need a second constraint on
the sampling distribution $q$, which is a
plausible re-weighting of $p$, to enforce a \emph{ceiling} on the average per-token cross-entropy of $q$ against $p$.
The resulting program is, for some parameters $\beta$ and $\gamma$,
{\small
\begin{equation}\tag{PE}\label{eq:PE}
\max_{q}\ \passk(q)~~ \text{s.t.}~~~~~
\tfrac{1}{n}\sum_t \E_{w^{t-1}\sim q} H\!\big(q_t(\cdot|w^{t-1})\big)\ge \beta,
-\tfrac{1}{n}\sum_t \E_{w^{t-1}\sim q}\!\left[\sum_w q_t \log p_t(w\mid w^{t-1})\right] \le \gamma.
\end{equation}
}
The two constraints encode the two failure modes that any decoder must avoid: entropy collapse and implausibility.
What makes \eqref{eq:PE} difficult is the structure of the objective. $\passk$ is a sequence-level quantity that depends on the whole generated sequence through a verifier observed only after the last token, while the decoder operates token by token. The bridge is the verifier potential $V_t(w\mid w^{t-1})$, a backward continuation free-energy associated with committing to token $w$ after prefix $w^{t-1}$. It combines the terminal verifier reward with the downstream likelihood weighting induced by the two dual prices, and satisfies a backward Bellman recursion (Lemma~\ref{lem:bellman}). This yields an exact token-level dynamic-programming characterization of the sequence-level optimum. The resulting optimum has conditional form
$
q^\star_t(w\mid w^{t-1})
\propto
\bigl(
p_t(w\mid w^{t-1})
\exp(\eta V_t(w\mid w^{t-1}))
\bigr)^{1/T},
$
with $T=1/\lambda^\star$ and $\eta=\mu^\star/\lambda^\star$. We call this the \emph{tilted temperature} distribution. The family interpolates between likelihood-driven and increasingly verifier-driven decoding. At zero verifier price, its sequence law reduces to tempering of the base sequence distribution; ordinary per-token temperature sampling closely approximates this law when the remaining-horizon correction has limited variation across candidate tokens, and is exact in the zero-variation case. Since $\passk(q)$ is strictly increasing in the single-sample success, for fixed budgets $(\beta,\gamma)$ the optimizer itself is independent of $k$; the $k$-dependence arises through the budget-selection rule for $\beta$ and is discussed in Appendix~\ref{app:beta}. 

The form of the optimum is established by Theorem~\ref{thm:tilt}; what is not given in practice is the verifier potential $V$.
We handle this unknown quantity in two complementary ways:\\
{\bf 1.} \emph{Online analytical construction: jointly track the prices and estimate $V$.}
Our \textsc{J-EOS} construction (Section~\ref{sec:tracking}) uses binary verdicts from completed rollouts to update a backward Monte-Carlo estimate of the verifier potential while causally tracking the two dual prices $(\lambda,\mu)$ associated with the active constraints. For a fixed potential estimate, we prove that prefix-wise joint price tracking remains close to the offline optimum in parameter space and consequently in sequence law and $\passk$ (Theorems~\ref{thm:prefix-tracking} and~\ref{thm:passk-tracking}). \textsc{J-EOS} is used as an analytical construction establishing how the two constrained prices can be tracked online; it is not our recommended deployment algorithm.\\
{\bf 2.} \emph{Calibration-based, offline solution: making the verifier tilt redundant.}
Rather than estimating $V_t$ directly, we exploit the sequence-level form of the optimum. We formalize \emph{$(s,\varepsilon)$-calibration} by requiring the terminal verifier to be approximately affine in the mean sequence log-probability,
$
U(w^n)=s\,\bar\ell(w^n)+\bar c+\bar r(w^n),\qquad |\bar r(w^n)|\le\varepsilon.
$
Under perfect calibration, maximizing success is equivalent to minimizing cross-entropy over the feasible set. The optimum therefore has sequence law $q^\star(w^n)\propto p(w^n)^{a_\beta}$, where the entropy floor pins $a_\beta$ and the cross-entropy ceiling acts only as a feasibility check. When the remaining-horizon correction has limited variation, ordinary per-token temperature sampling closely approximates this sequence optimum. This yields \emph{Calibrated Temperature Decoding (CTD)}: the regression estimates whether positive likelihood--correctness alignment holds and how large the residual is (Theorem~\ref{thm:ctd}).

For deployment beyond the calibrated regime, we use the theory as a design principle rather than run \textsc{J-EOS} directly. Our practical decoder, \textsc{TMAD-Gated} (Algorithm~\ref{alg:tmad-gated}), intervenes only at uncertain tokens, uses short lookahead to construct a continuation-success proxy, and selects among tilted-temperature regimes subject to plausibility and diversity. 
TMAD-Gated achieves the best pass@$k$ in 16 of 18 settings across three models and two code-generation benchmarks, outperforming both fixed-temperature and adaptive temperature sampling in most settings.

\vspace{-0.15in}
\paragraph{Contributions.}
Our main contributions can be summarized as follows:
(i) \textit{The decoding problem under entropy and plausibility constraints (Section~\ref{sec:formulation}).} We formulate
\eqref{eq:PE}, in which the entropy floor enforces diversity and the cross-entropy ceiling enforces plausibility; the
constraints correspond to the two distinct failure modes any deployable decoder must avoid.
(ii) \textit{Tilted temperature as the closed-form optimum (Section~\ref{sec:main}).} We prove the optimum is $q^\star\propto(p\,e^{\eta V})^{1/T}$ (Theorem~\ref{thm:tilt}); under the budget-selection rule of Proposition~\ref{prop:mono}, the induced temperature is non-decreasing in $k$ in the high-coverage regime and strictly increasing at interior points where the inequality is strict.
(iii) \textit{A sequence-level verifier admits an exact token-level potential characterization.}
Although success is observed only after a complete generation, we show that the optimal decoder can be expressed through a backward verifier potential $V_t$ satisfying a Bellman-style recursion. This converts the sequence-level $\passk$ objective into a token-level
tilted decoding rule through an acyclic Bellman recursion.
(iv) \textit{Joint online price tracking as an analytical construction (Section~\ref{sec:tracking}).}
We introduce \textsc{J-EOS}, which couples online potential estimation with causal tracking of both dual prices. Holding the potential estimate fixed, we prove $O(1/t)$ prefix-price tracking and corresponding sequence-law and $\passk$ stability bounds (Theorems~\ref{thm:prefix-tracking} and~\ref{thm:passk-tracking}). We use this construction to analyze the online constrained problem rather than as the practical decoder in our experiments.
(v) \textit{Calibration-based offline decoding: CTD (Section~\ref{sec:calib}).}
We show that sequence-level $(s,0)$-calibration with $s>0$ collapses the optimum to the tempered sequence law $q^\star(w^n)\propto p(w^n)^{a_\beta}$, where the entropy floor pins $a_\beta$ and the cross-entropy ceiling checks feasibility. Ordinary per-token temperature sampling closely approximates this optimum when the remaining-horizon correction has limited variation, with prefix independence as the zero-variation special case. \textsc{CTD} estimates the likelihood--correctness alignment offline by OLS, uses it as a certificate rather than as the temperature parameter, and admits a suboptimality bound controlled by calibration residual and continuation variation (Theorems~\ref{thm:collapse}--\ref{thm:ctd}).
(vi) \textit{Theory-guided practical decoding and code-generation experiments (Section~\ref{sec:experiments}).}
We instantiate the tilted-temperature structure in \textsc{TMAD-Gated} (\emph{Tilted Model-Adaptive Decoding}), a confidence-gated decoder that combines short-horizon continuation scoring with constrained selection of $(T,\eta)$ regimes. On HumanEval \citep{chen2021evaluating} and MBPP \citep{austin2021program} we evaluate \textsc{CTD} and \textsc{TMAD-Gated} against fixed-temperature and adaptive decoding baselines, testing when calibration alone is sufficient and when explicit continuation information provides additional value.

\vspace{-0.1in}
\paragraph{Related work.}
A recent line of work treats the number of decoding samples as a resource to
be scaled: coverage grows log-linearly in $k$ over orders of magnitude
\citep{brown2024monkeys}, compute-optimal frontiers trade model size against
sample count \citep{wu2024inference}, and allocation can be adapted per
prompt \citep{snell2024scaling}. These studies establish pass@$k$ as the
operational objective but hold the sampling distribution fixed at a
hand-tuned temperature. We ask instead, which distribution maximizes pass@$k$
under an explicit budget, and show that the answer is the tilted
temperature \eqref{eq:opt1}. A complementary thread improves
\emph{selection} over a fixed pool---outcome verifiers \citep{cobbe2021gsm8k},
process reward models \citep{lightman2024verify}, best-of-$n$ reranking
\citep{nakano2021webgpt}, and self-consistency
\citep{wang2023selfconsistency}, whereas our potential $V_t$ enters the
\emph{process} through Bellman recursion, and becomes redundant under
calibration (Theorem~\ref{thm:collapse}). Heuristic reshaping rules
—nucleus sampling \citep{holtzman2020degeneration}, truncation-as-desmoothing
\citep{hewitt2022truncation}, are justified empirically. Under perfect sequence-level
calibration, the optimum becomes a pure power of the base sequence law, and
ordinary per-token temperature sampling closely approximates it when the
remaining-horizon correction has limited variation across candidate tokens;
prefix independence is the zero-variation special case. The tilted form itself is the Gibbs family of maximum-entropy inference
\citep{jaynes1957information,cover2006elements} and KL-regularized control
\citep{ziebart2008maxent,levine2018rlinference,korbak2022rlbayesian,rafailov2023dpo}.
We differ in operating at inference time under a cross-entropy budget, which
is what pins $T$ to the budget rather than to a training-time penalty. Our
calibration assumption builds on Platt scaling
\citep{platt1999probabilistic,guo2017calibration} and evidence that
pretrained models partly track their own correctness
\citep{desai2020calibration,kadavath2022know}.

\vspace{-0.1in}
\section{Problem Formulation}\label{sec:formulation}

\providecommand{\KL}{\mathrm{KL}}
 \providecommand{\Ent}{\mathrm{H}}

Let $\W$ denote the finite token alphabet and $\W^n$ the set of length-$n$ sequences.
Let $p_t(w)=\Pr(w \mid w^{t-1})\in\Delta(\W)$ be the model's conditional distribution at
step $t$. A \emph{policy} (sampling distribution) is $q=(q_1,\dots,q_n)$ with
$q_t$ depending only on $\{p_i\}_{i\le t})$, inducing $q(w^n)=\prod_t q_t(w_t\mid w^{t-1})$. We
assume $p_t(w)\ge\rho>0$ for all $t,w$. Let $U:\W^n\to[0,1]$ be the verifier; often
binary, $U(w^n)\in\{0,1\}$ (e.g.\ code: unit tests pass; math: answer matches). The
single-sample success probability of $q$ is $u_q:=\E_{w^n\sim q}[U(w^n)]$, and drawing
$k$ i.i.d.\ samples gives $\passk(q):=1-(1-u_q)^k$ leading to
\begin{equation}
    \label{eq-pe}
  \max_{q_t\in\Delta(\W)}\; u_q
  ~~\text{s.t.}\quad
  \frac1n\sum_{t}\E_{w^{t-1}\sim q}\Big[\Ent\big(q_t(\cdot\mid w^{t-1})\big)\Big]\;\ge\;\beta,
  ~~
  -\frac1n\sum_{t}\E_{w^{t-1}\sim q}\Big[\sum_{w}q_t\log p_t\Big]\;\le\;\gamma,
\end{equation}
with $\Ent(q_t)=-\sum_w q_t(w\mid w^{t-1})\log q_t(w\mid w^{t-1})$ the entropy of $q$.
To analyze token-level optimization, we summarize a committed prefix--token pair by a
value function. Let $\lambda\ge 0$ and $\mu_E\ge 0$ denote the KKT multipliers of the
cross-entropy ceiling and entropy floor, respectively. Whenever $\mu_E>0$, set
$\lambda^\star:=n\lambda/\mu_E\ge0$ and $\mu^\star:=n/\mu_E>0$. If
$\lambda^\star>0$, we equivalently write $T:=1/\lambda^\star$ and
$\eta:=\mu^\star/\lambda^\star$, so that $T$ is the temperature and
$\eta$ the verifier tilt.
For $\mu^\star>0$ the verifier potential at step $t$ is the free energy of the tilted
continuation,
\begin{equation}\label{eq:V-def}
  V_t(w\mid w^{t-1})\;:=\;\frac{1}{\mu^\star}\,
  \log\!\!\sum_{w_{t+1:n}\in\W^{\,n-t}}\ \prod_{s=t+1}^{n}
  p_s(w_s\mid w^{s-1})^{\lambda^\star}\, e^{\mu^\star U(w^n)},
  ~~ w^n=(w^{t-1},w,w_{t+1:n}),
\end{equation}
with the convention that at $t=n$ the continuation set $\W^0$ is the single empty tuple
and the product is empty, so that the terminal condition reads
$V_n(w\mid w^{n-1})=U(w^{n-1},w)$. 
It scores the tilted continuation value of a prefix--token pair, combining terminal verifier reward with the downstream likelihood weighting induced by the dual prices. It depends on $(p,U,\lambda^\star,\mu^\star)$ but not on any policy. This is the object the optimum of Theorem~\ref{thm:tilt} tilts on, and we denote it by $V_t$ throughout the paper. The next lemma follows by peeling the innermost sum. The proof is in Appendix~\ref{app:lem1}.

\begin{lemma}[Bellman consistency]\label{lem:bellman}
For any $\lambda^\star\ge0$, $\mu^\star>0$, step $t\in\{1,\dots,n-1\}$, prefix $w^{t-1}$,
and token $w$,
\[
  V_t(w\mid w^{t-1})\;=\;\frac{1}{\mu^\star}\log\sum_{w'\in\W}
  p_{t+1}(w'\mid w^{t-1}w)^{\lambda^\star}\,
  \exp\!\big(\mu^\star V_{t+1}(w'\mid w^{t-1}w)\big),
\]
with terminal condition $V_n(w\mid w^{n-1})=U(w^{n-1},w)$. When
$\lambda^\star>0$, equivalently $V_t=\tfrac{T}{\eta}\log Z_{t+1}(w^{t-1}w)$ with
\[
  Z_{t+1}(w^{t-1}w)\;=\;\sum_{w'\in\W}
  \Big(p_{t+1}(w'\mid w^{t-1}w)\, e^{\eta V_{t+1}(w'\mid w^{t-1}w)}\Big)^{1/T}.
\]
The recursion involves $(p,U,\lambda^\star,\mu^\star)$, so $V$ exists, is unique, and
is obtained by one backward sweep over prefix states.
\end{lemma}

\begin{remark}\label{rem:prompt}
Program (PE) is written for a fixed prompt $x$: $p_t(w\mid w^{t-1})$ already conditions on
$x$, so the optimum $q^\star$ is a per-prompt object. 
In Appendix~\ref{app:prompt}, we make the prompt
dependence explicit through the natural prompt-level objectives  (see Theorem~\ref{thm:inv}).
In Appendix~\ref{app:adaptive}, we study adaptive (PE) where $(\beta, \gamma)$ may take $D$ different values in $[0,n]$. 
\end{remark}

\section{Main Results}\label{sec:main}

In Section~\ref{sec:optimal}, we derive the optimal tilted-temperature sampling distribution.
Section~\ref{sec:tracking} introduces \textsc{J-EOS} as an analytical online construction for joint price tracking and potential estimation, while Section~\ref{sec:calib} develops \textsc{CTD}, an offline calibration-based decoder that removes the need to explicitly estimate the potential under additional structure.
\subsection{Optimal Distribution}\label{sec:optimal}
Program~\eqref{eq:PE} is a finite-horizon entropy-regularized control problem: a completed sequence is rewarded by $U$, and each step by a mixture of the base log-likelihood and the policy entropy. Its optimum is the Gibbs policy on the verifier potential of Lemma~\ref{lem:bellman}. We use the following:

\begin{lemma}[Gibbs variational principle]\label{lem:gibbs}
For $c\in\R^{\W}$ and $\tau>0$, $\ \max_{q\in\Delta(\W)}\{\langle q,c\rangle+\tau H(q)\}=\tau\log\sum_{w} e^{c(w)/\tau}$, uniquely attained at $q(w)=e^{c(w)/\tau}/\sum_{w'} e^{c(w')/\tau}$, where $\langle q,c\rangle$ is the scalar product of $q$ and $c$ over $w$.
\end{lemma}
\vspace{-0.05in}
\begin{proof}
$\langle q,c\rangle+\tau H(q)=-\tau\,\mathrm{KL}(q\,\|\,g)+\tau\log\sum_w e^{c(w)/\tau}$ with $g(w)\propto e^{c(w)/\tau}$; with equality iff $q=g$.
\end{proof}

\begin{theorem}[Tilted temperature optimum: Proof in Appendix~\ref{app:proofTh1}]\label{thm:tilt}
Let $\mu_E\ge0$ and $\lambda\ge0$ be the multipliers of the entropy floor and cross-entropy ceiling of \eqref{eq:PE}. For $\mu_E>0$, set $\lambda^\star:=n\lambda/\mu_E$ and $\mu^\star:=n/\mu_E$. The optimum of \eqref{eq:PE} is
\begin{equation}\label{eq:opt1}
q^\star_t(w\mid w^{t-1})=\frac{p_t(w)^{\lambda^\star}\exp(\mu^\star V_t(w\mid w^{t-1}))}{Z_t(w^{t-1};\lambda^\star,\mu^\star)},\qquad Z_t=\sum_{w} p_t(w)^{\lambda^\star}\exp(\mu^\star V_t(w\mid w^{t-1})),
\end{equation}
with $V_t$ the verifier potential of Lemma~\ref{lem:bellman}. If $\lambda^\star>0$, this is equivalent, with $T=1/\lambda^\star$ and $\eta=\mu^\star/\lambda^\star$, to
\begin{equation}\label{eq:opt2}
q^\star_t(w\mid w^{t-1})\propto\big(p_t(w\mid w^{t-1})\exp(\eta\, V_t(w\mid w^{t-1}))\big)^{1/T}.
\end{equation}
\end{theorem}
\begin{remark}
We show in Appendix~\ref{app:beta} that, under a natural rule for scaling
the diversity budget with $k$, the induced optimal temperature is
non-decreasing in the high-coverage regime
(Proposition~\ref{prop:mono}). A HumanEval temperature sweep with
Qwen2.5-Coder-1.5B shows the
same trend; see Figure~\ref{fig:tstar}.
\end{remark}

\subsection{ONLINE TRACKING OF THE JOINT DUAL PRICES}
\label{sec:tracking}

\begin{algorithm}[!htbp]
\caption{J-EOS: Joint online tracking of $(\lambda,\mu)$ with online potential estimation.}
\label{alg:jeos}
\begin{algorithmic}[1]
\Require budgets $(\beta,\gamma)$; base model $p$; verifier $U$; per-round
  rollout batch size $m$; rounds $N$; Newton tolerance $\epsilon$;
  max inner iterations $J_{\max}$; initial $\theta_0=(\lambda_0,0)$ (plain
  temperature, no tilt)
\State $\hat V_t(w) \gets 0$ for all $t,w$ \Comment{warm start: no verifier signal yet}
\For{$r = 1,\dots,N$}
  \State $\theta \gets \theta_{r-1}$ \Comment{carry the previous round's converged prices}
  \For{$t = 1,\dots,n$} \Comment{inner loop: causal prefix tracking}
    \For{$j = 1,\dots,J_{\max}$}
      \State $q_i^\theta(w\mid w^{i-1}) \gets \dfrac{p_i(w)^{\lambda}\exp\!\big(\mu\,\hat V_i(w)\big)}{\displaystyle\sum_{w'} p_i(w')^{\lambda}\exp\!\big(\mu\,\hat V_i(w')\big)}\ \text{ for } i\le t$; evaluate $F_t(\theta)$ via \eqref{eq:Ft}
      \If{$\|F_t(\theta)\|_2 \le \epsilon$} \textbf{break} \EndIf
      \State $\theta \gets \theta -  J_t(\theta)^{-1}F_t(\theta)$ \Comment{Newton step using the exact Jacobian $DF_t(\theta)$}
    \EndFor
    \State $\theta_t \gets \theta$; sample $w_t \sim q_t^{\theta_t}$
  \EndFor
  \State freeze the round policy $b_{r,t}:=q_t^{\theta_t}$ and draw $m$ fresh i.i.d. sequences from $b_r$; collect verdicts $U_j$, $j\in[m]$
  \State set $\vartheta_r=(\lambda_r,\mu_r):=\theta_n$; if $\mu_r>0$, update $\hat V_t(w) \gets \frac{1}{\mu_r}\log\hat Z^{(r)}_{t+1}(w)$ on visited cells via \eqref{eq:jeos-steponly} \Comment{Appendix~\ref{app:mc}}
  \State $\theta_r \gets \theta_n$ 
  \Comment{carry forward as next round's warm start}
\EndFor
\State \Return $\bar\theta := (\bar\lambda,\bar\mu) = \frac1{N}\sum_r \theta_r$; final trained table $\hat V$
\end{algorithmic}
\end{algorithm}

Theorem~\ref{thm:tilt} pins the optimal decoder \emph{given} the true
dual prices $\theta^\star=(\lambda^\star,\mu^\star)$, but those prices are
defined by two budget equalities that average over the \emph{entire} horizon of (PE).
A causal decoder at step $t$ only has access to the prefix $(p_1,\hat V_1),\dots,(p_t,\hat V_t)$, which is already an
approximation of the true potential, not $V_t$ itself. Therefore, $\theta^\star$ cannot be computed on the fly even setting
approximation error aside: both sums in (PE) run over the full horizon $1,\dots,n$, including
positions the decoder has not yet reached. This subsection addresses this causal-deployment gap directly, assuming
that some estimate $\hat V_t$ of the verifier potential is available.
We treat $\hat V_t$ as fixed throughout this subsection, separating the
price-tracking problem from the question of potential estimation.
Appendix~\ref{app:mc} derives the specific step-only estimator used in
Algorithm~\ref{alg:jeos}.
Write $\theta=(\lambda,\mu)^\top\in\Theta\subset\mathbb R^2$,
$\ell_t(w)=\log p_t(w)$, $\phi_t(w)=(\ell_t(w),V_t(w))^\top$, and for a common
parameter $\theta$, define the tilted family
\begin{equation}
q_t^\theta(w) \;=\; \frac{\exp\bigl(\lambda\,\ell_t(w)+\mu\,V_t(w)\bigr)}
{\sum_{w'\in\W}\exp\bigl(\lambda\,\ell_t(w')+\mu\,V_t(w')\bigr)}
\;\propto\; p_t(w)^\lambda\,e^{\mu V_t(w)},
\label{eq:tilted-family}
\end{equation}
which is exactly the optimal conditional \eqref{eq:opt1} at
$\theta=\theta^\star$. Define the two-dimensional cumulative constraint map
\begin{equation}
F_t(\theta) \;=\;
\begin{pmatrix}
\sum_{i=1}^t\bigl[\CE(q_i^\theta,p_i)-\gamma\bigr] \\[2mm]
\sum_{i=1}^t\bigl[H(q_i^\theta)-\beta\bigr]
\end{pmatrix},
\label{eq:Ft}
\end{equation}
so the full-horizon KKT point satisfies $F_n(\theta^\star)=0$. 
Note that \eqref{eq:Ft} alone does not certify optimality, that is the tilted
Gibbs form together with primal feasibility and complementary slackness established in Theorem~\ref{thm:tilt}. Here, it is specialized to
the interior regime where both duals are strictly positive.

\paragraph{Prefix tracking.} The causal analog of the length-$n$ root
equation is to solve the \emph{prefix} KKT system
\begin{equation}
F_t(\theta_t)=0, \qquad \text{and deploy } q_t=q_t^{\theta_t}.
\label{eq:prefix-root}
\end{equation}
Crucially, \eqref{eq:prefix-root} is solved retrospectively with one trial
parameter $\theta$ applied uniformly to $i\le t$ and it does not substitute the
already-deployed $q_i^{\theta_i}$ back into the equation. Define the prefix
residual of the \emph{final} offline prices,
$R_t := F_t(\theta^\star)$ and  $R:=\max_{1\le t\le n}\|R_t\|_2$ (see \cite{sima2026}),
so $R_n=0$ and $R_t$ for $t<n$ measure how far $\theta^\star$ is from
satisfying the length-$t$ budgets.

\begin{assumption}[Two-dimensional regularity]
\label{ass:jacobian}
Let $ J_t := \int_0^1 DF_t\bigl(\theta_t+s(\theta^\star-\theta_t)\bigr)\,ds$.
For every $t$ at which \eqref{eq:prefix-root} is solved, the smallest singular value
$\sigma_{\min}( J_t)\ge c\,t$ for a constant $c>0$ independent of $t$, where $D$ represents the derivative.
\end{assumption}

\begin{theorem}[Prefix-price tracking: Proof in Appendix~\ref{app:thm23}]
\label{thm:prefix-tracking}
Under Assumption~\ref{ass:jacobian}, $\displaystyle
\|\theta_t-\theta^\star\|_2 \le \frac{\|R_t\|_2}{ct} \le \frac{R}{ct}.$
\end{theorem}

\begin{assumption}[Uniform feature covariance bound]
\label{ass:cov}
There exists $M<\infty$ such that the largest eigenvalue $\lambda_{\max}\bigl(\Cov_{q_t^\theta}(\phi_t(W))\bigr)\le M$
for every relevant $t$ and every $\theta$ on the segment joining $\theta_t$
and $\theta^\star$. This holds whenever $p_t(w)\ge\rho>0$ and $|V_t(w)|\le V_{\max}$,
with $M$ depending only on $\log(1/\rho)$ and $V_{\max}$.
\end{assumption}

\begin{theorem}[Sequence-law and pass@$k$ stability. Proof in Appendix~\ref{app:thm23}]
\label{thm:passk-tracking}
Let $Q$ be the sequence law of the online policies $q_t^{\theta_t}$ and $Q^\star$
that of the offline optimum. Under Assumptions~\ref{ass:jacobian}
and~\ref{ass:cov},
\begin{equation}
\KL(Q\|Q^\star) \le \frac{M\pi^2}{12c^2}R^2,
\qquad
\TV(Q,Q^\star) \le \frac{\pi\sqrt M}{2\sqrt6\,c}\,R,
\label{eq:kl-tv-tracking}
\end{equation}
and, with $u_{\min}\le\min\{u_Q,u_{Q^\star}\}$,
\begin{equation}
\bigl|\passk(Q)-\passk(Q^\star)\bigr|
\;\le\;
\min\Bigl\{1,\; k(1-u_{\min})^{k-1}\,\tfrac{\pi\sqrt M}{2\sqrt6\,c}R\Bigr\}.
\label{eq:passk-tracking}
\end{equation}
\end{theorem}

\begin{remark}[Why KL, not summed TV]
Summing per-step total variation directly gives
$\TV(Q,Q^\star)\le\sum_t\TV(q_t^{\theta_t},q_t^{\theta^\star})$, combining
$\|\theta_t-\theta^\star\|=O(R/t)$ with $\sum_t 1/t=O(\log n)$ for an
$O(R\log n)$ bound. Routing through the KL chain rule instead is sharper: KL
is locally quadratic in the parameter error, so the relevant series is
$\sum_t 1/t^2=O(1)$, giving an $O(R)$ bound with no additional explicit logarithmic dependence
on the horizon in \eqref{eq:kl-tv-tracking}.
Furthermore, 
$R=\max_t\|F_t(\theta^\star)\|_2$ measures the maximal prefix imbalance of the
final offline prices; Theorem~\ref{thm:passk-tracking} is a deterministic
stability statement — small prefix residual implies small pass@$k$ deviation rather than a regret bound in the online-learning sense, since $V_t$ is held
fixed throughout. Note also that if $V_t$ is produced by a bounded head (e.g.,\
a sigmoid output, $V_{\max}=1$), Assumption~\ref{ass:cov}'s constant $M$ is
immediate and does not need to be separately verified.
\end{remark}

\vspace{-0.1in}
\paragraph{An information-theoretic lower bound.}
The preceding guarantees isolate the deterministic price-tracking error for a
fixed potential estimate. Independently of that construction, finite terminal
verifier feedback imposes a statistical limitation on online adaptation within
the tilted family. For an online algorithm $\mathcal A$, let $q_r$ denote the
feasible tilted policy selected at feedback round $r$, and define
$
\Reg_{\mathcal A}(N)
:=
\sum_{r=1}^{N}
\Bigl[\passk(q^\star)-\passk(q_r)\Bigr].
$

\begin{theorem}[Verifier-feedback lower bound. Proof in Appendix~\ref{app:lower}]
\label{thm:lower}
Under the non-degeneracy and richness condition of
Assumption~\ref{ass:lb} of Appendix~\ref{app:lower}, there exists a finite family of feasible tilted
instances such that every online algorithm selecting from that family using
only terminal binary verifier feedback satisfies
$
\sup_{\mathcal I}
\E_{\mathcal I}\!\left[\Reg_{\mathcal A}(N)\right]
\ge c\sqrt{N}
$
for all sufficiently large $N$. % where $c>0$ is independent of $N$.
\end{theorem}
% \end{remark}
\FloatBarrier
\subsection{Calibrated Temperature Decoding (CTD)}\label{sec:calib}
Theorem~\ref{thm:tilt} establishes the form of the optimum but requires the verifier potential. We now show that, under a checkable condition on the base model, the potential is not needed at all: the objective of~\eqref{eq:PE} becomes a function of the cross-entropy, the plausibility ceiling drops out, and the optimum is the tempering of $p$ pinned by the entropy floor. The starting point is the \emph{sequence-level} form. %of the optimum.
\begin{corollary}[Sequence-level optimum]\label{cor:seqtilt}
Telescoping the per-token form of Theorem~\ref{thm:tilt} over $t$, the optimum of \eqref{eq:PE} is
\[
q^\star(w^n)\ \propto\ p(w^n)^{\lambda^\star}\exp\!\big(\mu^\star U(w^n)\big),\qquad p(w^n)=\prod_{t=1}^n p_t(w_t\mid w^{t-1}).
\]
When $\lambda^\star>0$, this is equivalently $q^\star(w^n)\propto\big(p(w^n)e^{\eta U(w^n)}\big)^{1/T}$ with $T=1/\lambda^\star$ and $\eta=\mu^\star/\lambda^\star$.
\end{corollary}
Since the tilt acts on $U$ through the whole sequence, the relevant calibration is between the verdict and the mean log-probability $\bar\ell(w^n):=\tfrac1n\sum_{t}\log p_t(w_t\mid w^{t-1})=\tfrac1n\log p(w^n)$.
\begin{definition}[$(s,\varepsilon)$-calibration]\label{def:calib}
The base model $p$ is $(s,\varepsilon)$-calibrated w.r.t.\ $U$ if there is a single model-global slope $s\ge0$, an intercept $\bar c$, and a uniform residual bound $\varepsilon\ge0$ such that
\begin{equation}\label{eq:calib}
U(w^n) = s\,\bar\ell(w^n) + \bar c + \bar r(w^n),\qquad \sup |\bar r(w^n)|\le\varepsilon.
\end{equation}
\end{definition}
The slope $s$ is a property of the model–verifier pair and not of any prefix. It quantifies how steeply confidence converts into correctness. \eqref{eq:calib} is a linearized, sequence-level instance of Platt scaling~\cite{platt1999probabilistic}. The key observation is that the plausibility functional of~\eqref{eq:PE} is itself a mean log-probability: by the tower property, $-\tfrac1n\sum_t\E_{w^{t-1}\sim q}\big[\sum_w q_t\log p_t\big]=-\E_q[\bar\ell]=:\CE(q)$, and by the chain rule $\bar H(q):=\tfrac1n\sum_t\E[H(q_t)]=\tfrac1n H(q)$. Hence, under~\eqref{eq:calib},
$u_q=\bar c-s\,\CE(q)+\E_q[\bar r]$.
Calibration turns $\passk$ maximization into cross-entropy minimization.

\begin{theorem}[Perfect calibration: the entropy floor pins the temperature. Proof in Appendix~\ref{app:partial}]\label{thm:collapse}
Suppose $p$ is $(s,0)$-calibrated with $s>0$, $p(w^n)$ has a unique mode, and $\beta\in(0,\log|\W|)$. For $a>0$ let $q_a(w^n)\propto p(w^n)^a$ and let $a_\beta$ be the unique root of $\bar H(q_a)=\beta$. Then~\eqref{eq:PE} is feasible iff $\gamma\ge\CE(q_{a_\beta})$, and for every $k$, its solution has \emph{sequence law}
$q^\star(w^n)\ \propto\ p(w^n)^{a_\beta}$.
The ceiling $\gamma$ plays no role beyond feasibility; in the notation of Theorem~\ref{thm:tilt}, $\lambda^\star=0$ and $\mu^\star=na_\beta/s$. The autoregressive conditionals are
\begin{equation}\label{eq:pertoken}
q^\star_t(w\mid w^{t-1})\ \propto\ p_t(w\mid w^{t-1})^{a_\beta}\,\Psi^{(a_\beta)}_t(w^{t-1}w),
\ \ \ 
\Psi^{(a)}_t(w^{t})=\!\!\sum_{w_{t+1:n}\in\W^{\,n-t}}\ \prod_{s=t+1}^{n}p_s(w_s\mid w^{s-1})^{a},
\end{equation}
with $\Psi^{(a)}_n\equiv1$. If the base conditionals are prefix-independent, $\Psi^{(a)}_t$ is token-free and the optimum is pure per-token temperature sampling
\begin{equation}\label{eq:tempopt}
q^\star_t(w\mid w^{t-1}) = \frac{p_t(w\mid w^{t-1})^{a_\beta}}{\sum_{w'} p_t(w'\mid w^{t-1})^{a_\beta}}.
\end{equation}
\end{theorem}

Two consequences arise: First, $k$ enters only through $\beta$: with $\beta^\star(k)$ of Proposition~\ref{prop:mono}, the effective temperature is
$T_{\mathrm{eff}}^\star(k)=1/a_{\beta^\star(k)}$, and since
$\beta\mapsto a_\beta$ is strictly decreasing, a non-decreasing
$\beta^\star(k)$ induces a non-decreasing $T_{\mathrm{eff}}^\star(k)$
in the calibrated regime. Second, only the \emph{sign} of $s$ enters the policy, not its magnitude.

\begin{remark}[The collapse is sequence-level]\label{rem:seqvspertoken}
$p(w^n)^a=\prod_t p_t(w_t\mid w^{t-1})^a$ factorizes, but its normalization does not: per-token temperature drops exactly the correction $\log\Psi^{(a)}_t(w^{t-1}w)$ of~\eqref{eq:pertoken}, which is token-constant under prefix independence and nearly so under low continuation variation.
\end{remark}

\begin{corollary}[Limited continuation variation]\label{cor:psivar}
Let $q^{\mathrm{temp}}_{a,t}(w\mid h)\propto p_t(w\mid h)^a$ and define
$
\delta_t(h):=
\max_{w}\log\Psi_t^{(a)}(hw)-\min_{w}\log\Psi_t^{(a)}(hw),
$
with
$\delta_t:=\sup_h\delta_t(h)$.
Under perfect calibration,
\[
\TV\!\left(q_t^\star(\cdot\mid h),q^{\mathrm{temp}}_{a,t}(\cdot\mid h)\right)
\le \frac12\!\left(e^{\delta_t(h)}-1\right),
\]
and hence
$
\TV(q^\star,q_a^{\mathrm{temp}})
\le \frac12\sum_{t=1}^n\!\left(e^{\delta_t}-1\right).
$
Therefore, the $\passk$ gap is at most
\[
\Delta_{\Psi,k}:=
\big|\passk(q^\star)-\passk(q_a^{\mathrm{temp}})\big|
\le
\min\!\left\{1,\frac{k}{2}\sum_{t=1}^n\!\left(e^{\delta_t}-1\right)\right\}
=O\!\left(k\sum_t\delta_t\right)
\]
for small continuation variation. Prefix-independent conditionals correspond  to $\delta_t=0$.
\end{corollary}

\begin{theorem}[Partial calibration: Proof in Appendix~\ref{app:partial}]\label{thm:partial}
Under $(s,\varepsilon)$-calibration, let $\tilde q:=q_{a_\beta}$ be the tempering of Theorem~\ref{thm:collapse} and suppose $\gamma\ge\CE(\tilde q)$. Then
(i) $0\le u_{q^\star}-u_{\tilde q}\le \E_{q^\star}[\bar r]-\E_{\tilde q}[\bar r]\le2\varepsilon$;
(ii) $\big|\passk(q^\star)-\passk(\tilde q)\big|\le2\varepsilon\,\sigma_k^{+}$, with $\sigma_k^{+}:=k(1-u_{q^\star}+2\varepsilon)^{k-1}$;
(iii) if $2\varepsilon<u_{q^\star}$, $\sigma_k^{+}$ is strictly unimodal in $k$ with peak near $1/(u_{q^\star}-2\varepsilon)$, so explicit residual tilting pays off only in a bounded window of $k$.
\end{theorem}

\begin{algorithm}[!htbp]
\caption{Calibrated Temperature Decoding (CTD): certified fast path with verifier-aware fallback.}
\label{alg:ctd}
\begin{algorithmic}[1]
\Require base model $p$; budgets $(\beta,\gamma)$; sample budget $k$; calibration set $\{(x_i,w^n_i,U_i)\}_{i=1}^N$; calibration-loss tolerance $\delta_{\rm cal}$; confidence level $\alpha$; $\sigma_k^{+}:=k(1-u_{q^\star}+2\varepsilon)^{k-1}$
\Statex \textbf{Phase 1 (offline, once per domain: certificate).}
\State for each $i$ compute $\bar\ell_i=\frac1n\sum_t\log p_t(\cdot)$ \Comment{one forward pass}
\State pooled OLS: $\shat\leftarrow$\eqref{eq:ols}; $\hat\varepsilon\leftarrow$ 95\% quantile of $|U-\shat\bar\ell-\hat{\bar c}|$ on a held-out split
\State form a Lower Confidence Bound $\mathrm{LCB}_\alpha(s)$ from the concentration bound of Theorem~\ref{thm:ctd}
\State $\hat u\leftarrow \frac1N\sum_i U_i$; $\hat\sigma_k^+\leftarrow k(1-\hat u+2\hat\varepsilon)^{k-1}$
\Statex \textbf{Phase 2 (once per $\beta$).}
\State solve $\hat H(a)=\beta$ by bisection, $\hat H(a)$ the occupancy-weighted per-token entropy of $q^{\mathrm{temp}}_a$ on calibration rollouts; set $\hat a$
\State check deployment feasibility $\CE(q^{\mathrm{temp}}_{\hat a})\le\gamma$; if it fails, do not certify the CTD fast path
\Statex \textbf{Phase 3 (generation, per prompt).}
\State \textbf{if} $\mathrm{LCB}_\alpha(s)>0$ \textbf{and} $2\hat\varepsilon\hat\sigma_k^+\le\delta_{\rm cal}$ \textbf{then}
\State \quad \textbf{3a (CTD):} sample $w_t\sim q_t$, $q_t(w)\propto p_t(w\mid w^{t-1},x)^{\hat a}$ \Comment{plain temperature cost}
\State \textbf{else}
\State \quad \textbf{3b (fallback):} use a verifier-aware decoder that retains explicit continuation information \Comment{e.g., \textsc{TMAD-Gated} in Section~\ref{sec:experiments}}
\State \textbf{end if}
\end{algorithmic}
\end{algorithm}
Algorithm~\ref{alg:ctd} uses the regression only as a \emph{certificate}: the Lower Confidence Bound (LCB) on $s$ checks positive likelihood--correctness alignment, while $\hat\varepsilon$ is an empirical residual diagnostic used by the deployment gate; the exponent itself is set by the entropy floor. Note that $\hat\varepsilon$ is a practical residual diagnostic, not a certified estimator of the worst-case $\varepsilon$ in Definition~1; Theorems~6--7 are stated in terms of the true $\varepsilon$, with no finite-sample guarantee on $\hat\varepsilon-\varepsilon$. Phase~3a is the zero-overhead CTD fast path. Phase~3b is used when the calibration certificate fails and hands decoding to a verifier-aware fallback that retains explicit continuation information; Section~\ref{sec:experiments} instantiates this fallback with \textsc{TMAD-Gated} (Algorithm~\ref{alg:tmad-gated}). Such a decoder can model residual signal that pure temperature omits. Concretely, drawing $N$ independent prompt--completion pairs $(x_i,w_i^n)$ with $x_i\sim\mathcal D$ and $w_i^n\sim p(\cdot\mid x_i)$, with labels $U_i\in\{0,1\}$ and scores $\bar\ell_i$ computed in one pass; OLS gives
\begin{equation}\label{eq:ols}
\shat = \frac{\sum_i(\bar\ell_i-\bar{\bar\ell})(U_i-\bar U)}{\sum_i(\bar\ell_i-\bar{\bar\ell})^2},\qquad
\hat{\bar c}=\bar U-\shat\,\bar{\bar\ell}.
\end{equation}
Phase~3a is ordinary temperature sampling, so CTD fast path adds no generation-time overhead.
\begin{proposition}[Sequence-level identification of $s$]\label{prop:ident}
Under~\eqref{eq:calib} the population OLS slope of $U$ on $\bar\ell$ is $s+O(\varepsilon)$, and $\shat$ concentrates around it at rate $O(1/\sqrt{N})$, with no step-level value estimation.
\end{proposition}
\begin{assumption}[Calibration-set regularity]\label{ass:calib}
The $N$ calibration pairs $(\bar\ell_i,U_i)$ are i.i.d.\ with $U\in\{0,1\}$ and $(s,\varepsilon)$-calibration~\eqref{eq:calib}. Score is bounded, $|\bar\ell|\le B:=\log(1/\rho)$, and  $s^2_{\mathrm{spr}}:=\Var(\bar\ell)\ge s^2_0>0$.
\end{assumption}
\begin{theorem}[Suboptimality of CTD. Proof in Appendix~\ref{app:ctd}]\label{thm:ctd}
Under $(s,\varepsilon)$-calibration, let $\hat q_{\mathrm{CTD}}:=q^{\mathrm{temp}}_{\hat a}$ be a Phase~3a policy satisfying $\bar H(\hat q_{\mathrm{CTD}})\ge\beta$ and $\CE(\hat q_{\mathrm{CTD}})\le\gamma$, and define the structural cross-entropy excess $\Delta_{\CE}:=\CE(\hat q_{\mathrm{CTD}})-\CE(q_{a_\beta})\ge0$. Then, deterministically,
\begin{equation}\label{eq:ctdbound}
0\le\passk(q^\star)-\passk(\hat q_{\mathrm{CTD}})\le\sigma_k^{++}\big(\underbrace{s\,\Delta_{\CE}}_{\text{continuation}}+\underbrace{2\varepsilon}_{\text{calibration floor}}\big),
\end{equation}
where $\sigma_k^{++}:=k\big(1-u_{q^\star}+s\Delta_{\CE}+2\varepsilon\big)^{k-1}$
Under Assumption~\ref{ass:calib} with $s>C_0\varepsilon$, Phase~1 certifies with probability at least $1-8\exp\!\big(-N(s-C_0\varepsilon)^2/C_s^2\big)$, where $C_0,C_s$ depend only on $(B,s_0)$.
\end{theorem}
The proof mirrors Theorem~\ref{thm:partial}(i): $\hat q_{\mathrm{CTD}}$ meets the floor with equality, and $\CE(q^\star)\ge\CE(q_{a_\beta})$. Relative to a tilted-exponent analysis, three things disappear: no dual price, no horizon factor, and no $O(1/\sqrt{N})$ estimation term---the regression certifies but does not tune, so its error enters only through the exponentially small probability of a wrong sign. The structural term $s\Delta_{\CE}$ is the cross-entropy price of dropping $\Psi^{(a)}$ at fixed entropy; it vanishes under prefix independence ($\delta_t\equiv0$ in Corollary~\ref{cor:psivar}), where $q^{\mathrm{temp}}_a=q_a$ for every $a$. The prediction is that in the calibrated regime, the right temperature is the entropy-matched one, and CTD should track tuned temperature sampling rather than beat it. Failure of the per-token feasibility check in Phase~2 means only that the CTD deployment approximation is not certified; exact infeasibility is the sequence-level condition of Theorem~\ref{thm:collapse}.

\begin{remark}[CTD versus explicit online estimation]\label{sec:which}
\textsc{CTD} certifies whether verifier outcomes are explained by the sequence log-probability; when they are, the potential is redundant, and temperature at the entropy-matched exponent is near-optimal at no generation-time cost. \textsc{J-EOS} assumes no such reduction: it estimates the continuation potential from terminal feedback while tracking both prices, at the cost of rollouts. Since our \textsc{J-EOS} analysis covers price tracking for a fixed potential estimate rather than the coupled procedure, we treat \textsc{J-EOS} as an analytical construction, and \textsc{CTD} plus a verifier-aware fallback. %as the practical design pattern. 
Section~\ref{sec:experiments} instantiates this fallback with practical algorithm \textsc{TMAD-Gated}.
\end{remark}

\FloatBarrier
\section{Experimental Results}\label{sec:experiments}

We test the two practical decoders derived from theory on code generation.
CTD tests the calibrated regime, where temperature alone should suffice
(Theorem~\ref{thm:collapse}), while \textsc{TMAD-Gated} (Algorithm~\ref{alg:tmad-gated})
tests whether an explicit continuation-success proxy for the potential $V_t$ adds
value beyond temperature (Theorem~\ref{thm:tilt}). We compare both against fixed-temperature sampling and AdapT, an adaptive
temperature baseline reviewed in Appendix~\ref{app:related}.

\textsc{TMAD-Gated} (Algorithm~\ref{alg:tmad-gated}, Appendix~\ref{app:exp-details})
keeps the tilted-temperature family $(T,\eta)$ of Theorem~\ref{thm:tilt} stochastic
at intervention points. The gate combines model confidence with structurally
critical code positions; active steps use a learned continuation-success proxy and
short lookahead to select a local regime by an estimated-success/plausibility score.

\begin{table}[!htbp]
\centering
\caption{Pass@$k$ (\%) on HumanEval and MBPP, reported as
mean $\pm$ standard deviation over independent runs.
Within each run, methods use the same held-out subset.
\textbf{Bold} marks the best mean in each column.}
\label{tab:main-results}

\footnotesize
\setlength{\tabcolsep}{2pt}
\renewcommand{\arraystretch}{0.76}

\begin{tabular}{llccc|ccc}
\toprule
& &
\multicolumn{3}{c|}{HumanEval}
& \multicolumn{3}{c}{MBPP} \\
\cmidrule(lr){3-5}\cmidrule(lr){6-8}
Model & Method
& pass@5 & pass@10 & pass@15
& pass@5 & pass@10 & pass@15 \\
\midrule

\multirow{4}{*}{Qwen2.5-1.5B}
& Fixed-$T$
& $46.2 \pm 5.8$
& $59.7 \pm 7.8$
& $64.8 \pm 7.7$
& $36.7 \pm 1.1$
& $52.0 \pm 2.2$
& $61.0 \pm 2.2$ \\

& CTD
& $45.1 \pm 5.5$
& $59.7 \pm 7.0$
& $68.0 \pm 7.7$
& $38.2 \pm 1.4$
& $56.1 \pm 7.2$
& $62.7 \pm 12.2$ \\

& AdapT
& $46.2 \pm 2.3$
& $59.7 \pm 3.5$
& $66.4 \pm 6.6$
& $41.5 \pm 3.2$
& $58.3 \pm 1.1$
& $67.6 \pm 1.1$ \\

& TMAD-Gated
& $\mathbf{47.6 \pm 4.5}$
& $\mathbf{65.3 \pm 3.3}$
& $\mathbf{72.2 \pm 3.4}$
& $\mathbf{54.5 \pm 0.6}$
& $\mathbf{69.7 \pm 2.3}$
& $\mathbf{76.6 \pm 2.2}$ \\

\midrule

\multirow{4}{*}{Qwen3-4B}
& Fixed-$T$
& $54.4 \pm 4.5$
& $63.5 \pm 7.6$
& $67.8 \pm 6.8$
& $61.6 \pm 0.6$
& $71.2 \pm 0.1$
& $\mathbf{75.8 \pm 1.1}$ \\

& CTD
& $50.6 \pm 7.6$
& $62.6 \pm 10.5$
& $68.4 \pm 4.6$
& $59.7 \pm 2.5$
& $68.3 \pm 1.5$
& $72.7 \pm 1.1$ \\

& AdapT
& $52.6 \pm 5.0$
& $62.5 \pm 4.0$
& $69.1 \pm 5.6$
& $60.9 \pm 0.2$
& $70.5 \pm 0.3$
& $74.2 \pm 1.1$ \\

& TMAD-Gated
& $\mathbf{55.4 \pm 7.5}$
& $\mathbf{66.5 \pm 10.0}$
& $\mathbf{71.9 \pm 9.9}$
& $\mathbf{66.3 \pm 1.9}$
& $\mathbf{71.8 \pm 1.2}$
& $74.2 \pm 1.1$ \\

\midrule

\multirow{4}{*}{CodeLlama-7B}
& Fixed-$T$
& $23.3 \pm 4.2$ & $34.2 \pm 3.1$ & $40.2\pm5.8$
& $26.0 \pm 2.1$ & $41.3\pm 1.2$ & $49.0 \pm 1.4$ \\

& CTD
& $22.5\pm 3.1$ & $34.2\pm 4.1$ & $41.3\pm 4.2$
& $27.3\pm0.5$ & $40.2\pm0.8$ & $47.6\pm1.0$ \\

& AdapT
& $24.1\pm1.6$ & $35.2\pm2.2$ & $\mathbf{43.2\pm1.8}$
& $28.0\pm 1.3$ & $42.3\pm1.0$ & $50.5\pm 1.2$ \\

& TMAD-Gated
& $\mathbf{25.2\pm4.8}$ & $\mathbf{36.0\pm3.9}$ & $42.3\pm4.0$
& $\mathbf{30.3\pm1.7}$ & $\mathbf{44.0\pm1.2}$ & $\mathbf{52.1\pm1.4}$ \\

\bottomrule
\end{tabular}
\end{table}

\vspace{-0.1in}
\paragraph{TMAD-Gated.}
A lightweight continuation-success head is trained on fixed-temperature rollouts.
At active positions, short lookahead refines its score for the highest-probability
candidates, giving a practical proxy $\hat S_t$ for the Bellman potential $V_t$
(see (\ref{eq-S}) in Appendix~\ref{app:exp-details}). We sample from
$
q_{T,\eta}(w\mid w^{t-1})
\propto
\bigl(p_t(w\mid w^{t-1})e^{\eta\hat S_t(w\mid w^{t-1})}\bigr)^{1/T},
$
while inactive positions use conservative low-temperature sampling.

\vspace{-0.1in}
\paragraph{Evaluation protocol and results.}
We evaluate HumanEval and MBPP with Qwen2.5-Coder-1.5B, Qwen3-4B-Base, and
CodeLlama-7B, reporting pass@$k$ for $k\in\{5,10,15\}$. Fixed-$T$'s temperature
is chosen via a grid search over $T\in[0.2,0.8]$ maximizing pass@5, reused at
pass@10 and pass@15. For each benchmark, we randomly partition the problems into
a training pool (for calibration and head training) and a disjoint test set of 64
problems, on which all four methods are evaluated with the same pipeline (details
in Appendix~\ref{app:exp-details}; Figure~\ref{fig:passk-main-curves} visualizes
the Qwen results). \textsc{TMAD-Gated} is best in 16 of 18 settings, with its
clearest gains on MBPP. As predicted by Theorem~\ref{thm:ctd}, \textsc{CTD}
generally tracks tuned temperature sampling rather than consistently beating it,
indicating that explicit continuation information provides the main additional gains.

\textbf{Concluding Remarks and Discussion.}
We showed that pass@$k$ maximization under an entropy floor and a cross-entropy
ceiling is solved by the tilted temperature $q^\star \propto (p\,e^{\eta V})^{1/T}$,
with ordinary temperature sampling as the calibrated special case. We present both online and offline solutions to this problem and validate them on three models and two code benchmarks.
The remaining open
problem relates to the estimation of the verifier potential $V_t$ at scale, since better potential
estimators are the most direct route to closing the gap to the optimum.

\FloatBarrier

\bibliographystyle{plainnat}
\bibliography{ref-ctd}

\renewcommand{\Pr}{\mathbb{P}}
\newcommand{\sgn}{\operatorname{sgn}}
\renewcommand{\qedsymbol}{$\blacksquare$}
\newcommand{\pass}{\mathrm{pass@}k}
\newcommand{\half}{\tfrac12}

\appendix
\section*{Appendix}
\addcontentsline{toc}{section}{Appendix}

\section{Expanded Related Work}
\label{app:related}

A recent line of work treats the number of samples drawn at decoding as a
resource to be scaled, in analogy with parameters and pretraining tokens.
\citet{brown2024monkeys} show that coverage---the fraction of problems solved
by \emph{any} of $k$ samples, i.e.\ $\mathrm{pass@}k$---grows log-linearly over
several orders of magnitude in $k$, and \citet{wu2024inference} chart the
compute-optimal frontier between model size and sample count.
\citet{snell2024scaling} make the allocation adaptive, spending more
inference compute on harder prompts. These studies establish
$\mathrm{pass@}k$ as the operational objective but treat the sampling
distribution as fixed (typically the base model at a hand-tuned temperature);
our contribution is to ask which decoding distribution \emph{maximizes}
$\mathrm{pass@}k$ under an explicit compute budget, and to show that the
answer is the tilted-temperature as in ~\eqref{eq:opt1}.

\textit{Verifiers and selection.}
A complementary thread improves the \emph{selection} step rather than the
proposal distribution. Outcome-level verifiers trained on solution
correctness \citep{cobbe2021gsm8k}, process reward models that score
intermediate steps \citep{lightman2024verify}, and best-of-$n$ reranking
against a learned reward \citep{nakano2021webgpt} all re-weight a fixed pool
of samples. Our verifier signal plays a different role: rather than being used only to
select among completed samples, it enters the proposal distribution through
the continuation potential. In \textsc{J-EOS}, this potential is estimated
from terminal verifier feedback, while under the calibration conditions of
Section~\ref{sec:calib} its effect can be absorbed into the tempered
sequence law. Self-consistency
\citep{wang2023selfconsistency} can be read as a verifier-free aggregator and
is recovered as a special case when the verifier is uninformative.

\textit{Temperature and truncation decoding.}
Heuristic decoding modifications---nucleus sampling
\citep{holtzman2020degeneration} and truncation-as-desmoothing
\citep{hewitt2022truncation}---reshape the next-token distribution to trade
diversity against quality. These transformations are justified empirically; we instead derive the
optimal reshaping from a constrained $\passk$ objective. Under perfect
calibration the optimum reduces to a pure power of the base sequence law,
and ordinary per-token temperature sampling closely approximates this
optimum when the remaining-horizon correction has limited variation across
candidate tokens at each prefix.

\textit{Tilted distributions and KL-regularized control.}
The optimum $q^\star\propto\bigl(p\,e^{\eta V}\bigr)^{1/T}$ is an
exponentially tilted (Gibbs) distribution, the same form that arises from the
maximum-entropy principle \citep{jaynes1957information,cover2006elements} and
from KL-regularized control. Casting decoding as inference in a soft-optimal
MDP \citep{ziebart2008maxent,levine2018rlinference} and viewing the KL penalty
in RLHF as Bayesian conditioning \citep{korbak2022rlbayesian} both produce
tilted policies; \citet{rafailov2023dpo} exploit the closed form of the tilt to
bypass explicit reward modeling. We share the tilted-distribution structure but
work at inference time under a cross-entropy (compute) budget rather than a
training-time KL penalty, which is what pins the temperature $T=1/\lambda^\star$
to the budget.

\textit{Calibration.}
Our reduction from verifier estimation to temperature rests on calibration of
the base model's log-probabilities. Temperature scaling and Platt scaling
\citep{platt1999probabilistic,guo2017calibration} are the standard
post-hoc calibrators we adopt for the affine model of Section~\ref{sec:calib}; calibration
of pretrained transformers \citep{desai2020calibration} and the finding that
language models are partly aware of their own correctness
\citep{kadavath2022know} motivate the $(s,\varepsilon)$-calibration assumption.
Theorem~\ref{thm:partial} quantifies how the residual
miscalibration $\varepsilon$ degrades $\passk$, and
Theorem~\ref{thm:ctd} the finite-sample cost of estimating
the calibration slope $s$.

\textbf{Literature on practical algorithms.} A separate line of work adapts the sampling rule online rather than fixing it in
advance. Mirostat \citep{basu2021mirostatneuraltextdecoding} targets a fixed
perplexity level, while locally typical sampling \citep{meister-etal-2023-locally}
and $\eta$-sampling \citep{hewitt2022truncation} truncate using token-level
information content or entropy. More recent methods adapt using confidence or
structure: AdapT \citep{zhu2024hot} raises sampling temperature at line-initial
tokens, while AdaDec \citep{he2026adadecuncertaintyguidedlookaheaddecoding}
combines an entropy threshold with lookahead reranking. Its final token selection
is deterministic, so repeated decoding adds no sampling diversity and pass@$k$
collapses to pass@$1$. ARGS \citep{ICLR2024_4d134431} similarly performs
reward-guided greedy selection, and Uncert-CoT
\citep{zhu2025uncertaintyguidedchainofthoughtcodegeneration} triggers multi-path
reasoning above an uncertainty threshold.

\section{Proof of Lemma 1 (Bellman Consistency)}
\label{app:lem1}

Recall from Section 2 that, for prices $\lambda^\star\ge0$ and $\mu^\star>0$, a step
$t\in\{1,\dots,n\}$, a prefix $w^{t-1}$ and a token $w$ (so $w_t=w$),
\begin{equation}\label{eq:B-Vdef}
  V_t(w\mid w^{t-1})\;=\;\frac{1}{\mu^\star}\log\!\!
  \sum_{w_{t+1:n}\in\W^{\,n-t}}\ \prod_{s=t+1}^{n}p_s(w_s\mid w^{s-1})^{\lambda^\star}\,
  e^{\mu^\star U(w^n)},\qquad w^n=(w^{t-1},w,w_{t+1:n}),
\end{equation}
with the convention that at $t=n$ the continuation set $\W^0$ is the single empty tuple
and the product is empty. It is convenient to name the partition sum inside the
logarithm: for a length-$t$ prefix $w^t$ $(1\le t\le n)$ set
\begin{equation}\label{eq:B-Phi}
  \Phi(w^t)\;:=\;\sum_{w_{t+1:n}\in\W^{\,n-t}}\ \prod_{s=t+1}^{n}
  p_s(w_s\mid w^{s-1})^{\lambda^\star}\, e^{\mu^\star U(w^n)},
  \qquad \Phi(w^n):=e^{\mu^\star U(w^n)},
\end{equation}
so that \eqref{eq:B-Vdef} reads $\Phi(w^t)=\exp\!\big(\mu^\star V_t(w_t\mid w^{t-1})\big)$.

\medskip
\noindent\textbf{Lemma 1 (Bellman consistency, restated).}\ \emph{For any
$\lambda^\star\ge0$, $\mu^\star>0$, step $t\in\{1,\dots,n-1\}$, prefix $w^{t-1}$ and token
$w$,}
\[
  V_t(w\mid w^{t-1})=\frac{1}{\mu^\star}\log\sum_{w'\in\W}
  p_{t+1}(w'\mid w^{t-1}w)^{\lambda^\star}\exp\!\big(\mu^\star V_{t+1}(w'\mid w^{t-1}w)\big),
\]
\emph{with terminal condition $V_n(w\mid w^{n-1})=U(w^{n-1},w)$; equivalently
$V_t=\tfrac{T}{\eta}\log Z_{t+1}(w^{t-1}w)$ with}
\[
  Z_{t+1}(w^{t-1}w)=\sum_{w'\in\W}
  \Big(p_{t+1}(w'\mid w^{t-1}w)\,e^{\eta V_{t+1}(w'\mid w^{t-1}w)}\Big)^{1/T}.
\]
\emph{The recursion involves only $(p,U,\lambda^\star,\mu^\star)$, so $V$ exists, is
unique, and is obtained by one backward sweep over prefix states.}

\begin{proof}
\textbf{Well-posedness.} Fix $t$ and $w^t$. Since $\rho\le p_s(\cdot\mid w^{s-1})\le1$ and
$U\in[0,1]$, every summand of \eqref{eq:B-Phi} lies in
$\big[\rho^{\lambda^\star(n-t)},\,e^{\mu^\star}\big]$, and the index set $\W^{\,n-t}$ is
finite and nonempty. Hence
\begin{equation}\label{eq:B-bounds}
  0\;<\;\rho^{\lambda^\star(n-t)}\;\le\;\Phi(w^t)\;\le\;|\W|^{\,n-t}e^{\mu^\star}\;<\;\infty,
\end{equation}
so $V_t=(\mu^{\star})^{-1}\log\Phi(w^t)$ is finite for every $(t,w^{t-1},w)$, with
\[
  -\frac{\lambda^\star(n-t)}{\mu^\star}\log\frac1\rho\;\le\;V_t(w\mid w^{t-1})\;\le\;
  1+\frac{n-t}{\mu^\star}\log|\W| .
\]
(The two-sided bound \eqref{eq:B-bounds} is what makes the log-space form legitimate.)

\smallskip
\textbf{Recursion.} Fix $t\in\{1,\dots,n-1\}$ and a length-$t$ prefix $w^t=(w^{t-1},w)$.
Separate the first continuation token $w_{t+1}=w'$ and factor out
$p_{t+1}(w'\mid w^t)^{\lambda^\star}$, which is constant in $w_{t+2:n}$:
\begin{align*}
  \Phi(w^t)&=\sum_{w'\in\W}p_{t+1}(w'\mid w^t)^{\lambda^\star}
  \underbrace{\sum_{w_{t+2:n}}\ \prod_{s=t+2}^{n}p_s(w_s\mid w^{s-1})^{\lambda^\star}
  e^{\mu^\star U(w^n)}}_{=\ \Phi(w^t w')}\\
  &\;=\;\sum_{w'\in\W}p_{t+1}(w'\mid w^t)^{\lambda^\star}\,\Phi(w^tw'),
\end{align*}
where the brace equals $\Phi(w^tw')$ because for $s\ge t+2$ the factors
$p_s(w_s\mid w^{s-1})$ and the verifier $U(w^n)$ depend on $w_{t+1}$ only through the
now-fixed conditioning prefix $w^tw'$. Substituting
$\Phi(w^tw')=e^{\mu^\star V_{t+1}(w'\mid w^t)}$ and taking $\mu^{\star-1}\log$ yields the
stated recursion. At $t=n$ the empty-product convention gives
$\Phi(w^n)=e^{\mu^\star U(w^n)}$, hence $V_n(w\mid w^{n-1})=U(w^{n-1},w)$.

\smallskip
\textbf{Equivalent $(T,\eta)$ form.} With $\lambda^\star=1/T$ and $\mu^\star=\eta/T$
(so $\eta>0$ whenever $\mu^\star>0$),
\[
  p_{t+1}(w'\mid w^{t-1}w)^{\lambda^\star}e^{\mu^\star V_{t+1}(w'\mid w^{t-1}w)}
  =\Big(p_{t+1}(w'\mid w^{t-1}w)\,e^{\eta V_{t+1}(w'\mid w^{t-1}w)}\Big)^{1/T},
\]
so the sum in the recursion is exactly $Z_{t+1}(w^{t-1}w)$ --- the per-node normalizer of
the tilted-temperature step of Theorem 1 --- and $\mu^{\star-1}=T/\eta$ gives
$V_t=\tfrac{T}{\eta}\log Z_{t+1}(w^{t-1}w)$.

\smallskip
\textbf{Existence, uniqueness, backward sweep.} The closed form \eqref{eq:B-Vdef}
exhibits $V$ explicitly, so it exists. The recursion expresses $V_t$ through $V_{t+1}$
alone and the terminal condition fixes $V_n=U$; reading the layers
$t=n,n-1,\dots,1$ determines every $V_t$ exactly once. Any $V'$ obeying the recursion and
the terminal condition equals $V$ by downward induction on $t$ (they agree at $t=n$; if
they agree at $t+1$, the recursion forces agreement at $t$), giving uniqueness. As the
right-hand side involves only $(p,U,\lambda^\star,\mu^\star)$ and no policy $q$, $V$ is a
deterministic function of the instance and the two prices, computed by a single backward
pass of $n$ layers with no self-consistency condition to solve.
\end{proof}

\paragraph{$V_t$ is a free energy, not an ordinary conditional average.}
The potential $V_t$ is a log-partition quantity over tilted continuations rather than
an ordinary conditional expectation of the terminal verifier. This distinction is
important because the optimal decoder of Theorem 1 tilts directly on $V_t$. Four
facts summarize its structure.

\emph{(i) At the terminal step $V_t$ reduces to the verifier.}
When $t=n$ the continuation is empty and
\[
  V_n(w\mid w^{n-1})=U(w^{n-1},w).
\]

\emph{(ii) $V_t$ is a log-partition function.} Fix a committed pair
$w^t=(w^{t-1},w)$, let
$S_t(w^t):=\sum_{w_{t+1:n}}\prod_{s>t}p_s(w_s\mid w^{s-1})^{\lambda^\star}$
be the total tempered mass of its continuations, and let
$\nu_t(\cdot\mid w^t)$ be the probability measure on $\W^{n-t}$ proportional to
$\prod_{s>t}p_s(w_s\mid w^{s-1})^{\lambda^\star}$. Then \eqref{eq:V-def} reads
\begin{align}
\label{eq:V-logpart}
  V_t(w\mid w^{t-1})
  &\;=\;\frac{1}{\mu^\star}\log S_t(w^t)\;+\;\frac{1}{\mu^\star}\log
  \E_{\nu_t}\!\big[e^{\mu^\star U(w^n)}\big]\\
  &\;=\;\frac{1}{\mu^\star}\log S_t(w^t)+\E_{\nu_t}[U]
  +\frac{\mu^\star}{2}\Var_{\nu_t}(U)+O(\mu^{\star2}). 
  \nonumber
\end{align}
When $\lambda^\star=1$ (equivalently $T=1$), $\nu_t$ is the base-model
continuation law and $S_t\equiv1$, so as $\mu^\star\to0$,
\[
  V_t(w\mid w^{t-1})
  \longrightarrow
  \E_{w_{t+1:n}\sim p}\!\big[U(w^n)\mid w^t\big].
\]
At any other temperature $S_t$ need not equal one and the first term in
\eqref{eq:V-logpart} remains. This is harmless for the decoder: only the product
$\mu^\star V_t$ enters the optimum, and by \eqref{eq:V-logpart} it converges to
$\log S_t$ as $\mu^\star\downarrow0$.

\emph{(iii) $V_t$ is not a probability.} From \eqref{eq:V-def},
$p_s^{\lambda^\star}\le1$, $p_s\ge\rho$, and $U\in[0,1]$ give the explicit range
\[
  -\frac{\lambda^\star (n-t)}{\mu^\star}\,\log\frac1\rho
  \;\le\; V_t(w\mid w^{t-1})\;\le\;
  1+\frac{(n-t)}{\mu^\star}\,\log|\W| ,
\]
so $V_t$ need not lie in $[0,1]$.

\begin{remark}[Degenerate prices, and what $\eta=0$ actually gives]\label{rem:B-degenerate}
The lemma assumes $\mu^\star>0$, equivalently $\eta>0$. At $\mu^\star=0$ the definition
\eqref{eq:B-Vdef} is a $0/0$ and $V_t$ is undefined; what survives is the product
$\mu^\star V_t$, which is the only way $V_t$ enters the optimum. Taking the limit $\mu^\star\downarrow 0$ in
\eqref{eq:B-Phi} gives $\mu^\star V_t\to\log S_t(w^t)$, with
$S_t(w^t)=\sum_{w_{t+1:n}}\prod_{s>t}p_s(w_s\mid w^{s-1})^{\lambda^\star}$ the tempered
continuation mass of Section 2, so as $\mu^\star\downarrow0$ the tilted step of Theorem 1
becomes
\[
  q^\star_t(w\mid w^{t-1})\ \propto\ p_t(w\mid w^{t-1})^{1/T}\;S_t(w^{t-1}w).
\]
This is temperature sampling at the \emph{sequence} level --- telescoping (Corollary~\ref{cor:seqtilt})
gives $q^\star(w^n)\propto p(w^n)^{1/T}$ --- but it is \emph{not} per-token temperature
sampling: the extra factor $S_t(w^{t-1}w)$ steers toward branches whose continuation
carries cheap entropy, which is exactly the effect Corollary~\ref{cor:freeenergy} isolates. The two coincide
iff $S_t(w^{t-1}w)$ does not depend on $w$, i.e.\ at $n=1$, at $T=1$, or when the base
conditionals are prefix-independent (the i.i.d.-token regime of \cite{sima2026}). At $n\ge2$ with prefix-dependent conditionals
the factor is not constant.
\end{remark}

\section{Full proof of Theorem~\ref{thm:tilt}}\label{app:proofTh1}
 Here we give it in full, in
particular the passage from the Lagrangian to the per-step cost-to-go the recursion (\ref{eq:app-J}) below, which is the principle of optimality specialised to the
entropy-regularised program \eqref{eq:PE}, and the strong-duality step that
certifies the resulting stationary point as the constrained optimum.

Throughout, a policy is $q=(q_1,\dots,q_n)$ with
$q(w^n)=\prod_{t=1}^n q_t(w_t\mid w^{t-1})$, and
$q(w^{t-1}):=\prod_{s<t} q_s(w_s\mid w^{s-1})$ is the occupancy of the prefix
$w^{t-1}$ induced by $q$. We write $\langle q_t,c\rangle:=\sum_w q_t(w)c(w)$.

\subsection{Step 0: Lagrangian relaxation}

Since $\passk(q)=1-(1-u_q)^k$ is strictly increasing in $u_q:=\E_{w^n\sim q}[U(w^n)]$,
program \eqref{eq:PE} is equivalent to maximising $u_q$ under the same two
constraints. Attach multipliers $\mu_E\ge0$ to the entropy floor
$\frac1n\sum_t\E_{w^{t-1}\sim q}[H(q_t)]\ge\beta$ and $\lambda_0\ge0$ to the
plausibility ceiling $-\frac1n\sum_t\E_{w^{t-1}\sim q}[\sum_w q_t\log p_t]\le\gamma$.
The Lagrangian is
\[
\mathcal L_0(q)
= u_q
+\mu_E\Big(\tfrac1n\textstyle\sum_t\E[H(q_t)]-\beta\Big)
-\lambda_0\Big(-\tfrac1n\textstyle\sum_t\E[\sum_w q_t\log p_t]-\gamma\Big).
\]
Dropping the additive constant $-\mu_E\beta+\lambda_0\gamma$ (which does not affect
the maximiser in $q$) and renaming $\lambda:=\lambda_0/n$,
\begin{align}\label{eq:app-lag}
\mathcal L(q)=&E_{w^n\sim q}[U(w^n)]
+\sum_{t=1}^n\E_{w^{t-1}\sim q}\big[\rho_t(w^{t-1},q_t)\big],\\
\rho_t(w^{t-1},q_t):=&\tfrac{\mu_E}{n}H\big(q_t(\cdot\mid w^{t-1})\big)
+\lambda\!\sum_w q_t(w\mid w^{t-1})\log p_t(w\mid w^{t-1}),
\end{align}
which is the Lagrangian displayed in the main text. It is a finite-horizon control
problem: terminal reward $U(w^n)$ at the leaves and per-step reward $\rho_t$, a
concave functional of the action distribution $q_t$. The $1/n$ rescaling is
``harmless'' in the precise sense that it only relabels the dual variable; the
optimal $\lambda$ recovered below is $1/n$ times the multiplier of the ceiling as
written in \eqref{eq:PE}.
It is illustrated in Figure~\ref{fig:proof-tree}.
\makeatletter
\ifdefined\PREPRINTMAIN\else
  \ifx\documentclass\@twoclasseserror \def\PREPRINTMAIN{}\fi % a class is loaded => \input'd
\fi
\makeatother
\ifdefined\PREPRINTMAIN
\else
  \documentclass[border=8pt]{standalone}
  \usepackage{amsmath,amssymb}
  \usepackage{caption}
  \usepackage{tikz}
  \usetikzlibrary{arrows.meta,backgrounds,fit,calc,positioning,%
                  decorations.pathreplacing,shapes.geometric}
  \begin{document}
\fi
\usetikzlibrary{arrows.meta,backgrounds,calc,decorations.pathreplacing,positioning}
\newcommand{\ProofTreeCaption}{%
\textbf{The decision tree behind the proof of Theorem~1.}
Program~(PE) is a finite–horizon control problem on the token tree.
A \emph{node} is a committed prefix (a state) $w^{t-1}$; an \emph{edge}
is a token $w\!\in\!W$, so every node has $|W|$ children (drawn as $2$
for legibility) and the tree has depth $n$. A \emph{leaf} is a complete
sequence $w^{n}$ carrying the terminal reward $J_{n+1}(w^{n})=U(w^{n})$;
every step contributes the per-step reward
$\rho_t=\tfrac{\mu_E}{n}H(q_t)+\lambda\sum_w q_t\log p_t$.
The proof is \emph{two decoupled passes}. \emph{Backward} (right, Lemma~1):
the cost-to-go $J_t(w^{t-1})$—equivalently the potential
$V_t(w\mid w^{t-1}):=J_{t+1}(w^{t-1}w)$ that lives on the \emph{child}—is
formed from the children by a single log-sum-exp,
$J_t=\tfrac1{\mu^\star}\log\sum_w p_t^{\lambda^\star}e^{\mu^\star J_{t+1}}$;
this uses only $(p,U,\lambda^\star,\mu^\star)$, no policy and no fixed point.
\emph{Forward} (left): sample $w_t\sim q_t^\star(\cdot\mid w^{t-1})$.
The subtrees hanging off distinct children (shaded blue/green) are
\emph{disjoint}, so the tail conditionals decouple across them—this is
exactly what lets the global $\max$ over $q$ split into an independent
per-node $\max$ (principle of optimality, Lemma~3). At each node that
per-node problem is one Gibbs program (Lemma~2), solved by the
tilted-temperature rule
$q_t^\star\!\propto\!\big(p_t\,e^{\eta V_t}\big)^{1/T}$ (dashed box).
Because $V_t$ is a free-energy continuation value, the backward ($J$)
and forward ($u$) sweeps stay decoupled and exact (Corollary~3).}
\ifdefined\PREPRINTMAIN
  \begin{figure}[!htbp]\centering
  \resizebox{0.94\textwidth}{!}{%
\else
  \begin{minipage}{20cm}\centering
\fi
\begin{tikzpicture}[
  >=Latex,
  every node/.style={inner sep=1.2pt},
  state/.style={circle,draw=black!70,fill=white,minimum size=6mm,line width=.5pt},
  leaf/.style ={rectangle,rounded corners=.6pt,draw=black!70,fill=black!6,minimum size=4.2mm},
  hl/.style   ={draw=red!75!black,line width=1.1pt,fill=red!7},
  ed/.style   ={draw=black!55,line width=.55pt},
  hled/.style ={draw=red!75!black,line width=1.0pt},
  ann/.style  ={font=\footnotesize,text=black!80},
]
\def\yr{6.0}\def\ya{4.0}\def\yb{2.0}\def\yl{0.0}
\node[state] (R) at (0,\yr) {$\varnothing$};
\node[state] (A) at (-3.8,\ya) {};
\node[state,hl] (B) at (3.8,\ya) {};
\node[state] (AA) at (-5.7,\yb) {};
\node[state] (AB) at (-1.9,\yb) {};
\node[state,hl] (BA) at (1.9,\yb) {};
\node[state,hl] (BB) at (5.7,\yb) {};
\foreach \x/\n in {-6.65/L1,-4.75/L2,-2.85/L3,-0.95/L4,0.95/L5,2.85/L6,4.75/L7,6.65/L8}{
  \node[leaf] (\n) at (\x,\yl) {\scriptsize$U$};}
\begin{scope}[on background layer]
  \fill[blue!9]  ($(BA)+(0,.28)$) -- ($(L5)+(-.42,-.28)$) -- ($(L6)+(.42,-.28)$) -- cycle;
  \fill[green!12]($(BB)+(0,.28)$) -- ($(L7)+(-.42,-.28)$) -- ($(L8)+(.42,-.28)$) -- cycle;
\end{scope}
\draw[ed]  (R)--(A);   \draw[hled] (R)--(B);
\draw[ed]  (A)--(AA);  \draw[ed]  (A)--(AB);
\draw[hled](B)--(BA);  \draw[hled](B)--(BB);
\draw[ed] (AA)--(L1); \draw[ed] (AA)--(L2);
\draw[ed] (AB)--(L3); \draw[ed] (AB)--(L4);
\draw[ed] (BA)--(L5); \draw[ed] (BA)--(L6);
\draw[ed] (BB)--(L7); \draw[ed] (BB)--(L8);
\node[ann,fill=white,inner sep=1pt] at ($(R)!.5!(A)+(0.15,0.12)$) {$w\!\in\!W$};
\foreach \yy/\lab in {\yr/{$w^{0}=\varnothing$},\ya/{$w^{1}$},\yb/{$w^{2}$},\yl/{$w^{3}=w^{n}$}}{
  \node[ann,anchor=east,text=black!60] at (-7.7,\yy) {\lab};}
\node[ann,anchor=east,text=black!45,rotate=90] at (-8.55,3.0)
  {\scriptsize tokens emitted $\downarrow$};
\node[ann,anchor=north west,text=red!70!black,fill=white,inner sep=1.5pt]
  (VLAB) at (2.35,1.30) {$V_t(w\mid w^{t-1})=J_{t+1}(w^{t-1}w)$};
\draw[->,red!70!black,line width=.5pt] (VLAB.north west) to[bend left=12] (BA.south);
\draw[decorate,decoration={brace,amplitude=4pt,mirror},black!45]
  ($(L1)+(-.5,-.5)$) -- ($(L8)+(.5,-.5)$);
\node[ann,text=black!60] at (0,\yl-1.05)
  {leaves $w^{n}$: terminal reward $J_{n+1}(w^{n})=U(w^{n})$};
\node[ann,align=center,text=black!70] at (3.8,-2.1)
  {subtrees under distinct children are \emph{disjoint}\\
   $\Rightarrow$ tail conditionals decouple (Lemma~3)};
\node[draw=red!70!black,dashed,rounded corners=2pt,fill=red!4,
      align=left,text width=8.2cm,font=\footnotesize,inner sep=5pt]
  (BOX) at (2.9,8.35)
  {\textbf{per-node problem} (any state $w^{t-1}$; principle of optimality $+$ Gibbs):\\[2pt]
   $\displaystyle J_t(w^{t-1})=\!\!\max_{q_t\in\Delta(W)}\!\Big\{\underbrace{\tfrac{\mu_E}{n}H(q_t)+\lambda\!\sum_w q_t\log p_t}_{\rho_t}+\sum_w q_t\,J_{t+1}(w^{t-1}w)\Big\}$\\[3pt]
   $\displaystyle =\max_{q_t}\{\langle q_t,c\rangle+\tau H(q_t)\},\quad \tau=\tfrac{\mu_E}{n},\ \ c(w)=\lambda\log p_t+J_{t+1}(w^{t-1}w)$\\[3pt]
   $\displaystyle \Rightarrow\ \ q_t^\star(w\mid w^{t-1})\ \propto\ \big(p_t(w\mid w^{t-1})\,e^{\eta V_t(w\mid w^{t-1})}\big)^{1/T}.$};
\draw[->,red!70!black,line width=.7pt] (BOX.south) .. controls (4.6,6.6) .. (B.north east);
\draw[->,black!65,line width=.9pt] (7.55,-.1) -- (7.55,6.1);
\node[ann,rotate=90,anchor=south,text=black!65] at (7.35,3.0)
  {\textbf{backward sweep} (Lemma~1): $J_t\!=\!\tfrac1{\mu^\star}\log\sum_w p_t^{\lambda^\star}e^{\mu^\star J_{t+1}}$};
\draw[->,black!45,line width=.9pt] (8.35,6.1) -- (8.35,-.1);
\node[ann,rotate=90,anchor=north,text=black!50] at (8.55,3.0)
  {\textbf{forward pass}: sample $w_t\sim q_t^\star(\cdot\mid w^{t-1})$};
\end{tikzpicture}
\ifdefined\PREPRINTMAIN
  }
  \caption{\ProofTreeCaption}
  \label{fig:proof-tree}
  \end{figure}
\else
  \par\vspace{10pt}
  {\footnotesize\captionof{figure}{\ProofTreeCaption}}
  \end{minipage}
\fi
\ifdefined\PREPRINTMAIN\else
  \end{document}
\fi

\subsection{Step 1: The recursion}

The one detail suppressed in the main text is why \eqref{eq:app-lag} collapses to a
node-wise recursion. This is the principle of optimality; we prove it for the
present objective, whose per-step ``reward'' is a functional of the whole
conditional $q_t$ rather than of a sampled action.

\begin{lemma}[Principle of optimality]\label{lem:app-poo}
Let $J_{n+1}(w^n):=U(w^n)$ and, for $t=n,n-1,\dots,1$ and every prefix
$w^{t-1}$,
\begin{equation}\label{eq:app-J}
J_t(w^{t-1}):=\max_{q_t\in\Delta(\W)}
\Big\{\rho_t(w^{t-1},q_t)+\sum_{w\in\W}q_t(w\mid w^{t-1})\,J_{t+1}(w^{t-1}w)\Big\}.
\end{equation}
Then $\max_q\mathcal L(q)=J_1(\varnothing)$, and any policy that attains the maximum
in \eqref{eq:app-J} at every node is a global maximiser of \eqref{eq:app-lag}.
\end{lemma}

\begin{proof}
For a fixed policy $q$ and a prefix $w^{t-1}$, let
\[
G^q_t(w^{t-1}):=
\E_{w_{t:n}\sim q}\Big[U(w^n)+\sum_{s=t}^n\rho_s(w^{s-1},q_s)\ \Big|\ w^{t-1}\Big],
\qquad G^q_{n+1}(w^n):=U(w^n),
\]
the conditional tail objective, the expectation being over $w_t,\dots,w_n$ drawn
from $q_t,\dots,q_n$ started at $w^{t-1}$. By the tower property applied to
\eqref{eq:app-lag}, $\mathcal L(q)=G^q_1(\varnothing)$. We show
$J_t(w^{t-1})=\max_{q_t,\dots,q_n}G^q_t(w^{t-1})$ by downward induction on $t$; the
case $t=1$ is the assertion.

\emph{Base.} $G^q_{n+1}=U=J_{n+1}$, with no maximisation.

\emph{Step.} Assume the claim at $t+1$ and fix $w^{t-1}$. The summand $\rho_t$ is
determined by $q_t$ given the prefix, so conditioning on the next token $w_t=w$ and
using the tower property,
\begin{equation}\label{eq:app-tower}
G^q_t(w^{t-1})=\rho_t(w^{t-1},q_t)+\sum_{w\in\W}q_t(w\mid w^{t-1})\,
G^q_{t+1}(w^{t-1}w).
\end{equation}
For each $w$ the term $G^q_{t+1}(w^{t-1}w)$ depends only on the conditionals
$(q_{t+1},\dots,q_n)$ restricted to the sub-tree rooted at $w^{t-1}w$, and these
sub-trees are pairwise disjoint across distinct $w$. Hence the tail conditionals may
be optimised independently within each sub-tree, and the coefficients
$q_t(w\mid w^{t-1})\ge0$ in \eqref{eq:app-tower} do not couple them:
\begin{align*}
\max_{q_t,\dots,q_n}G^q_t(w^{t-1})
=&\max_{q_t}\Big\{\rho_t(w^{t-1},q_t)
+\sum_w q_t(w\mid w^{t-1})\max_{q_{t+1:n}}G^q_{t+1}(w^{t-1}w)\Big\}\\
=&\max_{q_t}\Big\{\rho_t+\sum_w q_t\,J_{t+1}(w^{t-1}w)\Big\}
=J_t(w^{t-1}),
\end{align*}
where the second equality is the induction hypothesis and the last is
\eqref{eq:app-J}. The interchange of $\max_{q_t}$ with the inner maxima is licit
because, once each $\max_{q_{t+1:n}}G^q_{t+1}(w^{t-1}w)=J_{t+1}(w^{t-1}w)$ is a
constant in $q_t$, the outer bracket is a function of $q_t$ alone. Assembling the
node maximisers into a policy yields a global maximiser.
\end{proof}

Substituting $\rho_t$ from \eqref{eq:app-lag} into \eqref{eq:app-J} writes the
recursion in the form used in the main text,
\[
J_t(w^{t-1})=\max_{q_t\in\Delta(\W)}
\sum_w q_t(w)\big[\lambda\log p_t(w\mid w^{t-1})+J_{t+1}(w^{t-1}w)\big]
+\tfrac{\mu_E}{n}H(q_t),
\qquad J_{n+1}=U.
\]

\subsection{Step 2: Per-node Gibbs solution}

The bracket
$c(w):=\lambda\log p_t(w\mid w^{t-1})+J_{t+1}(w^{t-1}w)$ does not involve $q_t$, and
$\rho_t=\langle q_t,\lambda\log p_t\rangle+\tfrac{\mu_E}{n}H(q_t)$, so the node
problem is $\max_{q_t}\{\langle q_t,c\rangle+\tau H(q_t)\}$ with $\tau=\mu_E/n>0$.
Lemma~\ref{lem:gibbs} (the Gibbs variational principle) applies verbatim and gives
\[
q^\star_t(w\mid w^{t-1})=\frac{e^{c(w)/\tau}}{\sum_{w'}e^{c(w')/\tau}},
\qquad
J_t(w^{t-1})=\tau\log\sum_w e^{c(w)/\tau}.
\]
Since $c(w)/\tau=\tfrac{n}{\mu_E}\lambda\log p_t(w\mid w^{t-1})
+\tfrac{n}{\mu_E}J_{t+1}(w^{t-1}w)=\lambda^\star\log p_t+\mu^\star J_{t+1}$ with
$\lambda^\star:=n\lambda/\mu_E$ and $\mu^\star:=n/\mu_E$ (so $\tau=1/\mu^\star$),
\begin{align}\label{eq:app-node}
q^\star_t(w\mid w^{t-1}) &\propto
p_t(w\mid w^{t-1})^{\lambda^\star}\exp\!\big(\mu^\star J_{t+1}(w^{t-1}w)\big),\\
J_t(w^{t-1})&=\tfrac{1}{\mu^\star}\log\sum_w
p_t(w\mid w^{t-1})^{\lambda^\star}\exp\!\big(\mu^\star J_{t+1}(w^{t-1}w)\big).
\label{eq:app-node2}
\end{align}

\subsection{Step 3: Identification with the potential of Lemma~\ref{lem:bellman}}

Set $V_t(w\mid w^{t-1}):=J_{t+1}(w^{t-1}w)$. The left identity in
\eqref{eq:app-node} is \eqref{eq:opt1}--\eqref{eq:opt2} with temperature
$T=1/\lambda^\star$ and tilt $\eta=\mu^\star/\lambda^\star$. The terminal value is
$V_n(w\mid w^{n-1})=J_{n+1}(w^{n-1}w)=U(w^{n-1}w)=U(w^n)$, and the identity in
\eqref{eq:app-node2}, rewritten through $V$, is
\[
\tfrac1{\mu^\star}\log Z_t(w^{t-1}),\qquad
Z_t(w^{t-1})=\sum_w p_t(w\mid w^{t-1})^{\lambda^\star}
\exp\!\big(\mu^\star V_t(w\mid w^{t-1})\big),
\]
i.e.\ the terminal condition and backward recursion of Lemma~\ref{lem:bellman}. That
recursion involves only $(p,U,\lambda^\star,\mu^\star)$ and no policy, so $V_t$ is
its potential: the optimum is one backward sweep for $V$ followed by the forward
policy \eqref{eq:opt2}, with no self-consistent pair $(q^\star,V^\star)$ to solve.

\subsection{Step 4: Convexity, Slater, strong duality, complementary slackness}

It remains to certify that the Lagrangian stationary point is the constrained
optimum and that the multipliers are the two active constraints. Pass to causal
occupancy-measure coordinates: a policy is equivalent to a state--action visitation
$d=(d_t)_t$, $d_t(w^{t-1},w):=q(w^{t-1})\,q_t(w\mid w^{t-1})$, subject to the linear
causal-consistency (Bellman-flow) constraints
$\sum_w d_t(w^{t-1},w)=\sum_{w'} d_{t-1}(w^{t-2},w')\,\mathbf 1[(w^{t-2},w')=w^{t-1}]$
and $d_t\ge0$; the feasible set $\mathcal D$ is a convex compact polytope. In these
coordinates:
\begin{itemize}
\item $u_q=\sum_{w^n}d_n(w^{n-1},w_n)\,U(w^n)$ and the plausibility functional
$\frac1n\sum_t\sum_{w^{t-1},w}d_t\log p_t$ are \emph{linear} in $d$;
\item the causal entropy
$\frac1n\sum_t\sum_{w^{t-1},w}d_t\log\frac{\sum_{w'}d_t(w^{t-1},w')}{d_t(w^{t-1},w)}$
is a sum of conditional entropies, each \emph{concave} in $d$.
\end{itemize}
Hence \eqref{eq:PE} maximises a linear objective over $\mathcal D$ under one concave
($\ge$) constraint and one linear ($\le$) constraint: a convex program on a convex
compact set. Slater's condition holds whenever $\beta<\log|\W|$ (the uniform policy
is strictly entropy-feasible) and $\gamma$ is not at its binding-tight extreme, so a
strictly feasible $d$ exists. Strong duality then applies: at the optimal multipliers
$(\lambda,\mu_E)$ the unconstrained maximiser of $\mathcal L$ --- which by Steps~1--3
is the tilted-temperature policy \eqref{eq:opt1}--\eqref{eq:opt2} --- solves
\eqref{eq:PE}. Complementary slackness gives
\[
\mu_E\bigl(\bar H(q^\star)-\beta\bigr)=0,
\qquad
\lambda\bigl(\gamma-\mathrm{CE}(q^\star,p)\bigr)=0.
\]
Thus a slack constraint has zero multiplier. In the interior case where both multipliers
are positive, both constraints hold with equality and pin $(\lambda,\mu_E)$---equivalently
$(T,\eta)$---through the budgets $(\beta,\gamma)$.
\qed
\paragraph{The optimum is a single backward sweep.}
By Lemma~\ref{lem:bellman}, the potential $V_t$ in \eqref{eq:opt2} depends only on
$(p,U,\lambda^\star,\mu^\star)$. Hence, for fixed dual prices, it is obtained by
a single backward sweep, after which the optimal forward policy is given directly
by \eqref{eq:opt2}; no policy--value fixed point is required.

\begin{corollary}[Free-energy decomposition]\label{cor:freeenergy}
For $t<n$, with $q^\star$ from \eqref{eq:opt2} and
$\mathrm{CE}(q,p):=-\sum_w q(w)\log p(w)$,
\[
V_t(w\mid w^{t-1})
=
\underbrace{
\E_{w'\sim q^\star_{t+1}(\cdot\mid w^{t-1}w)}
\!\big[V_{t+1}(w'\mid w^{t-1}w)\big]
}_{\text{expectation backup of }V}
+
\frac{1}{\mu^\star}
\Big[
H\big(q^\star_{t+1}\big)
-\lambda^\star\,\mathrm{CE}\big(q^\star_{t+1},p_{t+1}\big)
\Big].
\]
Thus $V_t$ is an expectation backup augmented by an entropy--plausibility
free-energy correction. Since
$\mathrm{CE}(q,p)=H(q)+\mathrm{KL}(q\Vert p)$, the correction can be written as
\[
\frac{1}{\mu^\star}
\Big[
(1-\lambda^\star)H\big(q^\star_{t+1}\big)
-\lambda^\star\mathrm{KL}\big(q^\star_{t+1}\Vert p_{t+1}\big)
\Big],
\]
so its sign is not fixed in general.
\end{corollary}
\vspace{-0.05in}
\begin{proof}
Take $\log$ of \eqref{eq:opt2} at step $t+1$ and average under
$q^\star_{t+1}$. Since
$\sum_{w'}q^\star_{t+1}\log q^\star_{t+1}=-H(q^\star_{t+1})$, we obtain
\[
-H(q^\star_{t+1})
=
-\lambda^\star\mathrm{CE}(q^\star_{t+1},p_{t+1})
+\mu^\star\E[V_{t+1}]
-\log Z_{t+1}.
\]
Solving for $\log Z_{t+1}$ and using
$V_t=\mu^{\star-1}\log Z_{t+1}$ from Lemma~\ref{lem:bellman} gives the claim.
\end{proof}

\providecommand{\X}{\mathcal{X}}
\providecommand{\Dist}{\mathcal{D}}
\providecommand{\Simp}{\Delta}
\providecommand{\W}{\mathcal{W}}
\providecommand{\E}{\mathbb{E}}
\providecommand{\Hbar}{\bar H}
\providecommand{\CE}{\mathrm{CE}}
\providecommand{\Vt}{\widetilde{V}}
\providecommand{\lamst}{\lambda^{\star}}
\providecommand{\must}{\mu^{\star}}
\providecommand{\qst}{q^{\star}}
\providecommand{\argmin}{\operatorname{arg\,min}}

\section{Prompt-Aware and Robust Formulations of (PE)}
\label{app:prompt}

\providecommand{\qst}{q^{\star}}
\providecommand{\argmin}{\operatorname{arg\,min}}

Program (PE) is written for a fixed prompt: $p_t(w\mid w^{t-1})$ already conditions on $x$,
so its optimum is a per-prompt object. In deployment a single decoder is applied across a
prompt distribution, and the natural objectives --- average and worst-case --- must be
written down explicitly. This appendix does that for three formulations and shows that each
is solved by the tilted-temperature family of Theorem~1 applied prompt-wise. The
formulations differ only in how the two prices $(T,\eta)$ are pinned: locally at each
prompt, globally across the prompt distribution, or globally at a least-favorable mixture.

Let $\X$ be the prompt space with reference distribution $\Dist$; take $\X$ finite (or
compact, with the usual continuity, so that the minima below are attained). For each
$x\in\X$ the model provides conditionals $p_t(\cdot\mid w^{t-1},x)\in\Simp(\W)$ with
$p_t\ge\rho>0$, and a possibly $x$-dependent verifier $U(\cdot,x):\W^n\to[0,1]$. A decoder
is a measurable map
\[
\Phi:\ (x,w^{t-1})\ \longmapsto\ q_t(\cdot\mid w^{t-1},x)\in\Simp(\W),\qquad t=1,\dots,n,
\]
i.e.\ it chooses a sampler at every prompt and prefix. Write the per-prompt quantities
\[
u_q(x):=\E_{w^n\sim q(\cdot\mid x)}\big[U(w^n,x)\big],
\qquad
\Hbar_q(x):=\frac1n\sum_t\E_{w^{t-1}\sim q}\,H\big(q_t(\cdot\mid w^{t-1},x)\big),
\]
\[
\CE_q(x):=-\frac1n\sum_t\E_{w^{t-1}\sim q}\Big[\textstyle\sum_w q_t\log p_t\Big].
\]
The observation that drives everything below: $u_q(x)$, $\Hbar_q(x)$ and $\CE_q(x)$ depend
on the decoder only through its restriction $q(\cdot\mid x)$ to prompt $x$. Distinct
prompts are coupled only if a constraint or the objective ties them together.

For prices $\lambda>0$, $\mu>0$ let $V^{\,x}_t(w\mid w^{t-1};\lambda,\mu)$ denote the
verifier potential of Lemma~1 built from the pair $\big(p(\cdot\mid x),U(\cdot,x)\big)$ at
those prices: it involves no policy, and is obtained by one backward sweep over prefix states
per prompt. We write $\varpi\in\Simp(\X)$ for a mixture over prompts, reserving $w$ for
tokens. We propose three prompt-level formulations:

\paragraph{(A) Average objective, per-prompt budgets.}
\begin{equation}\label{eq:A}
\max_{\Phi}\ \E_{x\sim\Dist}\big[u_q(x)\big]
\quad\text{s.t.}\quad
\Hbar_q(x)\ge\beta(x),\ \ \CE_q(x)\le\gamma(x)\ \ \text{for $\Dist$-a.e.\ }x.
\tag{A}
\end{equation}

\paragraph{(B) Average objective, pooled budgets.}
\begin{equation}\label{eq:B}
\max_{\Phi}\ \E_{x\sim\Dist}\big[u_q(x)\big]
\quad\text{s.t.}\quad
\E_{x\sim\Dist}\big[\Hbar_q(x)\big]\ge\beta,\ \
\E_{x\sim\Dist}\big[\CE_q(x)\big]\le\gamma.
\tag{B}
\end{equation}
One diversity budget and one plausibility budget are shared across the whole prompt
distribution --- the deployment regime, where a single temperature is spent over the
benchmark. This is the formulation the experiments of Section~4 instantiate: one
cross-entropy budget anchored at a reference temperature, one deployed $T$, applied across
all tasks.

\paragraph{(C) Worst-case prompt (robust / distributionally robust).}
\begin{equation}\label{eq:C}
\max_{\Phi}\ \min_{x\in\X}\ u_q(x)
\quad\text{s.t.\ per-prompt or pooled budgets}
\qquad=\qquad
\max_{\Phi}\ \min_{\varpi\in\Simp(\X)}\ \E_{x\sim\varpi}\big[u_q(x)\big].
\tag{C}
\end{equation}
The equality rewrites the worst prompt as the worst mixture (the inner objective is linear
in $\varpi$, so its minimum is attained at a vertex); $\min_\varpi$ over all of
$\Simp(\X)$ is the distributionally robust objective with the largest ambiguity set, and
any smaller set $\mathcal K\ni\Dist$ interpolates between \eqref{eq:B}
($\mathcal K=\{\Dist\}$) and \eqref{eq:C}.

\begin{remark}[The optimum is a \emph{rule}, computed in two stages]
\label{rem:two-stage}
It is tempting to read ``averaging over $x$'' in \eqref{eq:B} as eliminating the prompt
from the solution. It does not. Pooling eliminates $x$ from the \emph{constraints} and from
the scalar objective, but the decision variable is the whole decoder
$\Phi:x\mapsto q(\cdot\mid x)$, and the optimum is a rule, not a single distribution.
Equivalently: $(T,\eta)$ is a global \emph{operator} and
$\big(p(\cdot\mid x),U(\cdot,x)\big)$ is its prompt-local \emph{argument}. At fixed global
multipliers the Lagrangian integrand at prompt $x$ depends only on $q(\cdot\mid x)$, so the
inner maximization is pointwise and its argmax responds to the local pair even though the
multipliers do not. Pooling flattens the operator, never the argument; the only way
$\qst(\cdot\mid x)$ could be $x$-free is the degenerate case in which
$p(\cdot\mid x),U(\cdot,x)$ are constant in $x$. (Compare cross-validated ridge: the
penalty is chosen from the data distribution with no test point in hand; the prediction
still depends on the realized $x$.)

The two stages are therefore: \emph{design} (offline, $x$ random and only $\Dist$ known)
solves the two budget equations for the scalars $(T,\eta)$ from a calibration sample of
$\Dist$ alone; \emph{inference} (online, $x$ realized) runs the backward sweep of Lemma~1
on $\big(p(\cdot\mid x),U(\cdot,x)\big)$ at those prices to obtain $V^{\,x}$, then
samples \eqref{eq:tilt}. Since $V$ depends on the prices as well as on the prompt, the
sweep in \eqref{eq:B} uses the same global $(\lamst,\must)$ at every prompt, whereas in
\eqref{eq:A} each prompt's sweep uses its own.
\end{remark}

\begin{theorem}[Invariance]
\label{thm:inv}
Assume $\lambda^\star > 0$ and $\mu^\star > 0$. In each of \eqref{eq:A}, \eqref{eq:B}, \eqref{eq:C} the optimal
decoder is, at every $x$, prefix $w^{t-1}$ and token $w$,
\begin{equation}\label{eq:tilt}
\qst_t(w\mid w^{t-1},x)
=\frac{\big(p_t(w\mid w^{t-1},x)\,\exp\!\big(\eta\,V^{\,x}_t(w\mid w^{t-1})\big)\big)^{1/T}}
{Z_t(w^{t-1},x;T,\eta)},
\end{equation}
the tilted-temperature family of Theorem~1 applied prompt-wise, with $V^{\,x}$ the
potential of Lemma~1 for the pair $\big(p(\cdot\mid x),U(\cdot,x)\big)$ at the prices in
force. The formulations differ only in how $(T,\eta)$ are pinned:
\begin{itemize}[leftmargin=2.2em,itemsep=2pt]
\item[(A)] $(T,\eta)=\big(T(x),\eta(x)\big)$, solving the two budget equations
$\Hbar_q(x)=\beta(x)$, $\CE_q(x)=\gamma(x)$ separately at each prompt; the potential
$V^{\,x}$ is built at that prompt's prices.
\item[(B)] $(T,\eta)$ is a single global pair, solving the two pooled budget equations
$\E_\Dist[\Hbar_q]=\beta$, $\E_\Dist[\CE_q]=\gamma$; every $V^{\,x}$ is built at those
global prices.
\item[(C)] \emph{Per-prompt budgets:} \eqref{eq:C} is solved by the optimum of
\eqref{eq:A}, and robustness is vacuous. \emph{Pooled budgets:} $(T,\eta)$ is the single
global pair pinned by the two budget equations evaluated under a least-favorable mixture
$\varpi^\star\in\Simp(\X)$, and $\qst$ equalizes the per-prompt success on its support,
$u_{\qst}(x)=\min_{x'}u_{\qst}(x')$ for all $x\in\operatorname{supp}\varpi^\star$.
\end{itemize}
\end{theorem}

\begin{proof}
Throughout, write the per-prompt Lagrangian integrand at prices $(\lambda,\mu_E)$,
\begin{equation}\label{eq:Lx}
\mathcal L_x\big(q(\cdot\mid x)\big)
:= u_q(x)+\sum_{t=1}^{n}\E_{w^{t-1}\sim q}\Big[\tfrac{\mu_E}{n}H(q_t)
+\lambda\textstyle\sum_w q_t\log p_t\Big],
\end{equation}
which is exactly the objective appearing in the proof of Theorem~1 for the fixed prompt
$x$: a finite-horizon control problem with terminal reward $U(\cdot,x)$ and per-step reward
$\lambda\E_{q_t}[\log p_t]+\tfrac{\mu_E}{n}H(q_t)$. By that proof --- Lemma~2 applied node
by node to the optimal cost-to-go --- its unique maximizer is \eqref{eq:tilt} with
\begin{equation}\label{eq:prices}
\lamst=\frac{n\lambda}{\mu_E},\qquad \must=\frac{n}{\mu_E},\qquad
T=\frac{1}{\lamst}=\frac{\mu_E}{n\lambda},\qquad \eta=\frac{\must}{\lamst}=\frac1\lambda ,
\end{equation}
and $V^{\,x}$ the Lemma~1 potential of $\big(p(\cdot\mid x),U(\cdot,x)\big)$ at
$(\lamst,\must)$: it is the optimal cost-to-go itself,
$V^{\,x}_t(w\mid w^{t-1})=J_{t+1}(w^{t-1}w)$, produced by the dynamic program and
depending on no policy. Also, in occupancy-measure coordinates $u_q(x)$ and $\CE_q(x)$ are
linear and $\Hbar_q(x)$ is concave, so each per-prompt program is convex on a convex compact
set, and Slater's condition holds whenever $\beta<\log|\W|$ and $\gamma$ is not
binding-tight.

\smallskip
\noindent\emph{(A).} Objective and constraints are both collections of per-prompt terms,
and every term at $x$ involves only $q(\cdot\mid x)$. Hence \eqref{eq:A} separates into
independent copies of (PE), one per $x$; each is solved by Theorem~1, whose duals
$(\lamst(x),\must(x))$ --- and hence
$\big(T(x),\eta(x)\big)=\big(1/\lamst(x),\must(x)/\lamst(x)\big)$ --- are pinned by that
prompt's two active constraints.

\smallskip
\noindent\emph{(B).} Introduce scalar multipliers $\mu_E\ge0$ (pooled entropy floor) and
$\lambda\ge0$ (pooled cross-entropy ceiling). The Lagrangian is
\[
\mathcal L(\Phi;\lambda,\mu_E)=\E_{x\sim\Dist}\big[\mathcal L_x(q(\cdot\mid x))\big]-\mu_E\beta+\lambda\gamma .
\]
At fixed $(\lambda,\mu_E)$ the expectation is of terms each depending on a single
restriction $q(\cdot\mid x)$, so $\max_\Phi\mathcal L$ decouples across prompts,
$\max_\Phi\E_x[\mathcal L_x]=\E_x[\max_{q(\cdot\mid x)}\mathcal L_x]$, the pointwise
maximizers being measurable in $x$ by \eqref{eq:tilt} (a continuous function of the
conditionals and of $V^{\,x}$, itself a continuous function of the conditionals through
$n$ log-sum-exps). Each inner problem is \eqref{eq:Lx}, so its solution is \eqref{eq:tilt}
with the prices \eqref{eq:prices} --- \emph{global}, because $(\lambda,\mu_E)$ are global,
while the tilted object $\big(p(\cdot\mid x),V^{\,x}\big)$ is local. Convexity plus
Slater gives strong duality, so the maximizer of $\mathcal L$ at the optimal multipliers
solves \eqref{eq:B}, and $(\lambda,\mu_E)$ are fixed by complementary slackness at the two
pooled active constraints $\E_\Dist[\Hbar_q]=\beta$, $\E_\Dist[\CE_q]=\gamma$.

\smallskip
\noindent\emph{(C), per-prompt budgets.} In \eqref{eq:A} both the objective and the
constraints separate, so its optimum maximizes $u_q(x)$ \emph{pointwise}, subject to that
prompt's budgets. A decoder that maximizes every $u_q(x)$ simultaneously maximizes
$\min_x u_q(x)$ as well; hence the \eqref{eq:A}-optimum solves \eqref{eq:C} and nothing is
lost by ignoring the robust objective. Prompts do not compete: each has its own budget.

\smallskip
\noindent\emph{(C), pooled budgets.} Now prompts compete for one budget and the problem is
a genuine saddle point. Pass to occupancy-measure (causal) coordinates: a decoder is
equivalent to a causal visitation $\omega=(\omega_x)_{x\in\X}$, in which
$u_q(x)=\langle\omega_x,U_x\rangle$ is linear in $\omega$, the pooled cross-entropy is
linear, and the pooled causal entropy is concave; the feasible set $\Omega$ is therefore
convex and compact. The rewritten objective
\[
G(\omega,\varpi):=\sum_{x}\varpi(x)\,\langle\omega_x,U_x\rangle
\]
is bilinear: linear (hence concave) in $\omega$ on $\Omega$, and linear --- hence convex and
continuous --- in $\varpi$ on the compact simplex $\Simp(\X)$. Sion's minimax theorem
applies, so $\max_\omega\min_\varpi G=\min_\varpi\max_\omega G$ and a saddle point
$(\omega^\star,\varpi^\star)$ exists. For fixed $\varpi$, the inner
$\max_\omega G(\cdot,\varpi)$ is exactly formulation \eqref{eq:B} with $\Dist$ replaced by
$\varpi$; hence its solution is \eqref{eq:tilt} with the \emph{single global} pair
$(T,\eta)$ pinned by the budgets under $\varpi$. Evaluating at $\varpi^\star$ gives the
stated form. Finally, optimality of $\varpi^\star$ in the outer min is a linear program
over $\Simp(\X)$ with cost $u_{\qst}(x)$; its optimal support lies in
$\argmin_x u_{\qst}(x)$, and complementary slackness forces $u_{\qst}(\cdot)$ to be constant
$\big(=\min_{x'}u_{\qst}(x')\big)$ on $\operatorname{supp}\varpi^\star$ --- the
equalization claim.
\end{proof}

\section{Adaptive Optimization: Piecewise-in-Horizon Budgets}
\label{app:adaptive}

Program (PE) uses one diversity floor $\beta$ and one plausibility ceiling
$\gamma$ over the full horizon. To model position-adaptive decoding, we let
these budgets take a small number $D$ of values along the horizon. The main
structural fact is that the optimum remains in the tilted-temperature family:
each block receives its own pair $(T_d,\eta_d)$. The global result of
Theorem~\ref{thm:tilt} is recovered at $D=1$.

This extension is useful for interpreting practical adaptive decoders. A
position-dependent temperature schedule can exploit coarse variation across
the horizon, while a value-aware tilted decoder can additionally react to the
state of the current prefix. The block model isolates the first effect.

\subsection{Block program and optimum}

Fix cut points
\[
0=t_0<t_1<\cdots<t_D=n,
\qquad
B_d:=\{t_{d-1}+1,\ldots,t_d\},
\qquad
n_d:=|B_d|.
\]
Each block carries its own diversity and plausibility budgets
$(\beta_d,\gamma_d)$. The block-constrained problem is
\begin{equation}
\label{eq:PED}
\max_q\ u_q
\quad\text{s.t.}\quad
\frac{1}{n_d}\sum_{t\in B_d}
\E_{w^{t-1}\sim q}
H\!\bigl(q_t(\cdot\mid w^{t-1})\bigr)\ge \beta_d,
\end{equation}
and
\begin{equation}
-\frac{1}{n_d}\sum_{t\in B_d}
\E_{w^{t-1}\sim q}
\Bigl[\sum_w q_t(w\mid w^{t-1})\log p_t(w\mid w^{t-1})\Bigr]
\le \gamma_d,
\qquad d=1,\ldots,D.
\tag{PE-$D$}
\end{equation}
As in (PE), maximizing $\passk$ is equivalent to maximizing the one-sample
success $u_q$ for fixed budgets.

Let $\mu_d>0$ and $\lambda_{0,d}>0$ denote the active block multipliers and
define
\[
\lambda_d:=\frac{\lambda_{0,d}}{n_d},
\qquad
\lambda_d^\star:=\frac{n_d\lambda_d}{\mu_d},
\qquad
\mu_d^\star:=\frac{n_d}{\mu_d},
\qquad
T_d:=\frac{1}{\lambda_d^\star},
\qquad
\eta_d:=\frac{\mu_d^\star}{\lambda_d^\star}.
\]
Write $d(t)$ for the block containing step $t$, set
$J_{n+1}(w^n):=U(w^n)$, and define the backward recursion
\begin{equation}
\label{eq:Jrec}
J_t(w^{t-1})
=
\frac{1}{\mu^\star_{d(t)}}
\log
\sum_{w\in W}
p_t(w\mid w^{t-1})^{\lambda^\star_{d(t)}}
\exp\!\Bigl(
\mu^\star_{d(t)}J_{t+1}(w^{t-1}w)
\Bigr).
\end{equation}
Let
\[
V_t(w\mid w^{t-1}):=J_{t+1}(w^{t-1}w).
\]

\begin{theorem}[$D$-block tilted temperature]
\label{thm:dblock}
Assume the block budgets are jointly attainable and all block entropy floors
are active. Then the optimum of \textup{(PE-$D$)} is
\begin{equation}
\label{eq:dblock-opt}
q_t^\star(w\mid w^{t-1})
\propto
p_t(w\mid w^{t-1})^{\lambda^\star_{d(t)}}
\exp\!\Bigl(
\mu^\star_{d(t)}V_t(w\mid w^{t-1})
\Bigr),
\end{equation}
or equivalently
\[
q_t^\star(w\mid w^{t-1})
\propto
\Bigl(
p_t(w\mid w^{t-1})
e^{\eta_{d(t)}V_t(w\mid w^{t-1})}
\Bigr)^{1/T_{d(t)}}.
\]
Thus the optimum remains tilted-temperature, with a pair
$(T_d,\eta_d)$ for each horizon block.
\end{theorem}

\begin{proof}[Proof sketch]
The proof is the proof of Theorem~\ref{thm:tilt} with position-dependent
dual prices. The Lagrangian still decomposes as a finite-horizon control
problem; only the entropy and log-likelihood coefficients are now piecewise
constant in $t$. Applying the Gibbs variational principle at each node gives
\eqref{eq:Jrec} and \eqref{eq:dblock-opt}. Strong duality and complementary
slackness then identify the block prices with the active budgets.
\end{proof}

\subsection{What changes at block boundaries}

For $D=1$, the backward recursion can be flattened into the single
sequence-level free energy used in the main text. With different block prices,
that flattening generally fails across a block boundary: the continuation
computed in one block is re-tempered by the prices of the preceding block.

For example, with two blocks and a boundary at $t_1$, for $t\in B_1$,
\begin{equation}
\label{eq:nested-D2}
V_t
=
\frac{1}{\mu_1^\star}
\log
\sum_{w_{t+1:t_1}}
\prod_{s=t+1}^{t_1}p_s^{\lambda_1^\star}
\left(
\sum_{w_{t_1+1:n}}
\prod_{s=t_1+1}^{n}p_s^{\lambda_2^\star}
e^{\mu_2^\star U(w^n)}
\right)^{\mu_1^\star/\mu_2^\star}.
\end{equation}
The inner and outer log-sum-exp operators merge only when the corresponding
diversity prices agree. Thus the Bellman recursion remains simple, but the
sequence law acquires seam corrections.

\begin{proposition}[Sequence-level block correction]
\label{prop:seq-dblock}
Under Theorem~\ref{thm:dblock},
\begin{equation}
\label{eq:seq-dblock}
q^\star(w^n)
\propto
\left(
\prod_{d=1}^{D}
\prod_{t\in B_d}
p_t(w_t\mid w^{t-1})^{\lambda_d^\star}
\right)
\exp\!\left(
\mu_D^\star U(w^n)
+
\sum_{d=1}^{D-1}
(\mu_d^\star-\mu_{d+1}^\star)
J_{t_d+1}(w^{t_d})
\right).
\end{equation}
When all $\mu_d^\star$ coincide, the seam corrections vanish. If, in
addition, all $\lambda_d^\star$ coincide, \eqref{eq:seq-dblock} reduces to
the sequence-level law of Corollary~1.
\end{proposition}

\begin{proof}
Multiplying the conditionals in \eqref{eq:dblock-opt} yields
\[
q^\star(w^n)
=
\left(\prod_t p_t^{\lambda^\star_{d(t)}}\right)
\exp\!\left(
\sum_{t=1}^n
\mu^\star_{d(t)}(J_{t+1}-J_t)
\right).
\]
The sum telescopes within each block. Collecting the remaining boundary terms
gives \eqref{eq:seq-dblock}.
\end{proof}

\subsection{Connection to sample budget and practical adaptive decoding}

For a fixed block schedule
$\{(\beta_d,\gamma_d)\}_{d=1}^D$, the optimizer remains independent of $k$,
because $\passk(q)=1-(1-u_q)^k$ is monotone in $u_q$. As in the global case,
$k$ affects the decoder only through how the diversity budgets are selected.
Hence the ``larger $k$ wants hotter'' mechanism can be applied blockwise:
larger sample budgets may justify more diversity, while the block structure
determines where that diversity is spent.

The distinction between block adaptation and value-aware adaptation is also
useful experimentally. A block temperature schedule can react to position but
not to the realized prefix within a block. By contrast, the continuation
potential $V_t(w\mid w^{t-1})$ is state dependent. This gives a direct
interpretation of the practical comparisons in Section~4: AdapT represents
position-dependent temperature adaptation, while TMAD-Gated uses similar
intervention opportunities but additionally incorporates confidence and a
continuation-success proxy. The block model therefore isolates the benefit
available from coarse position-dependent temperature schedules; any remaining
advantage from explicit continuation information corresponds to structure that
a fixed block schedule cannot express.

Under perfect calibration, a block-temperature decoder remains especially
simple when the diversity price is shared across blocks: the verifier
contribution folds into the block exponents, while unequal diversity prices
produce the seam terms in \eqref{eq:seq-dblock}. More generally, those seam
corrections require explicit continuation-value information rather than a
temperature schedule alone.

\section{How to choose $\beta$} 
\label{app:beta}

\noindent\textbf{Where $k$ enters the decoder.}
Since $\passk(q)=1-(1-u_q)^k$ is strictly increasing in $u_q$, for any fixed budgets $(\beta,\gamma)$ the maximizer of (PE) is the single tilted-temperature policy of Theorem~\ref{thm:tilt} and is \emph{independent of $k$}. The sample budget therefore enters the decoder only through the choice of entropy floor $\beta$: ``larger $k$ wants hotter'' is not a property of the optimal decoder at a fixed operating point, but of how the diversity budget is spent as $k$ grows. Making that dependence explicit requires an exogenous rule for selecting $\beta$ as a function of $k$; a natural one maximizes $\passk$ regularized by a strictly convex, strictly decreasing diversity-cost $S$,
\[
\beta^\star(k):=\argmax_{\beta\in[0,\log|\W|]}\ \big\{1-(1-u(\beta))^k-S(\beta)\big\},\qquad\text{e.g. } S(\beta)=c\,(\log|\W|-\beta)^2 ,
\]
where $u(\beta)$ is the optimal success at floor $\beta$ and $\lambda^\star(\beta)$ the corresponding cross-entropy dual. Under this rule larger $k$ favors a higher floor but,  as one should expect from the non-monotonicity of the pass@$k$ sensitivity $\sigma_k$ (Theorem~\ref{thm:partial}(iii)), only once the operating point is in the high-coverage regime; below it the comparative static reverses. The result below records this; note that $S$ is a modeling choice, so the statement concerns the induced budget schedule, not the optimum of (PE) itself. 

\begin{proposition}[Budget-selection comparative static]\label{prop:mono}
Assume $u:[0,\log|\W|]\to(0,1)$ is $C^1$ and strictly decreasing, $S$ is $C^1$, strictly decreasing and strictly convex, the dual $\lambda^\star(\beta)>0$ is non-increasing in $\beta$ on the budget range under consideration, and each $\beta^\star(k)$ is a non-degenerate interior maximizer. Then at every such $k$,
\[
\operatorname{sign}\tfrac{d}{dk}\beta^\star(k)=\operatorname{sign}\big(\passk(\beta^\star(k))-(1-e^{-1})\big).
\]
Consequently, on any budget range where the operating point clears the coverage line $\passk(\beta^\star(k))\ge 1-e^{-1}$ (equivalently $(1-u(\beta^\star(k)))^k\le e^{-1}$), the floor $\beta^\star(k)$---and hence the induced temperature $T^\star(k):=1/\lambda^\star(\beta^\star(k))$---is non-decreasing, strictly at interior points where the inequality is strict; below the crossover $k_0=-1/\log(1-u^\star)\approx 1/u^\star$ it is non-increasing. Since $\beta^\star(k)\to\log|\W|$ and $\passk(\beta^\star(k))\to1$ as $k\to\infty$, all sufficiently large $k$ lie in the increasing regime.
\end{proposition}

We prove Proposition~\ref{prop:mono} next. But first we recall the budget-selection objective
\[
G_k(\beta):=\underbrace{1-(1-u(\beta))^k}_{=:F_k(\beta)}-S(\beta),
\qquad
\beta^\star(k):=\argmax_{\beta\in[0,\log|\W|]}G_k(\beta),
\]
with $u:[0,\log|\W|]\to(0,1)$ the optimal success at floor $\beta$ ($C^1$, strictly
decreasing), $S$ the $C^1$, strictly decreasing, strictly convex diversity cost, and
$T^\star(k)=1/\lambda^\star(\beta^\star(k))$ with $\lambda^\star(\cdot)>0$
non-increasing on the budget range under consideration. It is convenient to write $v(\beta):=1-u(\beta)\in(0,1)$, so $v$ is
strictly increasing with $v'(\beta)=-u'(\beta)>0$. We treat the budget $k\ge1$ as a
continuous parameter; the integer statement follows by restriction, and a discrete
cross-check. Figure~\ref{fig:tstar} illustrates the comparative static
synthetically and on a real language model.
\begin{figure}[!htbp]
\centering

\begin{subfigure}[t]{0.48\textwidth}
    \centering
    \includegraphics[width=\linewidth]{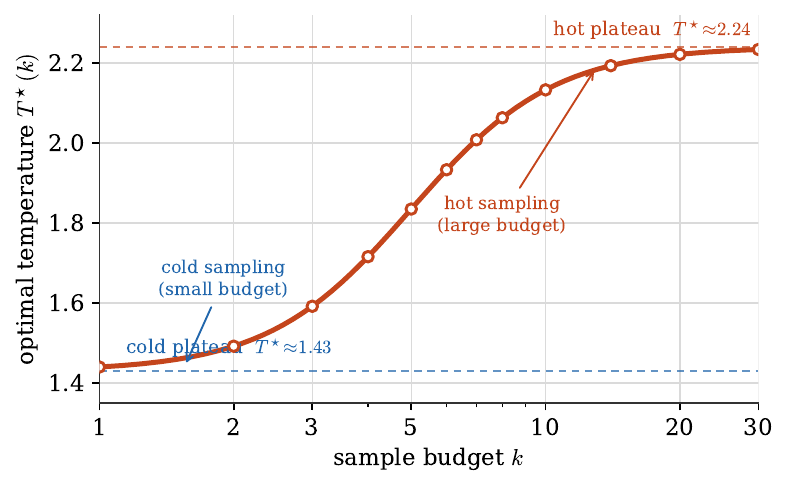}
    \caption{Synthetic illustration.}
    \label{fig:tstar-synth}
\end{subfigure}
\hfill
\begin{subfigure}[t]{0.48\textwidth}
    \centering
    \includegraphics[width=\linewidth]{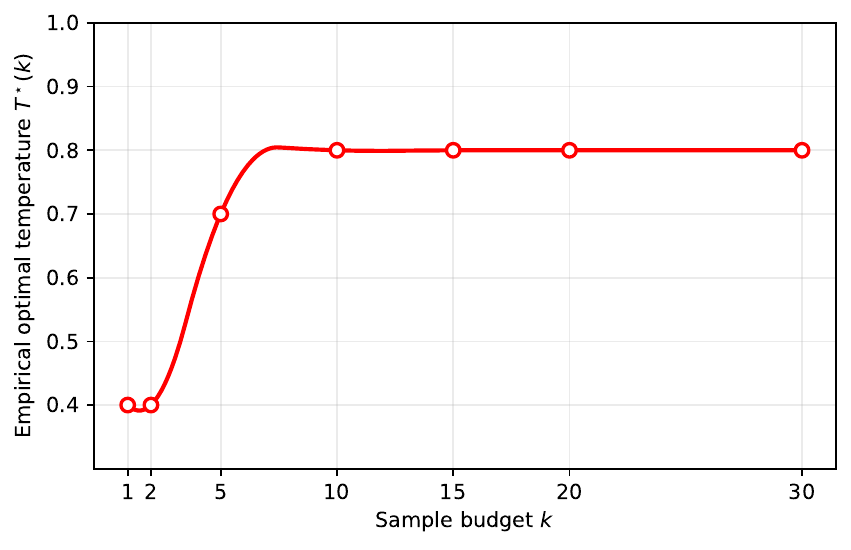}
    \caption{Empirical HumanEval sweep.}
    \label{fig:tstar-human}
\end{subfigure}

\caption{\textbf{Temperature as a function of sample budget.}
\textbf{(a)} For a synthetic instance, we select the entropy-floor point
maximizing pass@$k$ under the diversity cost $S$ and report the induced
temperature $T^\star(k)$.
\textbf{(b)} On 50 HumanEval problems with Qwen2.5-Coder-1.5B, we select the
best fixed temperature from $T\in[0.4,1.2]$ for each sample budget.
Both show the non-decreasing temperature--budget relation predicted by
Proposition~\ref{prop:mono}.}
\label{fig:tstar}
\end{figure}

\paragraph{The exact comparative static.}
Both derivatives below are elementary:
\begin{equation}\label{eq:app-Fk}
F_k'(\beta)=-k\,v(\beta)^{k-1}v'(\beta)<0,
\qquad
\partial_k F_k'(\beta)=-v'(\beta)\,v(\beta)^{k-1}\big(1+k\ln v(\beta)\big).
\end{equation}
The first shows $F_k$ is strictly decreasing in $\beta$ (a higher entropy floor
lowers coverage); the second is obtained by differentiating
$k\,v^{k-1}=k\,e^{(k-1)\ln v}$ in $k$.

\begin{lemma}[Sign of the comparative static]\label{lem:app-sign}
At a non-degenerate interior maximiser $\beta^\star(k)$ (i.e.\
$G_k'(\beta^\star)=0$ and $G_k''(\beta^\star)<0$), the selection $\beta^\star(\cdot)$
is $C^1$ near $k$ and
\[
\operatorname{sign}\frac{d\beta^\star}{dk}
=\operatorname{sign}\big(\partial_k F_k'(\beta^\star)\big)
=\operatorname{sign}\big(k(-\ln v^\star)-1\big)
=\operatorname{sign}\big(\passk(\beta^\star)-(1-e^{-1})\big),
\qquad v^\star:=1-u(\beta^\star(k)).
\]
\end{lemma}

\begin{proof}
Write the stationarity map $\Phi(\beta,k):=\partial_\beta G_k(\beta)=F_k'(\beta)-S'(\beta)$.
By hypothesis $\Phi(\beta^\star,k)=0$ and
$\partial_\beta\Phi(\beta^\star,k)=G_k''(\beta^\star)<0\neq0$, so the implicit
function theorem gives a $C^1$ selection $\beta^\star(\cdot)$ with
\[
\frac{d\beta^\star}{dk}
=-\frac{\partial_k\Phi(\beta^\star,k)}{\partial_\beta\Phi(\beta^\star,k)}
=-\frac{\partial_k F_k'(\beta^\star)}{G_k''(\beta^\star)},
\]
where $\partial_k\Phi=\partial_k F_k'$ because $S$ is $k$-free. As
$G_k''(\beta^\star)<0$, the sign equals that of $\partial_k F_k'(\beta^\star)$. By
\eqref{eq:app-Fk}, and since $v'>0$, $v^{k-1}>0$,
\[
\operatorname{sign}\big(\partial_k F_k'(\beta^\star)\big)
=\operatorname{sign}\big(-(1+k\ln v^\star)\big)
=\operatorname{sign}\big(k(-\ln v^\star)-1\big).
\]
Finally $k(-\ln v^\star)-1>0\iff (v^\star)^k<e^{-1}\iff
\passk(\beta^\star)=1-(v^\star)^k>1-e^{-1}$.
\end{proof}
\begin{corollary}[Hotter sampling for larger $k$]\label{prop:app-mono}
Under the stated regularity, on any budget interval $\mathcal K$ on which the optimal
operating point clears the coverage line,
\[
\passk\big(\beta^\star(k)\big)\ \ge\ 1-e^{-1}
\qquad\Longleftrightarrow\qquad
\big(1-u(\beta^\star(k))\big)^k\le e^{-1}
\qquad(k\in\mathcal K),
\]
the floor $\beta^\star(k)$ is non-decreasing on $\mathcal K$, strictly at every
interior $k$ where the inequality is strict; consequently
$T^\star(k)=1/\lambda^\star(\beta^\star(k))$ is non-decreasing on $\mathcal K$,
strictly at those $k$. Moreover this regime always obtains for all sufficiently
large $k$: $\beta^\star(k)\to\log|\W|$ and $\passk(\beta^\star(k))\to1$ as
$k\to\infty$.
\end{corollary}

\begin{proof}
On $\mathcal K$ the coverage inequality is exactly $(v^\star)^k\le e^{-1}$, i.e.\
$k(-\ln v^\star)-1\ge0$, so Lemma~\ref{lem:app-sign} gives
$\frac{d}{dk}\beta^\star(k)\ge0$ there, with strict inequality where
$(v^\star)^k<e^{-1}$. Integrating the $C^1$ selection over $\mathcal K$,
$\beta^\star$ is non-decreasing (strict at interior strict points). Since
$\lambda^\star(\cdot)$ is non-increasing, $\beta\mapsto1/\lambda^\star(\beta)$ is
non-decreasing, and composing with the non-decreasing $\beta^\star(\cdot)$ gives the
claim for $T^\star$.

For the large-$k$ statement, fix any $\beta_0<\log|\W|$. Because $u>0$ on
$[0,\log|\W|]$, $\min_{[\beta_0,\log|\W|]}u=:u_0>0$ and hence
$F_k(\beta)=1-(1-u(\beta))^k\to1$ uniformly on $[\beta_0,\log|\W|]$. Thus
$G_k\to1-S$ uniformly there, and since $S$ is strictly decreasing its maximiser over
$[\beta_0,\log|\W|]$ is $\log|\W|$; as $\beta_0$ is arbitrary,
$\beta^\star(k)\to\log|\W|$. Then $u(\beta^\star(k))\to u(\log|\W|)>0$, so
$\passk(\beta^\star(k))=1-(1-u(\beta^\star(k)))^k\to1\ge1-e^{-1}$, placing all large
$k$ in the regime.
\end{proof}

\begin{remark}[Where the naive global claim fails, and why it is expected]\label{rem:app-reverse}
Below the coverage line the comparative static \emph{reverses}: by
Lemma~\ref{lem:app-sign}, $\frac{d}{dk}\beta^\star(k)<0$ whenever
$\passk(\beta^\star(k))<1-e^{-1}$. The crossover budget is
\[
k_0(u^\star)=\frac{1}{-\ln(1-u^\star)}\ \approx\ \frac{1}{u^\star}\quad(u^\star\to0),
\]
which is exactly the maximiser of the pass@$k$ sensitivity
$\sigma_k=k(1-u)^{k-1}$ appearing in Theorem~\ref{thm:partial}(iii). At a base success rate
$u^\star\approx0.2$ one has $k_0\approx4.5$, so at $k=1,\dots,4$ the optimal floor
can \emph{decrease} in $k$ before the monotone regime takes over at $k\gtrsim5$; a
grid search over the quadratic cost $S(\beta)=c(\log|\W|-\beta)^2$ reproduces this
reversal precisely at the $1-e^{-1}$ threshold. The non-monotonicity is therefore
not a defect of the selection rule but the same low-coverage effect the paper
already isolates for $\sigma_k$; the folklore ``larger $k$ wants hotter'' is a
theorem precisely on the high-coverage side of $k_0$.
\end{remark}
\paragraph{Choosing the budgets $\beta$ and $\gamma$.}
The two budgets decouple, which makes them easy to set in practice. The
plausibility ceiling $\gamma$ fixes the \emph{temperature scale}: on the
tempered family $q_a\propto p^{a}$ the per-token cross-entropy
$\mathrm{CE}(q_a,p)$ is strictly monotone in the exponent $a=1/T$
(proof of Theorem~\ref{thm:collapse}), so $\gamma\leftrightarrow T$ is a
bijection. We therefore anchor $\gamma$ at a reference temperature $T_0$
trusted for the model and domain, $\gamma:=\mathrm{CE}(q_{1/T_0},p)$, computed
with a single forward pass over the calibration set and no sampling; $\gamma$
then reads as ``$q$ may drift from $p$ by at most the per-token cross-entropy
$T_0$ induces,'' and is loosened to explore or tightened for precision. The
entropy floor $\beta$ fixes \emph{diversity} and carries the budget dependence:
following Proposition~\ref{prop:mono}, we set
$\beta^\star(k)=\arg\max_{\beta\in[0,\log|W|]}\{\,1-(1-u(\beta))^k-S(\beta)\,\}$
with a strictly convex, decreasing diversity cost $S(\beta)=c(\log|W|-\beta)^2$,
whose maximizer is non-decreasing in $k$ on the high-coverage regime---cold
floors for small $k$, hot floors for large $k$. The single constant $c$ trades
coverage against diversity and is fixed once per domain on validation pass@$k$;
absent an optimization, any schedule rising from $H(p)$ toward $\log|W|$
reproduces the same behavior. Feasibility requires only joint attainability:
the base model realizes $(\beta,\gamma)=(H(p),H(p))$, so any $\beta\le H(p)$
with $\gamma\ge H(p)$ admits $q=p$, and driving $\beta$ above $H(p)$ demands a
$\gamma$ loose enough to license the hotter optimum. Both quantities are
occupancy-weighted per-token averages, hence estimated from a handful of
calibration prompts at one forward pass each---exactly what CTD's Phase~1 computes.

\section{Proof of Theorems 2 and 3}
\label{app:thm23}

Throughout this appendix $\theta=(\lambda,\mu)^\top\in\Theta\subset\mathbb R^2$, $\ell_t(w):=\log p_t(w\mid w^{t-1})$,
$V_t(w):=V_t(w\mid w^{t-1})$, $\phi_t(w):=(\ell_t(w),V_t(w))^\top$, and
\[
q^\theta_t(w) \;=\; \frac{\exp\big(\lambda\,\ell_t(w)+\mu\,V_t(w)\big)}{\sum_{w'\in\mathcal W}\exp\big(\lambda\,\ell_t(w')+\mu\,V_t(w')\big)}
\;\propto\; p_t(w)^\lambda\,e^{\mu V_t(w)},
\]
as in Section~\ref{sec:tracking}, with cumulative constraint map
\[
F_t(\theta) \;=\;
\begin{pmatrix}
\sum_{i=1}^t \big(\mathrm{CE}(q^\theta_i,p_i)-\gamma\big)\\[2pt]
\sum_{i=1}^t \big(H(q^\theta_i)-\beta\big)
\end{pmatrix},
\]
so that the offline active KKT point $\theta^\star=(\lambda^\star,\mu^\star)$ satisfies $F_n(\theta^\star)=0$, the prefix
tracker satisfies $F_t(\theta_t)=0$ for every $t$, and $R_t:=F_t(\theta^\star)$, $R:=\max_{1\le t\le n}\|R_t\|_2$ as in the
main text. Throughout, $V_t$ is held fixed with respect to $\theta$, isolating the tracking question addressed here
from the separate question of how $V_t$ is estimated.

\subsection{Proof of Theorem~\ref{thm:prefix-tracking} (Prefix-price tracking)}

Recall Assumption~\ref{ass:jacobian}: for every $t$ at which \eqref{eq:prefix-root} is solved,
$J_t := \int_0^1 DF_t\big(\theta_t+s(\theta^\star-\theta_t)\big)\,ds$ satisfies $\sigma_{\min}(J_t)\ge ct$ for a constant
$c>0$ independent of $t$.

\begin{proof}
By the fundamental theorem of calculus along the segment joining $\theta_t$ and $\theta^\star$,
\[
F_t(\theta^\star) - F_t(\theta_t) \;=\; \int_0^1 DF_t\big(\theta_t+s(\theta^\star-\theta_t)\big)\,ds\;(\theta^\star-\theta_t)
\;=\; J_t\,(\theta^\star-\theta_t).
\]
Since $F_t(\theta_t)=0$ and $F_t(\theta^\star)=R_t$ by definition, this reads $R_t = J_t(\theta^\star-\theta_t)$. Taking
norms and using $\|J_t x\|_2 \ge \sigma_{\min}(J_t)\|x\|_2$,
\[
\|R_t\|_2 \;\ge\; \sigma_{\min}(J_t)\,\|\theta^\star-\theta_t\|_2 \;\ge\; ct\,\|\theta_t-\theta^\star\|_2,
\]
by Assumption~\ref{ass:jacobian}. Rearranging, and using $\|R_t\|_2\le R$ for every $t$, gives
$\|\theta_t-\theta^\star\|_2 \le \|R_t\|_2/(ct) \le R/(ct)$.
\end{proof}

\begin{remark}
No global uniqueness of the root of $F_t$ is assumed. The bound only requires that the particular root $\theta_t$
selected by the tracker and the offline point $\theta^\star$ are joined by a segment on which the averaged Jacobian
$J_t$ is uniformly nonsingular in the sense of Assumption~\ref{ass:jacobian}; this is exactly what the mean-value
step above uses.
\end{remark}

\subsection{A conditional KL perturbation lemma}

The bridge from the parameter-space guarantee of Theorem~\ref{thm:prefix-tracking} to a distributional and, ultimately,
a $\mathrm{pass@}k$ guarantee goes through the local geometry of the exponential family $q^\theta_t$, not through
summing total-variation distances directly (the KL-versus-TV discussion in Section~\ref{sec:tracking}). For $\theta\in\Theta$ let
$A_t(\theta) := \log\sum_{w\in\mathcal W}\exp\big(\theta^\top\phi_t(w)\big)$ be the log-partition function of $q^\theta_t$,
so that $\nabla A_t(\theta) = \mathbb E_{q^\theta_t}[\phi_t(W)]$ and $\nabla^2 A_t(\theta) = \mathrm{Cov}_{q^\theta_t}(\phi_t(W))$.

\begin{lemma}[Conditional KL perturbation]
\label{lem:kl-perturbation}
Under Assumption~\ref{ass:cov},
\[
D_{\mathrm{KL}}\big(q^{\theta_t}_t \,\big\|\, q^{\theta^\star}_t\big) \;\le\; \frac{M}{2}\,\|\theta_t-\theta^\star\|_2^2.
\]
\end{lemma}

\begin{proof}
For an exponential family with natural parameter $\theta$ and log-partition $A_t$, the KL divergence between two
members is exactly the Bregman divergence of $A_t$:
\[
D_{\mathrm{KL}}\big(q^{\theta_t}_t\,\|\,q^{\theta^\star}_t\big)
= A_t(\theta^\star) - A_t(\theta_t) - \nabla A_t(\theta_t)^\top(\theta^\star-\theta_t).
\]
By Taylor's theorem with integral remainder,
\[
A_t(\theta^\star) - A_t(\theta_t) - \nabla A_t(\theta_t)^\top(\theta^\star-\theta_t)
= \int_0^1 (1-s)\,(\theta^\star-\theta_t)^\top \nabla^2 A_t\big(\theta_t+s(\theta^\star-\theta_t)\big)(\theta^\star-\theta_t)\,ds.
\]
Since $\nabla^2 A_t(\theta) = \mathrm{Cov}_{q^\theta_t}(\phi_t(W))$, Assumption~\ref{ass:cov} bounds its largest
eigenvalue by $M$ at every point on the segment, so the integrand is at most
$(1-s)\,M\,\|\theta^\star-\theta_t\|_2^2$; integrating over $s\in[0,1]$ contributes a factor $\int_0^1(1-s)\,ds=\tfrac12$,
giving the claim.
\end{proof}

\subsection{Proof of Theorem~\ref{thm:passk-tracking} (Sequence-law and pass@k stability)}

Let $Q$ denote the autoregressive sequence law induced by the online conditionals $\{q^{\theta_t}_t\}_{t=1}^n$, and
$Q^\star$ the sequence law induced by the offline optimum $\{q^{\theta^\star}_t\}_{t=1}^n$.

\begin{proof}
\emph{KL bound.} By the chain rule for relative entropy on autoregressive sequence laws,
\[
D_{\mathrm{KL}}(Q\,\|\,Q^\star) \;=\; \sum_{t=1}^n \mathbb E_Q\Big[D_{\mathrm{KL}}\big(q^{\theta_t}_t(\cdot\mid W^{t-1})\,\big\|\,q^{\theta^\star}_t(\cdot\mid W^{t-1})\big)\Big].
\]
Applying Lemma~\ref{lem:kl-perturbation} termwise, then Theorem~\ref{thm:prefix-tracking} ($\|\theta_t-\theta^\star\|_2\le R/(ct)$),
\[
D_{\mathrm{KL}}(Q\,\|\,Q^\star) \;\le\; \frac{M}{2}\sum_{t=1}^n \|\theta_t-\theta^\star\|_2^2
\;\le\; \frac{M}{2}\sum_{t=1}^n \frac{R^2}{c^2t^2}
\;=\; \frac{MR^2}{2c^2}\sum_{t=1}^n\frac1{t^2}
\;\le\; \frac{M\pi^2}{12\,c^2}\,R^2,
\]
using $\sum_{t=1}^\infty 1/t^2=\pi^2/6$.

\emph{TV and success-probability bounds.} Pinsker's inequality, $\mathrm{TV}(Q,Q^\star)\le\sqrt{\tfrac12 D_{\mathrm{KL}}(Q\|Q^\star)}$,
gives $\mathrm{TV}(Q,Q^\star) \le \tfrac{\pi\sqrt M}{2\sqrt6\,c}\,R$. Since $U:\mathcal W^n\to[0,1]$ is bounded,
\[
|u_Q - u_{Q^\star}| \;=\; \big|\mathbb E_Q[U]-\mathbb E_{Q^\star}[U]\big| \;\le\; \mathrm{TV}(Q,Q^\star) \;\le\; \frac{\pi\sqrt M}{2\sqrt6\,c}\,R.
\]

\emph{$\mathrm{pass@}k$ bounds.} With $g_k(u):=1-(1-u)^k$, $g_k'(u)=k(1-u)^{k-1}\in[0,k]$ for $u\in[0,1]$, so $g_k$ is
$k$-Lipschitz and $|\mathrm{pass@}k(Q)-\mathrm{pass@}k(Q^\star)| = |g_k(u_Q)-g_k(u_{Q^\star})| \le k\,|u_Q-u_{Q^\star}|$;
combined with $\mathrm{pass@}k\in[0,1]$ this gives the stated
$\big|\mathrm{pass@}k(Q)-\mathrm{pass@}k(Q^\star)\big| \le \min\big\{1,\ k\,\tfrac{\pi\sqrt M}{2\sqrt6\,c}R\big\}$.
For the sharper form, the mean value theorem gives $\xi$ between $u_Q$ and $u_{Q^\star}$ with
$|g_k(u_Q)-g_k(u_{Q^\star})| = k(1-\xi)^{k-1}|u_Q-u_{Q^\star}|$; since $k(1-u)^{k-1}$ is decreasing in $u$ and
$\xi\ge u_{\min}:=\min\{u_Q,u_{Q^\star}\}$, this is at most $k(1-u_{\min})^{k-1}|u_Q-u_{Q^\star}|$, giving
$\big|\mathrm{pass@}k(Q)-\mathrm{pass@}k(Q^\star)\big| \le k(1-u_{\min})^{k-1}\,\tfrac{\pi\sqrt M}{2\sqrt6\,c}\,R$.
\end{proof}

\begin{remark}[Reading $R$]
\label{rem:reading-R}
$R=\max_{t\le n}\|F_t(\theta^\star)\|_2$ measures the maximal prefix imbalance of the \emph{final, offline} prices: how
far $\theta^\star$ is from satisfying the length-$t$ truncated budgets, at the worst truncation point. Theorem~\ref{thm:passk-tracking}
is a deterministic stability statement -- small prefix residual implies small $\mathrm{pass@}k$ deviation -- rather than
a regret bound in the online-learning sense, since $V_t$ is held fixed throughout this appendix; Appendix~\ref{app:mc}
addresses the separate problem of estimating $V_t$ from terminal verifier feedback. Because $\mathrm{pass@}k\in[0,1]$,
any bound exceeding one is vacuous, which is why the loose form of the bound is explicitly truncated at $1$.
\end{remark}
\section{Online Estimation of the Verifier Potential in J-EOS}
\label{app:mc}

J-EOS needs an estimate of the Bellman potential $V_t(w\mid w^{t-1})$ from
terminal pass/fail feedback.  After the causal price-tracking pass of round $r$,
it has a sequence of prices $\theta_{r,t}=(\lambda_{r,t},\mu_{r,t})$,
$t=1,\dots,n$.  We freeze the resulting J-EOS decoder and use it as an
importance-sampling behavior policy.  This appendix states the corresponding
rollout identity, the step-only estimator used by Algorithm~\ref{alg:jeos}, and
the approximation introduced by collapsing prefixes into a table indexed only
by $(t,w)$.

\paragraph{Target prices and frozen behavior.}
At the end of round $r$, let
\[
\vartheta_r=(\lambda_r,\mu_r):=\theta_{r,n}, \qquad \mu_r>0,
\]
be the price pair whose Bellman potential we want to estimate.  Freeze the
within-round prices $\{\theta_{r,t}\}_{t=1}^n$ and define the behavior policy
\begin{equation}
 b_{r,t}(w\mid w^{t-1})
 :=
 \frac{p_t(w\mid w^{t-1})^{\lambda_{r,t}}
       \exp\!\big(\mu_{r,t}\widehat V^{\rm prev}_t(w)\big)}
      {\sum_{a\in\W}p_t(a\mid w^{t-1})^{\lambda_{r,t}}
       \exp\!\big(\mu_{r,t}\widehat V^{\rm prev}_t(a)\big)} .
\label{eq:jeos-behavior}
\end{equation}
After freezing \eqref{eq:jeos-behavior}, draw a fresh batch of $m$ independent
rollouts from $b_r$.  Using a fresh batch is important: the causal trajectory
used to construct $\theta_{r,1:n}$ is adaptive and is not, in general, an i.i.d.
draw from the subsequently frozen behavior policy.

For a fixed target pair $\vartheta=(\lambda,\mu)$, define the Bellman partition
function after committing to $(w^{t-1},w)$ by
\begin{equation}
 Z^{\vartheta}_{t+1}(w^{t-1}w)
 :=
 \sum_{w_{t+1:n}}
 \left[\prod_{s=t+1}^n
 p_s(w_s\mid w^{s-1})^{\lambda}\right]
 e^{\mu U(w^n)},
 \qquad
 V_t^{\vartheta}(w\mid w^{t-1})
 =\frac1\mu\log Z^{\vartheta}_{t+1}(w^{t-1}w).
\label{eq:jeos-target-Z}
\end{equation}
This is exactly the potential of Lemma~\ref{lem:bellman}, evaluated at the
current target prices $\vartheta_r$.

\begin{proposition}[Off-policy rollout identity for J-EOS]
\label{prop:jeos-rollout}
Fix a round $r$ and target prices $\vartheta_r=(\lambda_r,\mu_r)$ with
$\mu_r>0$.  For every $t$, prefix $w^{t-1}$, and token $w$, if the behavior
policy has full support, then
\begin{equation}
Z^{\vartheta_r}_{t+1}(w^{t-1}w)
=
\E_{w_{t+1:n}\sim b_r}
\!\left[
 e^{\mu_r U(w^n)}
 \prod_{s=t+1}^{n}
 \frac{p_s(w_s\mid w^{s-1})^{\lambda_r}}
      {b_{r,s}(w_s\mid w^{s-1})}
 \ \middle|\ w^t
\right].
\label{eq:jeos-rollout}
\end{equation}
Consequently,
$V_t^{\vartheta_r}(w\mid w^{t-1})=\mu_r^{-1}\log
Z^{\vartheta_r}_{t+1}(w^{t-1}w)$.
\end{proposition}

\begin{proof}
Expand the expectation in \eqref{eq:jeos-rollout}.  The continuation probability
under $b_r$ is $\prod_{s>t}b_{r,s}(w_s\mid w^{s-1})$; it cancels the denominator
of the likelihood ratio, leaving
$\prod_{s>t}p_s(w_s\mid w^{s-1})^{\lambda_r}e^{\mu_rU(w^n)}$.
Summing over all continuations gives \eqref{eq:jeos-target-Z}.
\end{proof}

\paragraph{Step-only table used by J-EOS.}
Let $w^n_{(1)},\dots,w^n_{(m)}\stackrel{\rm iid}{\sim}b_r$ be the fresh rollout
batch and $U_{(j)}=U(w^n_{(j)})$.  For every visited cell $(t,w)$ define
\begin{align}
X^{(r,t)}_{(j)}
&:= e^{\mu_r U_{(j)}}
\prod_{s=t+1}^{n}
\frac{p_s(w_{(j),s}\mid w^{s-1}_{(j)})^{\lambda_r}}
     {b_{r,s}(w_{(j),s}\mid w^{s-1}_{(j)})},
\label{eq:jeos-weight}\\
\widehat Z^{(r)}_{t+1}(w)
&:=\frac{1}{N_{r,t}(w)}
\sum_{j:\,w_{(j),t}=w}X^{(r,t)}_{(j)},
\qquad
\widehat V^{(r)}_t(w)
:=\frac1{\mu_r}\log \widehat Z^{(r)}_{t+1}(w),
\label{eq:jeos-steponly}
\end{align}
where $N_{r,t}(w)=\#\{j:w_{(j),t}=w\}$.  Cells with
$N_{r,t}(w)=0$ retain their previous value (or are smoothed with a stated
pseudocount); they must not be assigned the logarithm of an empty sum.

Equation~\eqref{eq:jeos-steponly} is the update intended at the potential-update
line of Algorithm~\ref{alg:jeos}.  It requires storing the behavior probability
$b_{r,s}(w_{(j),s}\mid w^{s-1}_{(j)})$ for each sampled token, but it does not
require an oracle value of $(\lambda^\star,\mu^\star)$.

\paragraph{What the step-only table identifies.}
Because $\widehat V_t$ is indexed by $(t,w)$ rather than the full prefix, the
estimator cannot recover the exact prefix-conditioned potential in general.
Conditioning only on the event $w_t=w$ gives
\begin{equation}
\bar Z^{(r)}_{t+1}(w)
:=
\E_{w^{t-1}\sim b_r(\cdot\mid w_t=w)}
\!\left[Z^{\vartheta_r}_{t+1}(w^{t-1}w)\right]
=
\E_{w^{t-1}\sim b_r(\cdot\mid w_t=w)}
\!\left[e^{\mu_r V_t^{\vartheta_r}(w\mid w^{t-1})}\right],
\label{eq:jeos-marginal-target}
\end{equation}
and therefore
\begin{equation}
\bar V^{(r)}_t(w)
:=\frac1{\mu_r}\log \bar Z^{(r)}_{t+1}(w).
\label{eq:jeos-marginal-V}
\end{equation}
Thus the table converges to a behavior-weighted \emph{soft average} of the
prefix-conditioned potential.  This is an approximation intrinsic to the
$O(n|\W|)$ step-only representation, not a Monte-Carlo error.  It disappears
when $V_t^{\vartheta_r}(w\mid w^{t-1})$ is prefix-independent, and it can be
reduced by conditioning the value estimator on a richer state or learned hidden
representation.

\begin{theorem}[Per-round consistency of the J-EOS estimator]
\label{thm:jeos-consistency}
Fix a round $r$ and suppose that the $m$ rollouts are i.i.d.
from the frozen behavior policy $b_r$, that
$\Pr_{b_r}(W_t=w)\ge\rho_\star>0$ for all cells of interest, that
$p_s(a\mid h)\ge\rho>0$, and that
$b_{r,s}(a\mid h)\ge\kappa>0$ on the same support.  Set
\begin{equation}
\Delta_{r,t}
:=
 e^{\mu_r}\kappa^{-(n-t)}
 -\rho^{\lambda_r(n-t)}.
\label{eq:jeos-range}
\end{equation}
For $\delta\in(0,1)$, if
$m\ge\frac{8}{\rho_\star}\log\frac{2n|\W|}{\delta}$, then with probability
at least $1-\delta$, simultaneously for all visited cells,
\begin{equation}
\left|\widehat Z^{(r)}_{t+1}(w)-\bar Z^{(r)}_{t+1}(w)\right|
\le
\Delta_{r,t}
\sqrt{\frac{\log(4n|\W|/\delta)}{m\rho_\star}}.
\label{eq:jeos-concentration}
\end{equation}
Consequently $\widehat V^{(r)}_t(w)$ is consistent for
$\bar V^{(r)}_t(w)$.  Its logarithm produces the usual downward Jensen bias,
\begin{equation}
\E\!\left[\widehat V^{(r)}_t(w)\right]
=
\bar V^{(r)}_t(w)
-
\frac{\Var(\widehat Z^{(r)}_{t+1}(w))}
     {2\mu_r\,\bar Z^{(r)}_{t+1}(w)^2}
+o\!\left(\Var(\widehat Z^{(r)}_{t+1}(w))\right).
\label{eq:jeos-logbias}
\end{equation}
\end{theorem}

\begin{proof}
Conditional on $W_t=w$, the fresh rollouts are i.i.d. and the importance-weighted
quantity in \eqref{eq:jeos-weight} has mean
$\bar Z^{(r)}_{t+1}(w)$ by Proposition~\ref{prop:jeos-rollout} followed by the
prefix marginalization in \eqref{eq:jeos-marginal-target}.  Moreover,
$U\in[0,1]$ implies $e^{\mu_rU}\in[1,e^{\mu_r}]$, while
$p_s^{\lambda_r}/b_{r,s}\in[\rho^{\lambda_r},\kappa^{-1}]$.
Hence each weight lies in
$[\rho^{\lambda_r(n-t)},e^{\mu_r}\kappa^{-(n-t)}]$, an interval of width
$\Delta_{r,t}$.  A multiplicative Chernoff bound gives
$N_{r,t}(w)\ge\tfrac12m\rho_\star$ simultaneously with probability at least
$1-\delta/2$ under the stated sample-size condition.  Conditional on that
event, Hoeffding's inequality applied to each cell, followed by a union bound,
gives \eqref{eq:jeos-concentration}.  Finally,
$z\mapsto\mu_r^{-1}\log z$ is smooth on the positive range above; Jensen's
inequality gives the sign of the bias and a second-order Taylor expansion gives
\eqref{eq:jeos-logbias}.
\end{proof}

\begin{remark}[This does not by itself prove joint J-EOS convergence]
\label{rem:jeos-two-timescale}
Appendix~\ref{app:thm23} proves stability of the price tracker with the potential
held fixed.  The theorem above proves consistency of one potential update with
the target prices and behavior policy frozen within a round.  Combining these
two facts is a two-timescale tracking argument; it requires an additional
assumption controlling how quickly $\vartheta_r$ and the induced potential move
between rounds.  Therefore the present theory supports J-EOS as a per-round
price-tracking plus potential-estimation procedure, but it should not be stated
as a full no-regret or global-convergence theorem unless that extra argument is
added.
\end{remark}

\begin{remark}[Boundary case $\mu_r=0$]
The representation $V=\mu^{-1}\log Z$ is undefined at $\mu=0$.  Since J-EOS is
initialized at $\mu_0=0$, the first round should retain the zero-potential warm
start until the tracked solution enters the interior ($\mu_r>0$), or else the
implementation should separately track the log-partition product $\mu V$ near
the boundary.  The interior analysis in Section~\ref{sec:tracking} and
Appendix~\ref{app:thm23} already excludes the degenerate boundary case.
\end{remark}

\section{Verifier-feedback lower bound: Proof of Theorem~\ref{thm:lower}}\label{app:lower}
\newcommand{\Wn}{\mathcal{W}^n}
\newcommand{\Prob}{\mathbb{P}}
\newcommand{\etamax}{\eta_{\max}}
\newcommand{\qstar}{q^{\star}}
\newcommand{\hr}{\hat{\eta}_r}
\newcommand{\Ber}{\mathrm{Ber}}

The argument is Le~Cam's two--point method: we exhibit two instances whose $\passk$ curves agree everywhere except on two disjoint tilt windows, where one is higher and the other lower by $\Theta(\Delta)$. Distinguishing them under terminal bandit verifier feedback costs $\Omega(1/\Delta^2)$ observations, each misallocated round incurs regret $\Omega(\Delta)$, and balancing $\Delta\asymp N^{-1/2}$ yields cumulative regret $\Omega(\sqrt{N})$. The result is independent of any particular online update rule.

{\bf Setup and reduction to tilt regret.}
At rollout $r=1,\dots,N$ the algorithm $\mathcal{A}$ picks a tilt $\hr\in[0,\etamax]$ measurable w.r.t.\ the past, draws the $k$ sequences of round $r$ i.i.d.\ from the offline optimum $\qstar_{\hr}$ of Theorem~\ref{thm:tilt} (with $\lambda$ adjusted to meet the entropy floor), and observes only the binary verdicts $U_j\in\{0,1\}$, $j\in[k]$. Write $u_q:=\E_{w^n\sim q}[U(w^n)]$ and
\begin{equation}\label{eq:fdef}
  f(\eta):=\passk(\qstar_\eta)=1-\bigl(1-u_{\qstar_\eta}\bigr)^{k},
  \qquad
  \eta^{\star}\in\argmax_{\eta\in[0,\etamax]} f(\eta).
\end{equation}
Under the protocol the round-$r$ sampler is $q_r=\qstar_{\hr}$, so the regret of $\mathcal{A}$ is exactly the tilt regret
\begin{equation}\label{eq:reg}
  \Reg_{\mathcal{A}}(N)
  = N\,\passk(\qstar)-\sum_{r=1}^{N}\passk(q_r)
  = \sum_{r=1}^{N}\bigl(f(\eta^{\star})-f(\hr)\bigr),
\end{equation}
since $\passk(\qstar)=f(\eta^{\star})$ for the family we construct. The hard family is chosen so that the feasible tilted policies are parameterized by this single scalar tilt, hence lower-bounding \eqref{eq:reg} is algorithm-independent within the stated feedback/action model.

To host the perturbation we require the base instance to be non-degenerate and the verifier family rich, which is exactly the hypothesis referenced by Theorem~\ref{thm:lower}.
\begin{assumption}[Non-degenerate, rich verifier family]\label{ass:lb}
There is a base instance $(p,U,\mathcal{W},n,k)$ whose curve $f$ in \eqref{eq:fdef} is $C^2$ on $[0,\etamax]$ with $|f''|\le M$, $f'(\eta^{\star})=0$, and $f(\eta^{\star})\in(c_0,1-c_0)$ for absolute constants $M,c_0>0$. Moreover the family $\bigl\{\eta\mapsto \Prob_{\qstar_\eta}(S):S\subseteq\Wn\text{ measurable}\bigr\}$ is, after scaling, $L^\infty$-dense in the smooth $\etamax$-supported perturbations of $[0,\etamax]$.
\end{assumption}
\noindent The smoothness and interior-optimum conditions follow from Theorem~1 (the tilted optimum, hence $f$, is a smooth function of $\eta$); $f(\eta^{\star})\in(c_0,1-c_0)$ excludes the trivial regime where bandit feedback is uninformative; the richness clause is what lets us realize a prescribed curve perturbation at the verifier level (Lemma~\ref{lem:realize}).

{\bf The hard pair.}
\emph{Bumps.} Fix $\eta_-=\etamax/3$, $\eta_+=2\etamax/3$. Let $b:\mathbb{R}\to[0,1]$ be smooth and compactly supported with $b(0)=1$ and $\operatorname{supp} b\subseteq[-\etamax/6,\etamax/6]$, and set $\psi(\eta):=b(\eta-\eta_+)-b(\eta-\eta_-)$. The supports $B_-=[\etamax/6,\etamax/2]$ and $B_+=[\etamax/2,5\etamax/6]$ are disjoint (they meet only at the endpoint $\etamax/2$, where both bumps vanish), so $\psi(\eta_+)=+1$, $\psi(\eta_-)=-1$, and $\|\psi\|_\infty\le1$.

\emph{Curves.} Take the base verifier to have a \emph{flat} curve $f_0\equiv c$ with $c\in(c_0,1-c_0)$. For $\varepsilon\in\{+1,-1\}$ define the target
\begin{equation}\label{eq:target}
  f_\varepsilon(\eta)=c+\varepsilon\,\Delta\,\psi(\eta),
  \qquad \Delta\in(0,c_0/2).
\end{equation}
A flat base makes every gap inequality below hold directly, with no slope condition on $f_0$.

\emph{Realization in $\passk$ space.} The map $u\mapsto 1-(1-u)^k$ is nonlinear, so the perturbation must be designed one layer down, on the per-sample rate. Let $u_0(\eta)=u_{\qstar_\eta}$ under the base verifier, so $1-c=(1-u_0(\eta))^k$; flatness makes $u_0$ constant in $\eta$. To move the curve to \eqref{eq:target} the per-sample rate must move by $g_\varepsilon$ solving $(1-u_0-g_\varepsilon)^k=(1-u_0)^k-\varepsilon\Delta\psi$, i.e.
\begin{equation}\label{eq:jac}
  g_\varepsilon(\eta)=\frac{\varepsilon\,\Delta\,\psi(\eta)}{J}+O(\Delta^2),
  \qquad
  J:=k\,(1-u_0)^{k-1}=k\,(1-c)^{(k-1)/k}.
\end{equation}
Since $(k-1)/k\in(0,1)$ and $1-c\in(c_0,1-c_0)$ we have $k c_0\le J\le k$, hence $\|g_\varepsilon\|_\infty\le \Delta/(k c_0)+O(\Delta^2)$. The budget $k$ enters \eqref{eq:target} \emph{only} through this $O(1/k)$ rescaling of the per-sample rate; the target curves are $k$-free except through $c_0$. The next lemma records the two-sided scale of $u_0$ that this induces; it is what keeps the divergence bound below independent of $k$.

\begin{lemma}[Per-sample scale]\label{lem:scale}
For the flat base, $u_0=1-(1-c)^{1/k}$ is constant in $\eta$ and
\[
  \frac{c_0}{k}\ \le\ u_0\ \le\ c, \qquad 1-u_0\ \ge\ c_0,
  \qquad\text{so}\qquad J=k(1-u_0)^{k-1}\in[k c_0,\ k].
\]
\end{lemma}
\begin{proof}
Flatness gives $(1-u_0)^k=1-c$, so $u_0=1-(1-c)^{1/k}$ is constant. Since $1-c\in(0,1)$ and $1/k\le1$ we have $(1-c)^{1/k}\ge1-c$, whence $u_0\le c$ and $1-u_0\ge1-c\ge c_0$. For the lower bound, the map $x\mapsto\varphi(x):=(1-(1-c)^{x})/x$ is nonincreasing on $(0,1]$: with $t:=x\log\tfrac{1}{1-c}>0$ one has $\varphi(x)=\log\tfrac{1}{1-c}\cdot(1-e^{-t})/t$, and $t\mapsto(1-e^{-t})/t$ is decreasing while $t$ increases in $x$. Hence $k u_0=\varphi(1/k)\ge\varphi(1)=1-(1-c)=c$, i.e.\ $u_0\ge c/k\ge c_0/k$. Finally $J=k(1-u_0)^{k-1}=k(1-c)^{(k-1)/k}\in[k(1-c),k]\subseteq[kc_0,k]$.
\end{proof}

\begin{lemma}[Realizability]\label{lem:realize}
Under Assumption~\ref{ass:lb}, for every $\delta\in(0,\tfrac12]$ there is $\Delta_0>0$ such that for all $\Delta\le\Delta_0$ there exist measurable $h_\varepsilon:\Wn\to[-c_0/2,c_0/2]$ with $U_\varepsilon:=U_0+h_\varepsilon\in[0,1]$ and
\[
  \bigl\|\,\eta\mapsto\E_{\qstar_\eta}[h_\varepsilon]-g_\varepsilon\,\bigr\|_\infty\ \le\ \delta\,\|g_\varepsilon\|_\infty .
\]
Consequently the realized curve is $f_\varepsilon(\eta)=c+\varepsilon\Delta\psi(\eta)+e_\varepsilon(\eta)$ with $\|e_\varepsilon\|_\infty\le\delta'\Delta$, where $\delta'=\delta/c_0+O(\Delta)\to0$ as $\delta\to0$.
\end{lemma}
\begin{proof}
By \eqref{eq:jac}, $g_\varepsilon$ is smooth, $\etamax$-supported with $\|g_\varepsilon\|_\infty=O(\Delta/k)$. Normalize $\bar g_\varepsilon:=g_\varepsilon/\|g_\varepsilon\|_\infty$ (smooth, $\etamax$-supported, unit sup-norm). By the richness clause, scaled differences of $\eta\mapsto\Prob_{\qstar_\eta}(S)$ are $L^\infty$-dense in such functions; pick $S_+,S_-$ and a scale so that $\beta\,(\Prob_{\qstar_\eta}(S_+)-\Prob_{\qstar_\eta}(S_-))$ approximates $\bar g_\varepsilon$ to sup-norm $\delta$, and set $h_\varepsilon:=\|g_\varepsilon\|_\infty\,\beta(\mathbf 1_{S_+}-\mathbf 1_{S_-})$, so $\E_{\qstar_\eta}[h_\varepsilon]=g_\varepsilon+ (\text{error }\le\delta\|g_\varepsilon\|_\infty)$ and $\|h_\varepsilon\|_\infty=O(\Delta/k)\le c_0/2$ for $\Delta$ small. To keep $U_\varepsilon=U_0+h_\varepsilon\in[0,1]$ with $U_0\in\{0,1\}$, take $S_+\subseteq\{U_0=0\}$ and $S_-\subseteq\{U_0=1\}$ (both carry $\qstar_\eta$-mass since $c\in(c_0,1-c_0)$), so $h_\varepsilon\ge0$ where $U_0=0$ and $h_\varepsilon\le0$ where $U_0=1$; replacing $U_\varepsilon$ by the randomized verifier that outputs $1$ w.p.\ $U_\varepsilon(w^n)$ makes it binary without changing the per-sample rate. Finally the per-sample error propagates to the curve through $u\mapsto1-(1-u)^k$: since $\mathrm d f=J\,\mathrm d u$ with $J\le k$ and $J\ge kc_0$, the pass$@k$ error is $|e_\varepsilon|\le J\cdot\delta\|g_\varepsilon\|_\infty+O(\Delta^2)\le k\cdot\delta\cdot\tfrac{\Delta}{kc_0}+O(\Delta^2)=(\delta/c_0)\Delta+O(\Delta^2)$, i.e.\ $\delta'=\delta/c_0+O(\Delta)$.
\end{proof}
\noindent (The factor $J\asymp k$ amplifies the $O(\Delta/k)$ per-sample error into an $O(\Delta)$ curve error; this is why the realization is designed and controlled at the per-sample scale, and why $\delta'$ comes out $k$-free.)

Fix $\delta$ so that $\delta'\le 1/4$ and absorb $e_\varepsilon$ into the effective gap $\Delta':=\tfrac34\Delta$, written $\Delta$ below; \eqref{eq:target} is then exact up to this constant factor.

\begin{figure}[!htbp]
\centering
\includegraphics[width=0.85\textwidth]{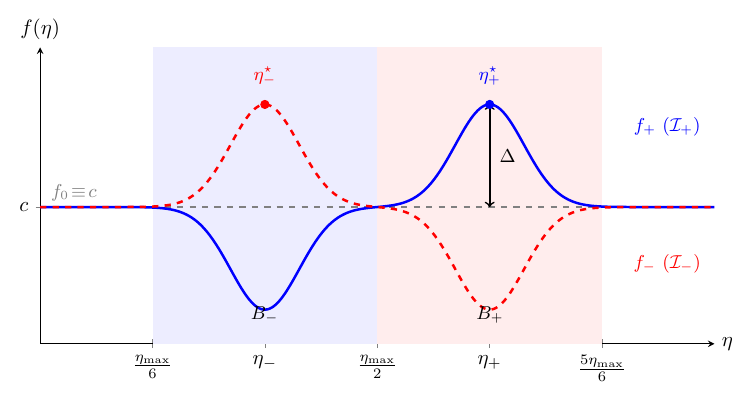}
\caption{\textbf{The hard pair.} The two $\mathrm{pass}@k$ curves coincide outside $B_-\cup B_+$ and differ by $\Delta$ inside; $f_+$ peaks in $B_+$ and $f_-$ in $B_-$. Under bandit feedback the learner reads the round's $k$ verdicts (equivalently, up to constants, the pass bit $\Pi_r\sim\mathrm{Ber}(f_\varepsilon(\hat\eta_r))$), so the two worlds are statistically close ($\mathrm{KL}\le C_0\Delta^2$ per round) yet have optima in different bumps.}
\label{fig:hardpair}
\end{figure}

\emph{Gap structure} (Figure~\ref{fig:hardpair}). With $f_0$ flat and $\psi\ge0$ on $B_+$, $\psi\le0$ on $B_-$, the maximizer of $f_\varepsilon$ lies in $B_\varepsilon$ with $f_\varepsilon(\eta^\star_\varepsilon)=c+\Delta$. For every $\eta\in B_-$, $f_+(\eta)=c-\Delta\,b(\eta-\eta_-)\le c$, so
\begin{equation}\label{eq:gap}
  f_+(\eta^\star_+)-f_+(\eta)=\Delta\bigl(1+b(\eta-\eta_-)\bigr)\ge\Delta
  \qquad\forall\,\eta\in B_-,
\end{equation}
and symmetrically for $f_-$ with $B_\pm$ swapped. Off both bumps $f_\varepsilon\equiv c$, so any off-bump query has regret exactly $\Delta$.

\medskip
\noindent {\bf Closing the argument.}
\emph{Occupancy regret.} Let $N_-=\sum_r\mathbf 1\{\hr\in B_-\}$, $N_+=\sum_r\mathbf 1\{\hr\in B_+\}$, and $N_\circ=N-N_--N_+$. By \eqref{eq:gap} and the exact off-bump regret $\Delta$,
\begin{equation}\label{eq:occ}
  \Reg_{\mathcal A}\ge\Delta\bigl(N_-+\tfrac12 N_\circ\bigr)\ \text{on}\ \mathcal I_+,
  \qquad
  \Reg_{\mathcal A}\ge\Delta\bigl(N_++\tfrac12 N_\circ\bigr)\ \text{on}\ \mathcal I_-,
\end{equation}
where $(N_-+\tfrac12 N_\circ)+(N_++\tfrac12 N_\circ)=N$ exactly.

\emph{KL divergence.} Round $r$ produces the $k$ i.i.d.\ verdicts $U^{(r)}_1,\dots,U^{(r)}_k\sim\Ber\bigl(u_\varepsilon(\hr)\bigr)$, where $u_\varepsilon(\eta):=u_{\qstar_\eta}$ under the realized verifier $U_\varepsilon$ of Lemma~\ref{lem:realize}; these verdicts are the learner's only observation, and $\hr=\hr(\text{past})$ is the same measurable map under both worlds. By the chain rule for the interaction and independence of the $k$ draws within a round,
\[
  \KL\!\bigl(\Prob_+^{N}\,\big\|\,\Prob_-^{N}\bigr)
  =\sum_{r=1}^{N}\E_+\Bigl[\, k\,\KL\!\bigl(\Ber(u_+(\hr))\,\big\|\,\Ber(u_-(\hr))\bigr)\Bigr].
\]
It remains to bound the per-round divergence. Write $d(\eta):=u_+(\eta)-u_-(\eta)$. By \eqref{eq:jac} and Lemma~\ref{lem:realize}, $d(\eta)=2\Delta\psi(\eta)/J+\rho(\eta)$ where the remainder $\rho$ collects the $O(\Delta^2)$ linearization term of \eqref{eq:jac} and the Lemma~\ref{lem:realize} error, both $O(\Delta/k)$; choosing $\Delta\le\Delta_\star$ makes $\|\rho\|_\infty\le\tfrac12\cdot 2\Delta/J$, so with $\|\psi\|_\infty\le1$ and $J\ge kc_0$ (Lemma~\ref{lem:scale}),
\[
  |d(\eta)|\ \le\ \frac{3\Delta}{J}\ \le\ \frac{3\Delta}{k c_0},
  \qquad \eta\in[0,\etamax].
\]
By Lemma~\ref{lem:scale} and $\|g_\varepsilon\|_\infty,\ \| \text{error}\|_\infty=O(\Delta/k)$, for $\Delta\le\Delta_\star$ we also have $u_-(\eta)\in[c_0/(2k),\,1-c_0/2]$, hence $u_-(1-u_-)\ge c_0^2/(4k)$. Using $\KL(\Ber(p)\|\Ber(q))\le(p-q)^2/(q(1-q))$,
\[
  k\,\KL\!\bigl(\Ber(u_+(\eta))\,\big\|\,\Ber(u_-(\eta))\bigr)
  \ \le\ \frac{k\,d(\eta)^2}{u_-(1-u_-)}
  \ \le\ \frac{k\cdot 9\Delta^2/(k^2c_0^2)}{c_0^2/(4k)}
  \ =\ \frac{36}{c_0^{4}}\,\Delta^2
  \ =:\ C_0\,\Delta^2 ,
\]
uniformly in $\eta$ (hence in $\hr$), with $C_0=36/c_0^{4}$ depending only on $c_0$ --- the budget $k$ cancels, precisely because $u_0=\Theta(1/k)$ (Lemma~\ref{lem:scale}) shrinks the per-round information at the same rate the per-sample gap does. Therefore
\begin{equation}\label{eq:kl}
  \KL\!\bigl(\Prob_+^{N}\,\big\|\,\Prob_-^{N}\bigr)\ \le\ C_0\,N\,\Delta^2.
\end{equation}

\emph{Remark (bandit sufficiency).} The pass bit $\Pi_r=1-\prod_{j=1}^k(1-U^{(r)}_j)\sim\Ber(f_\varepsilon(\hr))$ is a deterministic function of the round's verdicts, so by data processing $\KL(\Ber(f_+)\|\Ber(f_-))\le k\,\KL(\Ber(u_+)\|\Ber(u_-))$; conversely both sides are $\Theta(\Delta^2)$ because the flat base forces $u_0=\Theta(1/k)$ (successes are rare, so $\Pi_r$ is a sufficient statistic up to constants). Thus \eqref{eq:kl} is unchanged up to the constant $C_0$ whether the learner reads all $k$ verdicts or only the round's pass bit; we bound the former, which is the richer of the two feedback models.

\emph{Bretagnolle--Huber.} Let $A:=\{N_-+\tfrac12 N_\circ\ge N/4\}$; since the two counts sum to $N$, on $A^c$ we have $N_++\tfrac12 N_\circ>3N/4$. By the Bretagnolle--Huber inequality (see \cite{bretagnolle1979estimation}) and \eqref{eq:kl},
\[
  \Prob_+(A)+\Prob_-(A^c)
  \ge\tfrac12\exp\!\bigl(-\KL(\Prob_+^{N}\|\Prob_-^{N})\bigr)
  \ge\tfrac12\exp\!\bigl(-C_0 N\Delta^2\bigr).
\]
Choose $\Delta=(2C_0 N)^{-1/2}\wedge(c_0/2)$. For $N\ge N_0$, where $N_0=N_0(c_0,k,M)$ is large enough that both $\Delta=(2C_0N)^{-1/2}$ (so the minimum is the first term) and $\Delta\le\Delta_\star$ hold, we get $C_0 N\Delta^2=\tfrac12$ and
\begin{equation}\label{eq:bh}
  \Prob_+(A)+\Prob_-(A^c)\ \ge\ \tfrac12 e^{-1/2}\ \ge\ 0.30.
\end{equation}
By \eqref{eq:occ}, on $A$ the $\mathcal I_+$-regret is $\ge\Delta N/4$ and on $A^c$ the $\mathcal I_-$-regret is $\ge\Delta N/4$. Passing to the maximum via the average and using \eqref{eq:bh},
\[
  \max\{\E_+[\Reg],\E_-[\Reg]\}
  \ge\tfrac12\bigl(\E_+[\Reg]+\E_-[\Reg]\bigr)
  \ge\frac{\Delta N}{8}\bigl(\Prob_+(A)+\Prob_-(A^c)\bigr)
  \ge\frac{0.30\,\Delta N}{8}.
\]
Substituting $\Delta=(2C_0 N)^{-1/2}$, so that $\Delta N=\sqrt{N/(2C_0)}$,
\[
  \sup_{\theta\in\Theta}\E_\theta\bigl[\Reg_{\mathcal A}(N)\bigr]
  \ \ge\ \max\{\E_+[\Reg],\E_-[\Reg]\}
  \ \ge\ \frac{0.30}{8\sqrt{2C_0}}\,\sqrt{N}\ =:\ c\,\sqrt{N}
\]
for all $N\ge N_0$, with $c=0.30/(8\sqrt{2C_0})>0$ depending only on $c_0$. This is Theorem~\ref{thm:lower}. \hfill$\blacksquare$

\begin{remark}[Interpretation]
Dividing the cumulative lower bound by $N$ gives a worst-case average
pass@$k$ discrepancy of order $\Omega(N^{-1/2})$. This statistical barrier is
separate from the deterministic prefix-tracking discrepancy controlled by
Theorem~\ref{thm:passk-tracking}: even if price tracking were exact, finite
terminal verifier feedback still prevents uniformly faster adaptation over
the hard family.
\end{remark}

\section{Proof of Theorems~\ref{thm:collapse} and \ref{thm:partial}}\label{app:partial}
Throughout we use the \emph{sequence-level} calibration of Definition~\ref{def:calib},
\begin{equation}\label{eq:calib-recall}
U(w^n)=s\,\bar\ell(w^n)+\bar c+\bar r(w^n),\qquad
\bar\ell(w^n)=\tfrac1n\log p(w^n),\qquad \sup_{w^n}|\bar r(w^n)|\le\varepsilon,
\end{equation}
and write $\CE(q):=-\tfrac1n\E_q[\log p(w^n)]$, $\bar H(q):=\tfrac1n\sum_t\E_{w^{t-1}\sim q}\,H\bigl(q_t(\cdot\mid w^{t-1})\bigr)$. Every distribution on $\W^n$ factorizes autoregressively, so~\eqref{eq:PE} is an optimization over $\Delta(\W^n)$; we use this freely. Since $p_t\ge\rho>0$, $p(w^n)\ge\rho^n>0$ on all of $\W^n$.

\begin{lemma}[Sequence-level form of the budgets]\label{lem:app-seqbudgets}
For every policy $q$:
\textup{(a)} $-\tfrac1n\sum_t\E_{w^{t-1}\sim q}\bigl[\sum_w q_t\log p_t\bigr]=\CE(q)$;
\textup{(b)} $\bar H(q)=\tfrac1n H(q)$;
\textup{(c)} under~\eqref{eq:calib-recall}, $u_q=\bar c-s\,\CE(q)+\E_q[\bar r]$.
\end{lemma}
\begin{proof}
(a) $\E_q[\log p(w^n)]=\sum_t\E_q[\log p_t(w_t\mid w^{t-1})]=\sum_t\E_{w^{t-1}\sim q}\bigl[\sum_w q_t(w\mid w^{t-1})\log p_t(w\mid w^{t-1})\bigr]$ by the tower property. (b) Chain rule of entropy. (c) Take $\E_q$ in~\eqref{eq:calib-recall} and use $\E_q[s\bar\ell]=\tfrac sn\E_q\log p=-s\,\CE(q)$.
\end{proof}

\begin{lemma}[Tempered family]\label{lem:app-tempered}
For $a\ge0$ let $q_a(w^n)\propto p(w^n)^a$ and $A(a):=\log\sum_{w^n}p(w^n)^a$. Then:
\begin{enumerate}[label=\textup{(\roman*)},leftmargin=2.2em]
\item $H(q_a)=A(a)-aA'(a)$ and $\tfrac{d}{da}H(q_a)=-a\,\Var_{q_a}(\log p)$.
\item If $p(w^n)$ has a unique mode, then $H(q_0)=n\log|\W|$, $H(q_a)\to0$ as $a\to\infty$, and $a\mapsto H(q_a)$ is continuous and strictly decreasing on $(0,\infty)$. Hence for every $\beta\in(0,\log|\W|)$ there is a unique $a_\beta>0$ with $\bar H(q_{a_\beta})=\beta$, and $\beta\mapsto a_\beta$ is continuous and strictly decreasing.
\item \textup{(Entropy-constrained CE minimality.)} For every $q\in\Delta(\W^n)$ with $\bar H(q)\ge\beta$, $\CE(q)\ge\CE(q_{a_\beta})$, with equality iff $q=q_{a_\beta}$.
\end{enumerate}
\end{lemma}
\begin{proof}
(i) $\log q_a=a\log p-A(a)$ and $A'(a)=\E_{q_a}[\log p]$, so $H(q_a)=-\E_{q_a}[a\log p-A(a)]=A-aA'$; differentiating, $\tfrac{d}{da}H(q_a)=-aA''(a)=-a\Var_{q_a}(\log p)$.
(ii) $q_0$ is uniform on $\W^n$ and $q_a$ concentrates on the unique mode as $a\to\infty$. A unique mode with $|\W^n|\ge2$ makes $\log p$ non-constant, and $q_a$ has full support, so $\Var_{q_a}(\log p)>0$ and (i) gives strict decrease. The claims on $a_\beta$ follow from continuity and strict monotonicity, using (b) of Lemma~\ref{lem:app-seqbudgets}.
(iii) For any $q$, $0\le\mathrm{KL}(q\,\|\,q_{a_\beta})=-H(q)-a_\beta\E_q[\log p]+A(a_\beta)$. Hence, using $H(q)\ge n\beta=H(q_{a_\beta})$ and (i),
\[
a_\beta\,n\,\CE(q)=-a_\beta\E_q[\log p]\ \ge\ H(q)-A(a_\beta)\ \ge\ H(q_{a_\beta})-A(a_\beta)=-a_\beta\E_{q_{a_\beta}}[\log p]=a_\beta\,n\,\CE(q_{a_\beta}).
\]
Divide by $a_\beta n>0$. Equality forces $\mathrm{KL}(q\|q_{a_\beta})=0$.
\end{proof}
Part (iii) uses no calibration: it is the maximum-entropy characterization of tempering, read as a cross-entropy minimization.

\begin{proof}[Proof of Theorem~\ref{thm:collapse}]
With $\varepsilon=0$, Lemma~\ref{lem:app-seqbudgets}(c) gives $u_q=\bar c-s\,\CE(q)$. Since $\passk(q)=1-(1-u_q)^k$ is strictly increasing in $u_q$ and $s>0$, for every $k$ program~\eqref{eq:PE} is equivalent to minimizing $\CE(q)$ over $\mathcal F:=\{q:\bar H(q)\ge\beta,\ \CE(q)\le\gamma\}$. By Lemma~\ref{lem:app-tempered}(iii), $\inf\{\CE(q):\bar H(q)\ge\beta\}=\CE(q_{a_\beta})$, attained only at $q_{a_\beta}$. If $\gamma<\CE(q_{a_\beta})$, then $\mathcal F=\emptyset$. If $\gamma\ge\CE(q_{a_\beta})$, then $q_{a_\beta}\in\mathcal F\subseteq\{\bar H\ge\beta\}$, so it is the unique minimizer over $\mathcal F$. This proves feasibility iff $\gamma\ge\CE(q_{a_\beta})$; $\gamma$ enters only through feasibility.

\emph{Relation to Theorem~\ref{thm:tilt}.} With $\lambda^\star=0$ and $\mu^\star=na_\beta/s$, the sequence law of Corollary~\ref{cor:seqtilt} in $(\lambda^\star,\mu^\star)$ form is $p^{\lambda^\star}e^{\mu^\star U}\propto e^{\mu^\star\frac sn\log p}=p^{a_\beta}$. This is the boundary case of Theorem~\ref{thm:tilt} in which the ceiling is slack (its multiplier vanishes) and the entropy multiplier is $\mu_E=n/\mu^\star=s/a_\beta>0$. The $(T,\eta)$ parametrization degenerates there, and the exponent is carried by $\mu^\star$ alone.

\emph{Conditionals.} Write $p(w^n)=p(w^t)\prod_{s>t}p_s(w_s\mid w^{s-1})$ and marginalize the future:
\begin{equation}\label{eq:pertoken-app}
q^\star_t(w\mid w^{t-1})=\frac{\sum_{w_{t+1:n}}p(w^n)^{a_\beta}}{\sum_{w,w_{t+1:n}}p(w^n)^{a_\beta}}\ \propto\ p_t(w\mid w^{t-1})^{a_\beta}\,\Psi^{(a_\beta)}_t(w^{t-1}w),
\end{equation}
the prefix factor $p(w^{t-1})^{a_\beta}$ cancelling, with $\Psi^{(a)}_t$ as in~\eqref{eq:pertoken} and $\Psi^{(a)}_n\equiv1$. Under prefix-independent conditionals, $\Psi^{(a)}_t=\prod_{s>t}\sum_{w_s}p_s(w_s)^a$ does not depend on the prefix or on $w$, leaves the normalizer, and~\eqref{eq:pertoken-app} reduces to~\eqref{eq:tempopt}.

At $s=0$, $u_q\equiv\bar c$ and every feasible $q$ is optimal. Finally, $\beta\mapsto a_\beta$ is strictly decreasing by Lemma~\ref{lem:app-tempered}(ii), which is the monotonicity used after Theorem~\ref{thm:collapse}.
\end{proof}

\begin{remark}[Why sequence-level scoring]\label{rem:app-tokenvsseq}
Lemma~\ref{lem:app-seqbudgets}(c) is exact because the calibration score $\bar\ell$ is the same functional as the plausibility ceiling, by Lemma~\ref{lem:app-seqbudgets}(a). A token-level hypothesis $V_t(w)=s\log p_t(w)+c_t$ does not yield it. Under~\eqref{eq:calib-recall} with $\varepsilon=0$, the potential of Lemma~1 at the prices above is
$V_t(w\mid w^{t-1})=\bar c+\tfrac sn\log p(w^{t-1})+\tfrac sn\log p_t(w\mid w^{t-1})+\tfrac1{\mu^\star}\log\Psi^{(a_\beta)}_t(w^{t-1}w)$:
affine in $\log p_t$ with slope $s/n$, plus the remaining-horizon term. Sequence-level scoring is also exactly what~\eqref{eq:ols} estimates.
\end{remark}
\begin{lemma}[Sharp mean-difference inequality]\label{lem:meandiff}
For probability measures $P,Q$ on a common space and any bounded $g$,
$\bigl|\E_P[g]-\E_Q[g]\bigr|\le\bigl(\sup g-\inf g\bigr)\,\TV(P,Q)$.
\end{lemma}
\begin{proof}
For any constant $c$, $\E_P[g]-\E_Q[g]=\sum_x(P-Q)(x)(g(x)-c)$. Take $c=\tfrac12(\sup g+\inf g)$, so $|g-c|\le\tfrac12(\sup g-\inf g)$, and use $\sum_x|P(x)-Q(x)|=2\TV(P,Q)$.
\end{proof}

\begin{proof}[Proof of Theorem~\ref{thm:partial}]
Let $\tilde q:=q_{a_\beta}$, which is defined for $s\ge0$ since it depends only on $(p,\beta)$.

\textbf{(i)} By Lemma~\ref{lem:app-tempered}(ii), $\bar H(\tilde q)=\beta$, and $\CE(\tilde q)\le\gamma$ by hypothesis, so $\tilde q$ is feasible and $u_{q^\star}\ge u_{\tilde q}$. The optimum $q^\star$ is feasible, so $\bar H(q^\star)\ge\beta$, and Lemma~\ref{lem:app-tempered}(iii) gives $\CE(q^\star)\ge\CE(\tilde q)$. By Lemma~\ref{lem:app-seqbudgets}(c) and $s\ge0$,
\[
u_{q^\star}-u_{\tilde q}=s\bigl[\CE(\tilde q)-\CE(q^\star)\bigr]+\E_{q^\star}[\bar r]-\E_{\tilde q}[\bar r]\ \le\ \E_{q^\star}[\bar r]-\E_{\tilde q}[\bar r]\ \le\ 2\varepsilon.
\]
Lemma~\ref{lem:meandiff} sharpens the last step to $2\varepsilon\,\TV(q^\star,\tilde q)$.

\textbf{(ii)} The map $u\mapsto1-(1-u)^k$ has derivative $k(1-u)^{k-1}$, positive and non-increasing in $u$. By the mean value theorem there is $\xi\in[u_{\tilde q},u_{q^\star}]$ with $\passk(q^\star)-\passk(\tilde q)=k(1-\xi)^{k-1}(u_{q^\star}-u_{\tilde q})$. Since $\xi\ge u_{\tilde q}\ge u_{q^\star}-2\varepsilon$ by (i), $k(1-\xi)^{k-1}\le\sigma_k^{+}$. Combined with (i), the gap lies in $[0,2\varepsilon\sigma_k^{+}]$.

\textbf{(iii)} For $v\in(0,1)$ let $\sigma_v(k):=k(1-v)^{k-1}$, $k>0$. Then $\tfrac{d}{dk}\log\sigma_v(k)=\tfrac1k+\log(1-v)$, which is positive for $k<k_v:=-1/\log(1-v)$ and negative after. So $\sigma_v$ is strictly unimodal with maximizer $k_v$, and $\sigma_v(k)\to0$ as $k\to0^+$ and as $k\to\infty$. With $v=u_{q^\star}-2\varepsilon\in(0,1)$ we get $\sigma^{+}_k=\sigma_v(k)$, peaking at $k_v=1/v+O(1)$. Hence for every $\delta>0$ the set $\{k:2\varepsilon\sigma_k^{+}>\delta\}$ is a bounded (possibly empty) interval, and outside it $\tilde q$ is $\delta$-optimal by (ii).
\end{proof}

\section{Proof of Theorem~\ref{thm:ctd}: suboptimality and certification of CTD}\label{app:ctd}
We keep the notation of Appendix~\ref{app:partial}. The deployed policy is $\hat q_{\mathrm{CTD}}=q^{\mathrm{temp}}_{\hat a}$, per-token temperature at the Phase-2 exponent. The regression of Phase~1 does \emph{not} enter the policy: $\hat a$ depends only on $(p,\beta)$. The proof therefore has a deterministic part, the $\passk$ bound, and a stochastic part, the probability that Phase~1 certifies.

\paragraph{Part A: the deterministic bound.}
The proof uses only the following property of Phase~2:
\begin{equation}\label{eq:app-floor}
\bar H(\hat q_{\mathrm{CTD}})\ \ge\ \beta,\qquad \CE(\hat q_{\mathrm{CTD}})\le\gamma,
\end{equation}
which holds with equality in the entropy when $\hat a$ is the exact root. Then:
\begin{enumerate}[label=\textup{(\alph*)},leftmargin=2.2em]
\item $\hat q_{\mathrm{CTD}}$ is feasible, so $u_{q^\star}\ge u_{\hat q_{\mathrm{CTD}}}$.
\item By Lemma~\ref{lem:app-tempered}(iii) applied to $\hat q_{\mathrm{CTD}}$, $\Delta_{\CE}=\CE(\hat q_{\mathrm{CTD}})-\CE(q_{a_\beta})\ge0$, with equality iff $\hat q_{\mathrm{CTD}}=q_{a_\beta}$.
\item Applied to $q^\star$, which is feasible, the same lemma gives $\CE(q^\star)\ge\CE(q_{a_\beta})$.
\end{enumerate}
By Lemma~\ref{lem:app-seqbudgets}(c) and $s\ge0$,
\[
u_{q^\star}-u_{\hat q_{\mathrm{CTD}}}
= s\bigl[\CE(\hat q_{\mathrm{CTD}})-\CE(q^\star)\bigr]+\E_{q^\star}[\bar r]-\E_{\hat q_{\mathrm{CTD}}}[\bar r]
\ \le\ s\bigl[\CE(\hat q_{\mathrm{CTD}})-\CE(q_{a_\beta})\bigr]+2\varepsilon
= s\Delta_{\CE}+2\varepsilon .
\]
As in Theorem~\ref{thm:partial}(ii), the mean value theorem with $\xi\ge u_{q^\star}-(s\Delta_{\CE}+2\varepsilon)$ gives
$0\le\passk(q^\star)-\passk(\hat q_{\mathrm{CTD}})\le\sigma_k^{++}(s\Delta_{\CE}+2\varepsilon)$, which is~\eqref{eq:ctdbound}.

\emph{When $\Delta_{\CE}=0$.} Under prefix-independent conditionals, $\Psi^{(a)}_t$ is token-free for every $a$, so by~\eqref{eq:pertoken-app} $q^{\mathrm{temp}}_a=q_a$ for all $a$. Then $\bar H(\hat q_{\mathrm{CTD}})=\beta$ forces $\hat a=a_\beta$ by Lemma~\ref{lem:app-tempered}(ii), and $\Delta_{\CE}=0$. In general, $\Delta_{\CE}$ is the cross-entropy price, at fixed entropy, of dropping the remaining-horizon factor $\Psi^{(a)}_t$.

\begin{remark}[Monte-Carlo root in Phase~2]
Part~A needs only~\eqref{eq:app-floor}, not an exact root. If Phase~2 estimates $\bar H(q^{\mathrm{temp}}_a)$ to accuracy $\tau$ and bisects on the conservative target $\beta+\tau$, then~\eqref{eq:app-floor} holds on the accuracy event, and the extra slack enters the bound only through a larger $\Delta_{\CE}$.
\end{remark}

\paragraph{Part B: certification probability.}
Write $s_{\mathrm{spr}}^2=\Var(\bar\ell)\ge s_0^2$ and $s_{\mathrm{pop}}:=\Cov(\bar\ell,U)/\Var(\bar\ell)$.

\emph{Population slope} (Proposition~\ref{prop:ident}). By~\eqref{eq:calib-recall}, $\Cov(\bar\ell,U)=s\Var(\bar\ell)+\Cov(\bar\ell,\bar r)$. Since $|\bar r|\le\varepsilon$, Cauchy--Schwarz gives $|\Cov(\bar\ell,\bar r)|\le s_{\mathrm{spr}}\varepsilon$. Hence
\begin{equation}\label{eq:app-spop}
|s_{\mathrm{pop}}-s|\le\varepsilon/s_0=:C_0\varepsilon,\qquad |s_{\mathrm{pop}}|\le\tfrac{\sqrt{\Var U}}{s_{\mathrm{spr}}}\le\tfrac{1}{2s_0}.
\end{equation}

\emph{Concentration.} The empirical slope~\eqref{eq:ols} is $\hat s=\widehat\Cov/\widehat\Var$, a function of four empirical means: those of $\bar\ell\in[-B,0]$, $U\in\{0,1\}$, $\bar\ell U\in[-B,0]$, and $\bar\ell^2\in[0,B^2]$. By Hoeffding's inequality and a union bound, with probability at least $1-\delta/2$,
\[
|\widehat\Cov-\Cov|\vee|\widehat\Var-\Var|\le\eta_M:=c_1\bar B^2\sqrt{\tfrac{\log(8/\delta)}{N}},\qquad \bar B:=\max\{1,B\},
\]
with $c_1$ absolute. On this event, if $\eta_M\le s_0^2/2$, then $\widehat\Var\ge s_0^2/2$, and
\[
|\hat s-s_{\mathrm{pop}}|=\Bigl|\tfrac{\widehat\Cov-\Cov}{\widehat\Var}+s_{\mathrm{pop}}\tfrac{\Var-\widehat\Var}{\widehat\Var}\Bigr|
\le\frac{2\eta_M(1+|s_{\mathrm{pop}}|)}{s_0^2}\le\frac{C_s}{2}\sqrt{\tfrac{\log(8/\delta)}{N}},
\qquad C_s:=\frac{4c_1\bar B^2}{s_0^2}\Bigl(1+\frac1{2s_0}\Bigr).
\]

\emph{Sign.} Suppose $s>C_0\varepsilon$ and choose $\delta:=8\exp\bigl(-N(s-C_0\varepsilon)^2/C_s^2\bigr)$, so that $\tfrac{C_s}{2}\sqrt{\log(8/\delta)/N}=\tfrac12(s-C_0\varepsilon)$.
The side condition holds automatically: by~\eqref{eq:app-spop},
\[
\eta_M=c_1\bar B^2(s-C_0\varepsilon)/C_s\le c_1\bar B^2|s_{\mathrm{pop}}|/C_s<s_0^2/8.
\]
On the event, $\hat s\ge s_{\mathrm{pop}}-\tfrac12(s-C_0\varepsilon)\ge\tfrac12(s-C_0\varepsilon)>0$, using $s_{\mathrm{pop}}\ge s-C_0\varepsilon$. Hence
\[
\Pr(\hat s\le0)\ \le\ \delta/2\ \le\ 8\exp\!\bigl(-N(s-C_0\varepsilon)^2/C_s^2\bigr),
\]
with $C_0=1/s_0$ and $C_s$ depending only on $(B,s_0)$. This is the second claim of Theorem~\ref{thm:ctd}. \hfill$\blacksquare$

\begin{remark}[What the estimate does and does not control]
Unlike an analysis in which $\hat s$ sets the exponent, the policy here does not depend on $\hat s$. As a result there is no $O(1/\sqrt{N})$ term in~\eqref{eq:ctdbound} and no horizon factor to cancel. The estimate $(\hat s,\hat\varepsilon)$ serves two purposes: it decides whether the calibrated regime applies, and $\hat\varepsilon$ provides an empirical residual diagnostic for the deployment gate.
\end{remark}

\newif\ifPASSKstandalone
\ifdefined\qst\PASSKstandalonefalse\else\PASSKstandalonetrue\fi

\ifPASSKstandalone
  \documentclass[10pt]{article}
  \usepackage[letterpaper,margin=1in]{geometry}
  \usepackage{amsmath,amssymb,amsthm,bm}
  \usepackage[expansion=false]{microtype}
  \usepackage[colorlinks=true,allcolors=blue]{hyperref}
  \theoremstyle{plain}
  \newtheorem{theorem}{Theorem}
  \newtheorem{proposition}[theorem]{Proposition}
  \newtheorem{lemma}[theorem]{Lemma}
  \newtheorem{corollary}[theorem]{Corollary}
  \theoremstyle{remark}
  \newtheorem{remark}[theorem]{Remark}
  \setcounter{section}{9} % render as Appendix J when read alone
  \renewcommand{\thesection}{\Alph{section}}
  \begin{document}
\fi
\providecommand{\pass}{\mathrm{pass}@k}
\providecommand{\qst}{q^{\star}}
\providecommand{\lamst}{\lambda^{\star}}
\providecommand{\must}{\mu^{\star}}
\providecommand{\W}{\mathcal{W}}
\providecommand{\E}{\mathbb{E}}
\providecommand{\Var}{\mathrm{Var}}
\providecommand{\CE}{\mathrm{CE}}
\providecommand{\logit}{\mathrm{logit}}
\providecommand{\texorpdfstring}[2]{#1}   % harmless if hyperref is loaded

\def\sig{\sigma}
\def\one{\mathbb{1}}

\section{The optimal value \texorpdfstring{$\pass(\qst)$}{pass@k(q*)}}\label{app:passk}

Theorem~1 fixes the \emph{form} of the optimal decoder; here we compute the
\emph{value} it attains,
$\pass(\qst)=1-(1-u_{\qst})^k$ with
$u_{\qst}=\E_{w^n\sim\qst}[U(w^n)]$.
This provides an exact yardstick against which both estimators are scored,
without requiring Monte Carlo estimation of the optimum. Two facts organize the
computation. The potential $V$ of Lemma~1 depends only on $(p,U,\lamst,\must)$ and on no
policy, so nothing here requires a fixed point; and since the tilt carries $V$ while the value
averages $U$ --- objects that differ by a free energy (Corollary~\ref{cor:freeenergy}) --- $u_{\qst}$ is
\emph{not} a single softmax recursion but the second of two decoupled backward sweeps. At the
root the sequence-level tilt (Corollary~\ref{cor:seqtilt}), $\qst(w^n)\propto p(w^n)^{\lamst}e^{\must U(w^n)}$,
makes $(U,\log p)$ an exponential family, which pins $u_{\qst}$ in closed form. We then
specialize twice: to a binary verifier, where the value is a logistic in the tilt, and to the
count-threshold instances, where it is an $O(nM)$ dynamic program.

\paragraph{Two backward sweeps.}
Index the potential by the completed prefix, $V_t(w^{t}):=V_t(w_t\mid w^{t-1})$, so
$V_n(w^n)=U(w^n)$ and $V_0(\varnothing):=\tfrac1{\must}\log Z_1$. Lemma~1 reads, for
$0\le t\le n-1$,
\begin{equation}\label{eq:sweep1}
V_t(w^{t})=\tfrac1{\must}\log Z_{t+1}(w^{t}),\qquad
Z_{t+1}(w^{t})=\sum_{w'\in\W}p_{t+1}(w'\mid w^{t})^{\lamst}e^{\must V_{t+1}(w^{t}w')}=e^{\must V_t(w^{t})},
\end{equation}
and the policy is
\begin{equation}\label{eq:policy}
\qst_{t+1}(w'\mid w^{t})=p_{t+1}(w'\mid w^{t})^{\lamst}\,e^{\must(V_{t+1}(w^{t}w')-V_t(w^{t}))}.
\end{equation}
Equation~\eqref{eq:policy} is worth reading slowly: the optimal decoder is the tempered base
model reweighted by the \emph{increment} of the potential, the term $\must V_t(w^t)$ acting as
the log-normalizer at the node. Both objects in \eqref{eq:sweep1} are computable from the
instance alone.

\begin{proposition}[Exact value: two decoupled sweeps]\label{prop:value}
Define the pass cost-to-go $u_{n+1}(w^n)=U(w^n)$,
$u_t(w^{t-1})=\sum_{w}\qst_t(w\mid w^{t-1})\,u_{t+1}(w^{t-1}w)$. Then
$\pass(\qst)=1-(1-u_1(\varnothing))^{k}$. Sweep \eqref{eq:sweep1} does not involve $u$ and
the value backup does not feed back into $V$; each is a finite acyclic pass, so the pair is
exact in two sweeps with no iteration.
\end{proposition}
\begin{proof}
The backup is the tower property for the fixed sampler $\qst$; by Lemma~1, $V$ (hence
$\qst$) is a deterministic function of $(p,U,\lamst,\must)$, so \eqref{eq:sweep1} closes
without $u$, and the $k$ draws are i.i.d.\ under $\qst$.
\end{proof}

\noindent The two sweeps must be kept apart. Folding them into the single recursion
$u_t\propto\sum_w p_t^{\lamst}e^{\must u_{t+1}}u_{t+1}$ is invalid once $V^\star\ne J$: by
Corollary~\ref{cor:freeenergy}, $V_{t-1}=\E_{\qst_t}[V_t]+\tfrac1{\must}[H(\qst_t)-\lamst\CE(\qst_t,p_t)]$ exceeds
its own backup by a free energy, so $V$ is not the pass probability. The two objects are
distinct in general; at the terminal step the free-energy continuation reduces directly to $U$,
while degenerate-price limits should be interpreted through the product $\must V_t$ as in
Appendix~\ref{app:lem1}. The first sweep is a log-sum-exp over children, the second a plain
average; only the first sits in an exponent.

\paragraph{Free-energy form.}
Let $\mathcal A(\lambda,\mu):=\log\sum_{w^n}p(w^n)^{\lambda}e^{\mu U(w^n)}$ be the cumulant
generating function of the pair $(\log p,U)$ over $\W^n$, so
$\mathcal A(\lamst,\must)=\log Z_1=\must V_0(\varnothing)$ by \eqref{eq:sweep1} at $t=0$: the
root potential \emph{is} the sequence free energy, obtained by one sweep with no enumeration
of $\W^n$.
\begin{lemma}[Dual gradients]\label{lem:dual}
For $q^{\lambda,\mu}\propto p^{\lambda}e^{\mu U}$,
$\partial_\mu\mathcal A=u_{q^{\lambda,\mu}}$, $\partial_\lambda\mathcal A=\E[\log p]=-\overline{\CE}$,
$\partial^2_\mu\mathcal A=\Var(U)\ge0$, $\partial^2_\lambda\mathcal A=\Var(\log p)\ge0$; in
particular $u_{\qst}=\partial_\mu\mathcal A|_{(\lamst,\must)}$.
\end{lemma}
\begin{proof}
$q^{\lambda,\mu}=\exp(\lambda\log p+\mu U-\mathcal A)$; differentiating $\sum q^{\lambda,\mu}=1$
once gives the means, twice the variances, and $\log p(w^n)=\sum_t\log p_t$ gives $-\overline{\CE}$.
\end{proof}
\noindent So the value is a dual gradient of the same free energy the decoder normalizes by;
the free energy is convex and the value is non-decreasing in the tilt. The next identity converts this into a
quantity the root-solve already produces.
\begin{proposition}[Value $=$ root potential minus budget]\label{prop:budget}
Here $\bar H:=\sum_{t=1}^n \E_{\qst}[H(\qst_t)]$ and
$\overline{\CE}:=-\sum_{t=1}^n \E_{\qst}\sum_w \qst_t(w\mid w^{t-1})
\log p_t(w\mid w^{t-1})$ denote the unnormalized horizon totals.
With $\tfrac1n\bar H=\beta$ and $\tfrac1n\overline{\CE}=\gamma$ the active budgets of (PE),
\begin{equation}\label{eq:budget}
u_{\qst}=V_0(\varnothing)-\frac{n}{\must}\big(\beta-\lamst\gamma\big),\qquad
\pass(\qst)=1-(1-u_{\qst})^{k}.
\end{equation}
\end{proposition}
\begin{proof}
Averaging $\log\eqref{eq:policy}$ under $\qst_t(\cdot\mid w^{t-1})$ gives Corollary~\ref{cor:freeenergy}; taking
the $\qst$-occupancy expectation and summing over $t$ telescopes to
$V_0(\varnothing)=u_{\qst}+\tfrac1{\must}[\bar H-\lamst\overline{\CE}]$, using
$V_n=U$. Substitute $\bar H=n\beta$, $\overline{\CE}=n\gamma$.
\end{proof}
\noindent Every term on the right comes from the same root-solve that returns $(T,\eta)$, so
comparing \eqref{eq:budget} with the forward value sweep of Proposition~\ref{prop:value} is a
free implementation check: the two must agree to machine precision.

\paragraph{Binary verifier: a logistic in the tilt.}
When $U\in\{0,1\}$ the sequence space splits into a passing and a failing part, and only the
tempered mass of each survives. For $\lambda>0$ put
\begin{equation}\label{eq:logodds}
P_\lambda:=\!\!\sum_{w^n:\,U(w^n)=1}\!\! p(w^n)^{\lambda},\qquad
Q_\lambda:=\!\!\sum_{w^n:\,U(w^n)=0}\!\! p(w^n)^{\lambda},\qquad
\ell_\lambda:=\log\frac{P_\lambda}{Q_\lambda}.
\end{equation}
The scalar $\ell_\lambda$ is the \emph{tempered log-odds of passing}: writing
$\pi_\lambda(w^n)\propto p(w^n)^{\lambda}$ for pure sequence-level tempering at $T=1/\lambda$
(no verifier tilt), one has $\Pr_{\pi_\lambda}[U=1]=P_\lambda/(P_\lambda+Q_\lambda)$ and hence
$\ell_\lambda=\logit\Pr_{\pi_\lambda}[U=1]$. At $\lambda=1$ it is the base model's own log-odds
of solving the task; $\ell_\lambda$ is finite exactly when both classes carry mass, which is the
only non-degeneracy needed below. Two cautions: $\ell_\lambda$ is a single number attached to
the whole instance and is sequence-level (it grows like $n$ through $\log p$), so it is not a
per-token quantity; and despite the shared letter it is unrelated to the per-sequence score
$\bar\ell(w^n)=\tfrac1n\log p(w^n)$ of Definition~1, which reappears in
Corollary~\ref{cor:ctdlink}.
\begin{theorem}[Exact value, binary verifier]\label{thm:logistic}
Let $U\in\{0,1\}$ and $\sig(x)=1/(1+e^{-x})$. Then $\mathcal A(\lambda,\mu)=\log(Q_\lambda+e^{\mu}P_\lambda)$,
$u_{q^{\lambda,\mu}}=\sig(\mu+\ell_\lambda)$, and
\begin{equation}\label{eq:logistic}
u_{\qst}=\sig(\must+\ell_{\lamst}),\qquad
\pass(\qst)=1-\big(1-\sig(\must+\ell_{\lamst})\big)^{k},
\end{equation}
with exact sensitivities $\partial_\mu u=u(1-u)>0$ and
$\partial_\lambda u=u(1-u)\big(\E_{\pi^{+}_\lambda}[\log p]-\E_{\pi^{-}_\lambda}[\log p]\big)$,
where $\pi^{\pm}_\lambda$ are $p^{\lambda}$ restricted to the passing / failing sequences and
normalized.
\end{theorem}
\begin{proof}
Split $\mathcal A$ by the value of $U$; then $u=\partial_\mu\mathcal A=e^{\mu}P_\lambda/(Q_\lambda+e^{\mu}P_\lambda)=\sig(\mu+\ell_\lambda)$,
$\partial_\mu u=\sig'=u(1-u)$, and $\partial_\lambda u$ follows by chaining through $\ell_\lambda$
with $\partial_\lambda\log P_\lambda=\E_{\pi^{+}_\lambda}[\log p]$.
\end{proof}
\noindent Read \eqref{eq:logistic} as an additive decomposition on the logit scale: temperature
sets the prior log-odds $\ell_{\lamst}$ and the verifier tilt shifts it by exactly $\must$. The
two prices therefore act on the same axis but through different channels --- $\must$ enters with
unit slope, while $\lamst$ enters only through how tempering redistributes mass between the
passing and failing classes, which is the difference of means in $\partial_\lambda u$ and can
vanish.
\begin{corollary}[The $\lambda$-sensitivity is CTD's regression slope]\label{cor:ctdlink}
With $\bar\ell=\tfrac1n\log p$, $\partial_\lambda u=n\,u(1-u)\big(\E_{\pi^{+}_\lambda}[\bar\ell]-\E_{\pi^{-}_\lambda}[\bar\ell]\big)$;
at $\lambda=1$ this difference equals $s\Var(\bar\ell)/(\bar u(1-\bar u))+O(\varepsilon)$ under
$(s,\varepsilon)$-calibration (point-biserial), the population form of CTD's OLS slope $\hat s$.
A near-zero slope therefore indicates weak local sensitivity of the value to temperature,
consistent with flat CTD curves on HumanEval and MBPP.
\end{corollary}
\noindent (Without binariness the logistic is lost, but $u_{\qst}=\partial_\mu\mathcal A$ and
$\partial_\mu u=\Var_{\qst}(U)\ge0$ still makes the value increasing in the tilt.)

\paragraph{Count-threshold benchmarks.}
For these count-threshold instances, the exact optimum is available and can be used to score
both estimators. The instance carries, at each step $t$, a \emph{good set}
$G_t\subseteq\W$: the tokens that count as progress at that step (in this construction, the
tokens whose merit clears the threshold; $G_t$ is part of the specification, not learned). A
sequence passes iff at least $M$ of its $n$ tokens are good,
\begin{equation}\label{eq:count}
U(w^n)=\mathbf{1}\big[\,|\{t:w_t\in G_t\}|\ge M\,\big],
\end{equation}
so $M=1$ is ``at least one correct step'' and $M=n$ is pure AND. Assume in addition that both
the base conditionals and the good sets are prefix-free, $p_t(\cdot\mid w^{t-1})=p_t(\cdot)$ and
$G_t$ depending on $t$ alone. Then no feature of a prefix matters except the running count
$c_t:=|\{s\le t:w_s\in G_s\}|$, which is a sufficient statistic for both sweeps: $V$ and $u$
collapse from functions on $\W^{t}$ to functions of $c\in\{0,\dots,t\}$. Writing the tempered
good and bad masses at step $t$ as $P^{(\lambda)}_t=\sum_{w\in G_t}p_t(w)^{\lambda}$,
$Q^{(\lambda)}_t=\sum_{w\notin G_t}p_t(w)^{\lambda}$ and $C^{(\lambda)}_t=P^{(\lambda)}_t+Q^{(\lambda)}_t$,
the two sweeps of Proposition~\ref{prop:value} become
\begin{align*}
S_t(c)&=\tfrac1{\must}\log\!\big(P^{(\lamst)}_{t+1}e^{\must S_{t+1}(c+1)}+Q^{(\lamst)}_{t+1}e^{\must S_{t+1}(c)}\big),\\
g_t(c)&=\frac{P^{(\lamst)}_{t+1}e^{\must S_{t+1}(c+1)}}{P^{(\lamst)}_{t+1}e^{\must S_{t+1}(c+1)}+Q^{(\lamst)}_{t+1}e^{\must S_{t+1}(c)}},
\end{align*}
\[
v_t(c)=g_t(c)\,v_{t+1}(c+1)+(1-g_t(c))\,v_{t+1}(c),\qquad
S_n(c)=v_n(c)=\one[c\ge M],\qquad u_{\qst}=v_0(0).
\]
Here $S_t(c)$ is the potential of a state with count $c$ after $t$ tokens, $g_t(c)$ is the
probability that the optimal decoder emits a \emph{good} token from that state, and $v_t(c)$ is
the pass value. The division of labour of Proposition~\ref{prop:value} is visible in the
displays: $S$ log-sum-exps, $v$ averages, and only $S$ appears in the exponent of $g$. Cost is
$O(nM)$ time and $O(M)$ memory, all $k$ at once, with no Monte-Carlo.

For pure AND ($M=n$) the count collapses to alive/dead. A dead branch still carries the
continuation free energy $e^{\must S^{\mathrm d}_{t+1}}=\prod_{s>t+1}C^{(\lamst)}_s\ne1$ (that
is the \emph{potential}; the pass \emph{value} on a dead branch is $0$), and
\begin{equation}\label{eq:pureand}
u_{\qst}=\prod_{t=0}^{n-1}g^{\star}_t,\qquad
g^{\star}_t=\frac{P^{(\lamst)}_{t+1}e^{\must S^{\mathrm a}_{t+1}}}
{P^{(\lamst)}_{t+1}e^{\must S^{\mathrm a}_{t+1}}+Q^{(\lamst)}_{t+1}e^{\must S^{\mathrm d}_{t+1}}},
\end{equation}
$S^{\mathrm a}_n=1,\ S^{\mathrm d}_n=0$. Replacing $e^{\must S^{\mathrm d}_{t+1}}$ by $1$ mis-scales the bad-token weight in \eqref{eq:policy}, so the recovered
decoder is not $\qst$ and its deployed $\pass$ falls short of the optimum --- by an order of
magnitude when the discarded product of normalizers is far from $1$; reference curves built
on it must be recomputed. On a dead branch \eqref{eq:policy} reduces to pure temperature
$p^{\lamst}$, so the optimum is state-dependent even at $M=n$ and the prefix-free
Poisson--binomial score $u_q=\Pr[\mathrm{PoiBin}(g_1,\dots,g_n)\ge M]$ does not substitute
for it.

\ifPASSKstandalone\end{document}\fi

\section{Additional Experimental Details}
\label{app:exp-details}

In this appendix we provide additional experimental results and details for our practical Algorithm~\ref{alg:tmad-gated}.

\begin{algorithm}[!htbp]
\caption{TMAD-Gated: continuation-value fitting and gated decoding.}
\label{alg:tmad-gated}
\begin{algorithmic}[1]
\Require base model $p$, verifier $U$, reference temperature $T^\star$, calibration prompts, gate $g$, conservative temperature $T_{\rm conf}$, regime set $\mathcal R=\{(T_j,\eta_j)\}_{j=1}^J$
\State \textbf{Offline:} collect fixed-$T^\star$ rollouts and fit $\hat S_\theta(h,w)$ to predict eventual verifier success.
\For{$t=1,\ldots,n$}
  \State compute $p_t(\cdot\mid w^{t-1})$ and mark $t$ active if confidence is low or the position is structurally critical
  \If{$t$ is inactive}
    \State sample conservatively, $w_t\sim q_{\rm conf}$ with $q_{\rm conf}(w)\propto p_t(w\mid w^{t-1})^{1/T_{\rm conf}}$
  \Else
    \State score candidate nucleus with $\hat S_t$ and refine the top candidates by short lookahead (\ref{eq-S})
    \State choose $(T,\eta)\in\mathcal R$ using an estimated-success/plausibility score
    \State sample $w_t\sim q_{T,\eta}$, where $q_{T,\eta}(w)\propto\bigl(p_t(w)e^{\eta\hat S_t(w)}\bigr)^{1/T}$
  \EndIf
\EndFor
\State \Return $w^n$
\end{algorithmic}
\end{algorithm}

\subsection{Continuation-success head training and diagnostics}
\label{app:exp-head}

We provide here only the additional implementation details for the
continuation-success head used by \textsc{TMAD-Gated}. Training data is
collected via \emph{fixed-temperature} rollouts: for each training problem
we sample $k$ completions at the reference temperature $T^\star$
(i.e.\ $\lambda_{\mathrm{ref}} = 1/T^\star$, with no value-based tilting),
and label every token position in a rollout with the rollout's realized
binary outcome $U \in \{0,1\}$ (unit-test pass/fail). To concentrate
training signal where it is most informative, we restrict attention to
problems for which the pooled draws contain both at least one pass and one
fail, rank them by the empirical variance $\hat p(1-\hat p)$ of their pass
rate $\hat p$, and use the top-$K$ problems to build the training buffer.

The head takes as input
$h_t = [h_t^{(L)}; \bar h_t] \in \mathbb{R}^{2d}$, the concatenation of the
final-layer hidden state $h_t^{(L)}$ with the mean of the last four
transformer-layer hidden states $\bar h_t$ at position $t$. The head is a
low-rank bilinear function over the vocabulary,
\begin{equation}
\label{eq-S}
\hat S_\theta(h_t,w)
=
\sigma\!\big(a_w^\top(B h_t) + b_w\big),
\qquad
A \in \mathbb{R}^{|\mathcal V|\times r},\;
B \in \mathbb{R}^{r\times 2d},
\end{equation}
where $a_w$ is the row of $A$ corresponding to token $w$, $r \ll d$ is the
bottleneck rank, and $\sigma$ is the sigmoid. Parameters
$\theta = (A,B,b)$ are trained with binary cross-entropy between
$\hat S_\theta(h_t,w_t)$ and the rollout label $U$, using a class-balancing
positive weight, Adam, and gradient clipping. We hold out a disjoint subset
of \emph{problems}, rather than token positions, for early stopping and
select the checkpoint with the lowest held-out binary cross-entropy.

\paragraph{TMAD-Gated decoding details.}
We use nucleus sampling with $p = 0.95$. A position is active when the
maximum base-model token probability is below $0.90$ or when it is a
structurally critical code position under the same newline/indentation state
used by our \textsc{AdapT} reproduction. Inactive positions use conservative
sampling at $T_{\mathrm{conf}} = 0.40$. At an active position, the value head
scores the full nucleus. We select from the finite regime grid
\[
T \in \{0.75T^\star,\, T^\star,\, 1.5T^\star\},
\qquad
\eta \in \{0,\, 1,\, 3\}.
\]
For each regime, the verifier tilt is scaled by the normalized entropy of
the current base distribution, and the regime is scored by the
Lagrangian-style criterion
\[
\mathbb{E}_q[\hat S_t^{\mathrm{LA}}]
\;+\;
\kappa\,\mathbb{E}_q[\log p_t],
\qquad
\kappa = 0.6,
\]
where $\hat S_t^{\mathrm{LA}}$ denotes the lookahead-refined continuation
score. The highest-scoring regime is selected and a token is sampled from
its stochastic tilted-temperature distribution. Thus the current
implementation uses a soft plausibility penalty rather than a hard per-step
budget-feasibility filter.

Because $\hat S_t$ is a practical proxy for continuation success rather than
an exact estimate of the free-energy potential $V_t$, we evaluate it
directly on held-out problems using ROC-AUC. This measures whether the
learned score ranks token--prefix states that eventually lead to successful
completions above those leading to failures. The held-out split is disjoint
at the problem level from the data used to train the head.

\begin{table}[!htbp]
\centering
\caption{Held-out ROC-AUC of the continuation-success head $\hat S_\theta$,
reported as mean $\pm$ standard deviation across runs. Higher is better;
AUC measures the practical proxy, not direct error in $V_t$.}
\label{tab:value-auc}
\small
\begin{tabular}{lcc}
\toprule
Model & HumanEval AUC & MBPP AUC \\
\midrule
Qwen2.5-Coder-1.5B & $0.632 \pm 0.022$ & $0.641 \pm 0.007$ \\
Qwen3-4B-Base      & $0.648 \pm 0.009$ & $0.628 \pm 0.092$ \\
CodeLlama-7B &  $0.646 \pm 0.072$ & $0.639 \pm 0.033$ \\
\bottomrule
\end{tabular}
\end{table}

\begin{figure}[!t]
\centering
\begin{subfigure}[t]{0.46\textwidth}
  \centering
  \includegraphics[width=\linewidth]{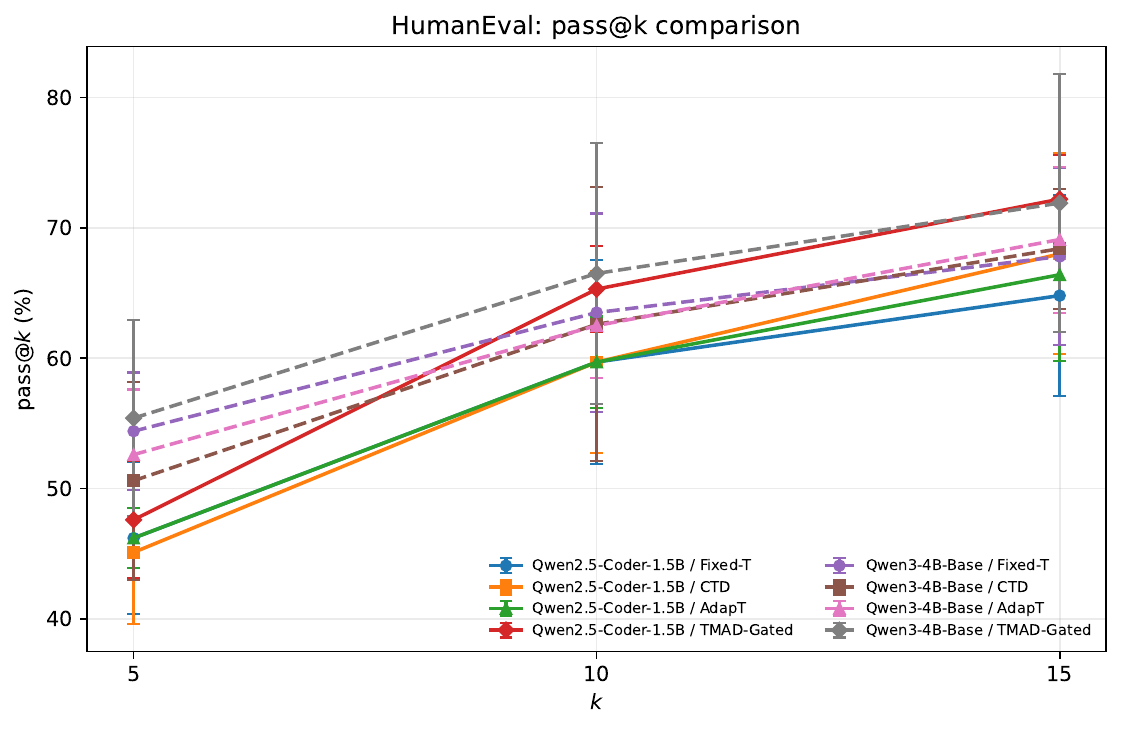}
  \caption{HumanEval.}
  \label{fig:passk-humaneval}
\end{subfigure}\hfill
\begin{subfigure}[t]{0.46\textwidth}
  \centering
  \includegraphics[width=\linewidth]{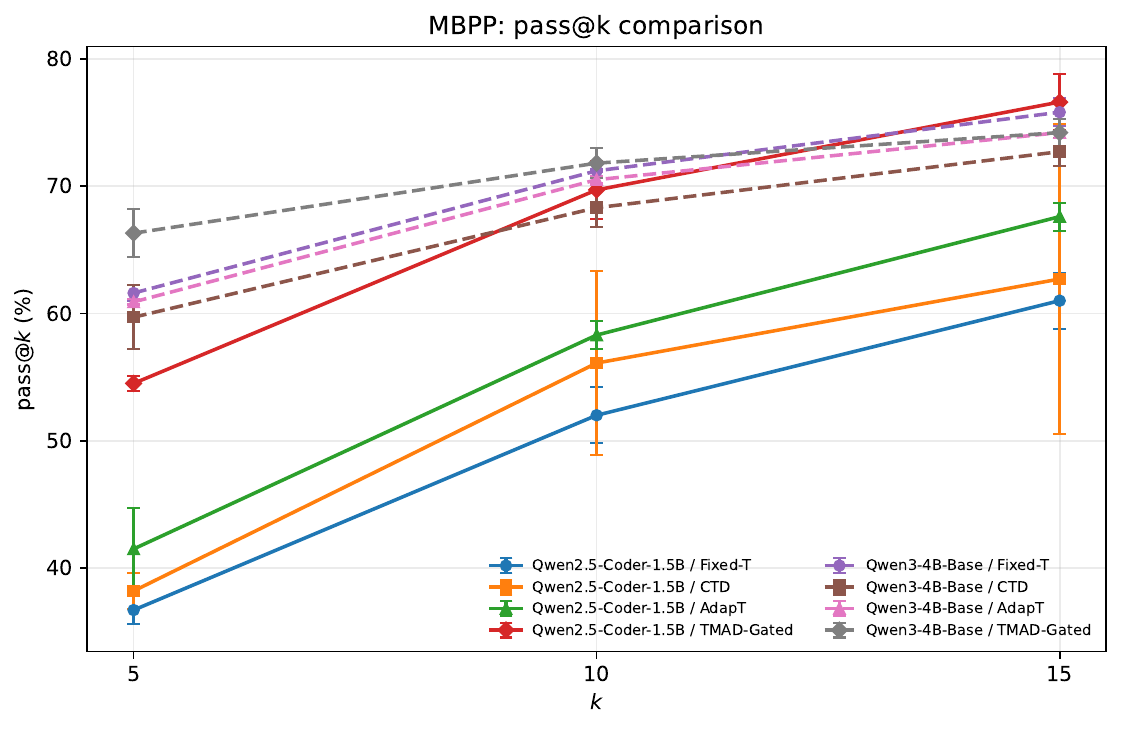}
  \caption{MBPP.}
  \label{fig:passk-mbpp}
\end{subfigure}
\caption{\textbf{Pass@$k$ curves for the final controlled comparison.}
Mean pass@$k$ for $k\in\{5,10,15\}$ across repeated runs, corresponding to
Table~\ref{tab:main-results}. Error bars indicate the standard deviation
across runs.}
\label{fig:passk-main-curves}
\end{figure}

\subsection{Held-out evaluation split}
\label{app:exp-split}

For each model and benchmark, a disjoint subset of problems is reserved for
calibration and continuation-head training, and all controlled decoding
methods are evaluated on the same remaining held-out problems. Consequently,
the primary matched comparison is among fixed-temperature sampling, our
\textsc{AdapT} reproduction, \textsc{CTD}, and \textsc{TMAD-Gated}. In all
reported pass@$k$ comparisons, methods within a model--benchmark pair are
evaluated on the same held-out subset.

\FloatBarrier

\end{document}